\pdfoutput=1

\documentclass[english,11pt]{article}
\usepackage[T1]{fontenc}
\usepackage[utf8]{inputenc}
\usepackage{geometry}
\usepackage{verbatim}
\usepackage{float}
\usepackage{bm}
\usepackage{amsmath}
\usepackage{amssymb}
\usepackage{amsfonts}
\usepackage{graphicx}
\usepackage{breakurl}
\usepackage{multirow}
\usepackage{mathtools}
\usepackage{makecell}
\usepackage{xr}

\RequirePackage[authoryear,round]{natbib} 
\RequirePackage[colorlinks,citecolor=blue,urlcolor=blue]{hyperref}

\usepackage{cleveref}
\usepackage{lipsum}
\usepackage{subfig}

\makeatletter
\def\blfootnote{\gdef\@thefnmark{}\@footnotetext}

\floatstyle{ruled}
\newfloat{algorithm}{tbp}{loa}
\providecommand{\algorithmname}{Algorithm}
\floatname{algorithm}{\protect\algorithmname}

\usepackage{amsthm}\usepackage{dsfont}\usepackage{array}\usepackage{mathrsfs}\usepackage{comment}\onecolumn

\usepackage{color}

\usepackage{enumitem}
\setlist[itemize]{leftmargin=1em}
\setlist[enumerate]{leftmargin=1em}
\allowdisplaybreaks

\usepackage[algo2e]{algorithm2e}
\usepackage{algorithmic}
\usepackage{arydshln}
\usepackage{enumitem}

\makeatother

\theoremstyle{plain} \newtheorem{lemma}{\textbf{Lemma}} 
\newtheorem{theorem}{\textbf{Theorem}}

\newtheorem{assumption}{\textbf{Assumption}}
\newtheorem{proposition}{\textbf{Proposition}}
\newtheorem{example}{\textbf{Example}}
\newtheorem{definition}{\textbf{Definition}}
\newtheorem{remark}{\textbf{Remark}}
 \theoremstyle{definition}

\usepackage{bbm}                                     
\usepackage{xcolor}                                  
\usepackage{tikz}
\usetikzlibrary{decorations.pathreplacing, angles, quotes}

\DeclareMathOperator{\E}{E}

\newcommand{\logn}{\log n}

\newcommand{\mbR}{\mathbb{R}}

\newcommand{\mcA}{\mathcal{A}}
\newcommand{\mcB}{\mathcal{B}}

\newcommand{\mcD}{\mathcal{D}}
\newcommand{\mcE}{\mathcal{E}}
\newcommand{\mcF}{\mathcal{F}}

\newcommand{\mcH}{\mathcal{H}}
\newcommand{\mcM}{\mathcal{M}}
\newcommand{\mcN}{\mathcal{N}}

\newcommand{\mcX}{\mathcal{X}}

\newcommand{\mcY}{\mathcal{Y}}

\begin{document}

\title{Fixed-Gaussian Spectral Algorithms: Minimax Optimal Rates for Misspecified Learning and Transfer}

\author{Haotian Lin\thanks{Department of Statistics, The Pennsylvania State University.} \and Matthew Reimherr\footnotemark[\value{footnote}]}


\date{}

\maketitle

\begin{abstract}
The principal objective of this work is twofold within nonparametric regression settings: (1) to establish the minimax optimal convergence rates for fixed-bandwidth Gaussian kernel spectral algorithms when the true regression function resides in a Sobolev space, and (2) to apply Gaussian spectral algorithms for achieving robust and adaptive transfer learning under concept shift. While minimax optimality of misspecified spectral algorithms has been established, existing guarantees are typically restricted to the non-saturation regime. We demonstrate that the infinite smoothness of fixed-bandwidth Gaussian kernels provides universal robustness to model misspecification by showing that this kernel choice enables any spectral algorithm to attain minimax optimal rates, provided the regularization parameter decays exponentially. This result effectively decouples optimality from the algorithm's inherent qualification. Building on this, we then advocate Gaussian spectral algorithms as powerful components in a learning framework for robust and adaptive transfer. Specifically, we derive the adaptive convergence rate of the excess risk for this framework and show that the rates are optimal up to logarithmic factors. Our results also reveal the impact of the magnitude of the concept shift and the sample size on the generalization error.
\end{abstract}


\section{Introduction}

In nonparametric regression, the primary objective of a learning algorithm is to construct an estimator over the hypothesis space $\mcH$ from $n$ noisy samples that minimizes the squared expected risk to the regression function $f^{*}$. Within this landscape, kernel methods, as one of the most popular approaches, choose a reproducing kernel Hilbert space (RKHS) associated with a kernel function $k$ as the hypothesis space. In this work, we focus on a broad class of kernel methods called \textit{spectral algorithms}, which include kernel ridge regression, a variety of different implementations of kernel gradient descent, kernel principal component regression, and many more. Beyond the standard single-task setting, spectral algorithms also serve as powerful building blocks for learning under distribution shift, where abundant data from a related source domain is leveraged to improve learning on a target domain with scarce labels. We study spectral algorithms from both perspectives: we first characterize their optimal learning behavior in the standard setting, and then exploit it to achieve robust transfer under distribution shift.

\subsection{Optimal learning with Gaussian spectral algorithms}
The minimax optimality of spectral algorithms has been studied extensively in the literature; see, e.g., \citet{caponnetto2006optimal,caponnetto2007optimal,fischer2020sobolev,zhang2024optimality}. Existing results show that the optimal convergence rate of the generalization error is attainable even when the estimator learned via spectral algorithms is from a misspecified RKHS, i.e., $f^{*}\notin \mcH$. Such robustness to misspecification typically depends on the qualification of the spectral algorithms (a quantity measuring the algorithm’s fitting capability; see definition in Section~\ref{subsec: def on SA}). In particular, when the relative regularity of the estimator to the regression function exceeds twice the qualification, these robustness results against misspecified learning cases will break down due to the well-known \textit{saturation effect} \citep{bauer2007regularization}. In this regime, the algorithm's convergence rate is strictly suboptimal and cannot attain the information-theoretic lower bound.

Our first main contribution is to provide some insights showing that using the Gaussian kernel, a particularly favored kernel, in spectral algorithms provides universal robustness, which decouples the selected algorithms from the saturation effect. 
\paragraph{Contribution 1 (Optimal rates of Gaussian spectral algorithms)} For regression functions that reside in fractional Sobolev spaces of order $m$ and dimension $d$, we show that employing \textit{fixed-bandwidth Gaussian kernels} within arbitrary spectral algorithms attains the classical minimax optimal convergence rates. Notably, the optimal order of the regularization parameter $\lambda$ for achieving non-adaptive rates should decay exponentially in sample size $n$, namely $\lambda \asymp \exp\{ -n^{\frac{2}{2m+d}} \}$. When $m$ is not known a priori, we show standard training validation achieves the same minimax optimality up to a logarithmic factor. To the best of our knowledge, these are the first optimal rates established for fixed-bandwidth Gaussian kernels within general spectral algorithms, serving as a complement to classical misspecified kernel methods research.


\subsection{Robust transfer learning under concept shift}
Distribution shift between training and testing domains poses a fundamental challenge for modern machine learning. A prominent instance of such a shift is known as \textit{concept shift}, where the conditional distribution of $Y|X$ varies across domains and the marginal distribution of covariates $X$ remains invariant. The problem of transfer learning under concept shift (also known as concept shift adaptation) posits that one has limited labeled samples from the target domain but abundant labeled samples from a similar source domain, with the goal of learning a predictor that performs well under the target distribution. This problem has been central to the literature of transfer learning, especially when the target labels are scarce or costly to obtain.

Hypothesis transfer learning (HTL) \citep{kuzborskij2013stability,kuzborskij2017fast} has emerged as a dominant framework for tackling concept shift. It operates by leveraging (pre-)trained source hypotheses from the source domain, followed by using labeled target samples to fine-tune the hypotheses. Existing theoretical guarantees for this framework focus predominantly on the \textit{well-specified} setting, where the regression function is assumed to reside in the specified hypothesis space. These existing works typically attain (near) optimal rates for excess risk only by assuming the function regularity is known a priori, allowing for perfect model specification and hyperparameter tuning \citep{li2022transfer,tian2022transfer,wang2016nonparametric,du2017hypothesis,lin2024on}. In contrast, provable and optimal guarantees for learning under concept shift under \textit{misspecification} remain largely unexplored, despite the ubiquity of misspecification in practice.

The second main contribution of this paper aims to close this gap. Instead of using existing misspecified results subject to saturation, we advocate using Gaussian spectral algorithms as the main component in the framework to obtain the first, to our knowledge, robust and rate-optimal transfer under the misspecified model setting. Here, our contribution is twofold.

\paragraph{Contribution 2 (A robust and adaptive learning framework under concept shift)} We propose a transfer learning framework that employs Gaussian spectral algorithms as the learning algorithms in both the source training and target fine-tuning phases. This design leverages the novel robustness against model misspecification exhibited by the Gaussian kernel and is thus free from saturation effects, enabling consistently rate-optimal learning without requiring prior knowledge of the function regularity.

\paragraph{Contribution 3 (Theoretical insights into transfer learning under concept shift)} We establish the minimax lower bound for the learning problem under concept shifts in terms of excess risk and demonstrate that our proposed framework attains the minimax optimal convergence rate (up to logarithmic factors). Crucially, our analysis reveals that the optimal rate depends not only on the smoothness of the source and intermediate shift functions but also on the \textit{shift-to-signal ratio}. We identify this ratio as a key factor governing transfer efficiency, a dependency largely overlooked in prior analyses.

    


The rest of the paper is organized as follows. Section~\ref{sec: related work} reviews related work on misspecified kernel methods and learning under concept shift. Section~\ref{sec: preliminaries} introduces spectral algorithms and formulates the nonparametric regression problem under concept shift. Section~\ref{sec: SA with Gaussians} establishes the minimax optimal non-adaptive and adaptive convergence rates for spectral algorithms with fixed-bandwidth Gaussian kernels. Section~\ref{sec: transfer learning} presents the proposed learning procedure under concept shift, followed by its minimax optimality analysis and discussions. Section~\ref{sec: numerical experiments} reports simulations confirming the derived convergence rates. Finally, Section~\ref{sec: discussion} concludes the paper.

\section{Related Work}\label{sec: related work}
\subsection{Misspecified spectral algorithm} 
For the theoretical landscape of misspecified spectral algorithms, a line of works \citep{steinwart2009optimal,dieuleveut2016nonparametric,rastogi2017optimal,dicker2017kernel,blanchard2018optimal,pillaud2018statistical,lin2018optimal,fischer2020sobolev,lin2020optimal} has investigated the cases where the true function $f^{*}$ belongs to the interpolation space of the imposed RKHS, i.e., $[\mcH]^{s}$, and derived the minimax optimal rate for spectral algorithms under different norms. To our knowledge, the state-of-the-art results in \citet{zhang2024optimality} recover the minimax optimal rate for spectral algorithms under weaker regularity conditions on $f^{*}$. However, a significant limitation of the existing literature is its reliance on the assumptions of polynomial eigenvalue decay or effective dimension bounded by $\lambda^{-\beta}$. Gaussian kernels, characterized by exponentially fast eigenvalue decay, do not satisfy these standard assumptions. While \citet{rastogi2023inverse} considers a logarithmic decay of the effective dimension, their results are restricted to specific cases, such as the univariate Gaussian kernel $K(x,y) = xy + \exp\{-8(x-y)^2\}$, which deviates from the classical multivariate Gaussian kernels, $\exp\{-\|x-y\|_{2}^2/2h^{2}\}$, considered in this work. Alternative approaches using variable bandwidth Gaussian kernels have demonstrated that polynomially decaying regularization parameters and bandwidths in KRR can yield convergence rates that are arbitrarily close to the optimal rate for $f^{*}$ Sobolev or Besov spaces \citep{eberts2013optimal}. This is further refined by \citet{hamm2021adaptive}, who improves convergence rates up to logarithmic factors under the so-called DIM condition.

Regarding the saturation effect, particularly within KRR, it was observed in practice and reported in \citep{bauer2007regularization,gerfo2008spectral}. It was only recently that the saturation effect on KRR was theoretically proved by \citet{li2023saturation} and later extended to spectral algorithms by \citet{li2024generalization}. In other specific settings, like (stochastic) distributed-based, online-based spectral algorithms and average gradient descent, the saturation effect was also noted and overcome by method-specified techniques \citep{zhang2015divide,lin2017distributed,mucke2019beating}. 

\subsection{Transfer learning under concept shift}
Under plausible distribution shifts relating two domains, the technique of leveraging data from source domains to enhance learnability in the target domain is known as transfer learning or domain adaptation \citep{pan2009survey,zhuang2020comprehensive}. This paradigm has been extensively studied within nonparametric methods and the RKHS context, covering scenarios such as covariate shift \citep{kpotufe2021marginal, ma2023optimally, wang2023pseudo}, and concept shift \citep{wang2014flexible, wang2015generalization,du2017hypothesis,cai2024transfer}. Within the study of concept shifts, HTL has emerged as a prominent methodological and theoretical tool. Early works \citep{kuzborskij2013stability,kuzborskij2017fast} study the shifted model being a linear model and establish generalization bounds through Rademacher complexity. The authors of \citet{wang2016nonparametric} assume true source functions in a Sobolev ellipsoid, while shifted offsets belong to smoother ones, and model them with finite basis functions. The work \citet{du2017hypothesis} proposes using general transformations to model domain shift, which is later refined by \citet{minami2024transfer} to derive optimal transformation functions under squared loss. A more recent line of research has been devoted to studying the minimax optimality for a special case of HTL, offset transfer learning. These works cover diverse models, including high-dimensional (generalized) linear regression \citep{li2022transfer,tian2022transfer,zhang2022class}, binary classification \citep{reeve2021adaptive,maity2022minimax,maity2024linear}, functional linear models \citep{lin2024on}, Gaussian mixture models \citep{tian2022unsupervised}, graphical models \citep{li2023transfer}, among others. These works, however, investigated minimax optimality exclusively under the well-specified case, leaving the robustness guarantees largely unaddressed.

Theoretical investigations into robustness against model misspecification in transfer learning and domain adaptation have historically concentrated on semiparametric models, as parametric components are particularly sensitive to misspecified models, which can degrade estimation efficiency. For example, \citet{hu2023optimal} employ model averaging to achieve asymptotic minimax optimality when the target model is misspecified. Another line of work \citep{liu2023augmented,zhou2024doubly,zhou2024model,cai2024semi} develops doubly or triply robust estimation procedures under the covariate shift setting, allowing for one or multiple specified models being misspecified. For RKHS-based nonparametric regression, \citet{lin2024smoothness} derives optimal rates using misspecified kernels, but their framework was confined to offset transfer learning with the learning algorithms in both phases being KRR. For nonparametric classification, several works obtain minimax optimal guarantees by using robust models, such as decision trees \citep{hamm2021adaptive} or nearest neighbor estimators \citep{cai2021transfer}, under the assumption that true classification functions reside in certain H\"older spaces. Despite variations in problem settings, these research efforts underscore a common objective: achieving robustness to model misspecification and adaptivity while maintaining minimax optimality is crucial for transfer learning and domain adaptation algorithms, as these properties underpin their performance guarantees.

\section{Preliminaries}\label{sec: preliminaries}
In this section, we provide the background of nonparametric regression and spectral algorithms and then formulate the problem of learning under concept shift.

\subsection{Background and basic concepts}\label{subsec: background and basic concepts} We consider the standard supervised learning setting on a compact input space $\mcX \subset \mathbb{R}^{d}$ with a Lipschitz boundary and an output space $\mcY \subseteq \mathbb{R}$. We assume access to a dataset of $n$ i.i.d. samples $D = \{(x_{i},y_{i})\}_{i=1}^{n}$, drawn from an unknown distribution $P$ over $\mcX \times \mcY$. The goal is to learn a function $f$ in the hypothesis space $\mcH$ based on the dataset $D$ to approximate the regression function $f^{*}(x) = \int_{\mcX} y d P(y|x)$ well for the excess risk, i.e., $L^{2}$-norm generalization error.

We investigate the setting where the $f^{*}$ resides in an RKHS $\mcH_{K}(\mcX)$ associated with a symmetric, positive-definite, and continuous kernel function $K$ over $\mcX \times \mcX$ satisfying $\sup_{x\in \mcX}K(x,x)\leq \kappa^2$. For a stationary kernel where $K(x,y) = K(x-y)$, with the domain $\mcX = \mathbb{R}^{d}$, one can characterize RKHS via Fourier transforms by Theorem 10.12 of \citet{wendland2004scattered}, i.e., 
\begin{equation}\label{eqn: Spectral Characetrization of RKHS}
    \mcH_{K}(\mbR^{d}) = \left\{ f \in L^{2}(\mathbb{R}^{d}) \cap C(\mathbb{R}^{d}) : \frac{\mcF(f)}{\sqrt{\mcF(K)}} \in L^{2}(\mathbb{R}^{d}) \right\}.
\end{equation}
where $\mcF(f)$ denotes the Fourier transform of a function $f$. When $\mcX \subset \mbR^{d}$, this definition remains valid via a norm equivalency result guaranteed by the extension theorem \citep{devore1993besov}. Denote $\mathcal{W}^{m,p}(\mcX)$ as the Sobolev space with order $m$, then $\mathcal{W}^{m,p}(\mcX)$ is equivalent to the RKHS of a stationary kernel when $p=2$. The following lemma from \citet{wendland2004scattered} describes the norm equivalency.  
\begin{lemma}\label{lemma: equivalence between RKHS and Sobolev}
    Let $K(x,x')$ be a stationary kernel. Suppose $\mcX$ has a Lipschitz boundary, and the Fourier transform of $K$ has the following spectral density of $m$, for $m >  d/2$, 
    \begin{equation}\label{eqn: Fourier Transform of reproducing kernel}
        c_{1}(1 + \|\cdot\|_{2}^{2})^{-m} \leq \mcF(K)(\cdot) \leq c_{2}(1 + \|\cdot\|_{2}^{2})^{-m}.
    \end{equation}
    for some constant $0 < c_{1} \leq c_{2}$. Then, the associated RKHS of $K$, $\mcH_{K}(\mcX)$, is norm-equivalent to the Sobolev space $\mathcal{W}^{m,2}(\mcX):= H^{m}(\mcX)$.
\end{lemma}
Some prominent classes of stationary kernels that satisfy \eqref{eqn: Fourier Transform of reproducing kernel} include the (isotropic) Mat\'ern kernel \citep{stein1999interpolation} and the generalized Wendland kernels \citep{wendland2004scattered}. Let $H^{m}(\mcX,R)$ represent a ball centered at the origin with radius $R$, i.e. $H^{m}(\mcX,R) = \{f : f\in H^{m}(\mcX), \|f\|_{H^{m}(\mcX)} \leq R \}$, we then abbreviate $H^{m}(\mcX)$ as $H^{m}$ and $H^{m}(\mcX,R)$ as $H^{m}(R)$ unless otherwise specified.

The following assumptions will be made throughout the paper.
\begin{assumption}[Smoothness]\label{assumption: smoothness}
Suppose there exist positive constants $m > d/2$ and $R$ such that $f^{*} \in H^{m}(R)$.
\end{assumption}

\begin{assumption}[Moment of error] \label{assumption: error tail} 
There are constants $\sigma$,$L>0$ such that for any $r\geq 2$, 
\begin{equation*}
    \E \left[ | y -f^{*}(x)|^{r} \mid x  \right] 
    \leq \frac{1}{2}r! (\sigma)^{2} (L)^{r-2}.
\end{equation*}
\end{assumption}

Assumption~\ref{assumption: smoothness} indicates the eigenvalue decay rate (EDR) of the Sobolev space is polynomial, which is a standard assumption in nonparametric regression, and induces the classical optimal rate  $n^{-\frac{2m}{2m+d}}$ for the excess risk \citep{stone1982optimal}. Assumption~\ref{assumption: error tail} is a standard assumption in the kernel method literature \citep{fischer2020sobolev,zhang2024optimality}, which controls the noise tail-probability decay speed.

\subsection{Spectral algorithms}\label{subsec: def on SA}
Let $\mcH_{K}$ be an RKHS with kernel $K$. Denote the $L^{2}$-space with respect to the marginal distribution $P_X$ as $L^{2}(\mcX, P_{X})$ (in short $L^{2}$). Define the integral operator $T_{K}:L^{2} \rightarrow L^{2}$ as
\begin{equation*}
    T_{K}(f)(x^{\prime}) = \int_{\mcX} K(x,x^{\prime}) f(x) dP_{X}(x),
\end{equation*}
which is a compact operator since it is positive, self-adjoint, and trace-class. By Mercer's theorem, there exists an at most countable index set $N$ such that the decomposition $K(x,x^{\prime}) = \sum_{j\in N}s_{j} e_{j}(x) e_{j}(x^{\prime})$ holds, where $\{s_{j}\}_{j\in N}$ are the eigenvalues in non-increasing order and $\{e_{j}\}_{j\in N}$ are the corresponding eigenfunctions, which form an orthonormal basis of $L^{2}$. 

For any $x\in \mcX$, define the evaluation operator $K_{x}: \mcH_{K} \rightarrow \mbR, f \mapsto \langle f, K_{x} \rangle_{\mcH_{K}}$ and its adjoint operator $K_{x}^{*}: \mbR \rightarrow \mcH_{K}, y \mapsto yK_{x}$. Given a dataset $D$, we define the sample covariance operator $T_{K,n}:\mcH_{K}\rightarrow \mcH_{K}$ and the sample bias function $g_{n}:\mbR^{n} \rightarrow \mcH_{K}$ as 
\begin{equation*}
    T_{K,n} := \frac{1}{n}\sum_{i=1}^{n}K_{x_{i}}^{*}K_{x_{i}} \quad \text{and} \quad g_{n}:=\frac{1}{n} \sum_{i=1}^{n} K_{x_{i}}^{*}y_{i}.
\end{equation*}
The learning goal is to learn a function $\hat{f} \in \mcH_{K}$ that achieves low excess risk. Thus, a natural approach is to minimize the empirical squared loss over $\mcH_{K}$, i.e., 
\begin{equation}\label{eqn: empirical loss minimization problem}
    \hat{f} = \underset{f\in \mcH_{K}}{\operatorname{argmin}}  \frac{1}{n} \sum_{i=1}^{n} \left( f(x_{i}) - y_{i} \right)^2,
\end{equation}
which gives $\hat{f}$ as the solution of the empirical, linear equation, $T_{K,n}f = g_{n}$. However, solving this equation is an ill-posed inverse problem, as the inverse of the sample covariance operator, $T_{K,n}$, in general, does not exist. A common approach to address this issue is to replace the inverse of $T_{K,n}$ with a regularized operator. This replacement corresponds to selecting a specific filter function, which leads to the formulation of spectral algorithms \citep{rosasco2005spectral,caponnetto2006optimal,bauer2007regularization}.

\begin{definition}[Filter function]
    The family of functions $\{ \phi_{\lambda} : [0,\kappa^2] \rightarrow \mbR^{+} | \lambda \in \mbR \}$ are called filter functions with qualification $\tau \geq 0$ and regularization parameter $\lambda$, if there exist some positive constants $E, F_{\tau} < \infty$ such that the following two conditions hold:
    \[
        \sup _{\beta \in[0,1]} \sup _{\lambda \in \Lambda} \sup _{\left.u \in[ 0, \kappa^2\right]}\left|u^\beta \phi_{\lambda}(u)\right| \lambda^{1-\beta} \leq E ,
    \]
    \[
        \sup _{\beta \in[0, \tau]} \sup _{\lambda \in \Lambda} \sup _{\left.u \in[ 0, \kappa^2\right]}\left|\left(1-\phi_{\lambda}(u) u\right)\right| u^\beta \lambda^{-\beta} \leq F_\tau. 
    \]
\end{definition}

The motivation for introducing the class of filter functions is that $\phi_{\lambda}$ approximates the function $\phi(t) = t^{-1}$ but with better behavior around $0$, e.g. $\phi_{\lambda}$ is bounded by $\lambda^{-1}$. We now define the corresponding spectral algorithms given a specific filter function $\phi_{\lambda}$.
\begin{definition}[Spectral algorithm]\label{def: spectral algorithms}
    Given a filter function $\phi_{\lambda}$, with parameter $\lambda$, the estimator produced by the corresponding spectral algorithm is
    \begin{equation*}
        \hat{f} = \phi_{\lambda}(T_{K,n}) g_{n}.
    \end{equation*}
\end{definition}
Here, the filter function $\phi_{\lambda}$ is understood as acting on the eigenvalues of the self-adjoint, finite-rank sample covariance operator $T_{K,n}$. Different choices of filter functions correspond to different regularization schemes. We list common examples below and refer readers to \citet{gerfo2008spectral} for more examples.
\begin{enumerate}
    \item Kernel Ridge Regression (KRR): The choice of Tikhonov filter function $\phi_{\lambda}(z) = (z + \lambda)^{-1}$ corresponds to kernel ridge regression. In this case, $\tau= E = F_{\tau}=1$.

    \item Gradient Descent (GD) and Gradient Flow (GF):
    The choice of $\phi_{\lambda}(z) = \eta \sum_{k=1}^{t-1} (1 - \eta z)^{k}$ with $\lambda = (\eta t)^{-1}$ corresponds to gradient descent, where $\eta > 0$ is a constant step size. Gradient descent with more complex update rules can be expressed in terms of filter functions as well \citep{lin2018optimal,mucke2019beating}. For an infinitely small step size, this converges to gradient flow, corresponding to $\phi_{\lambda}(z) = (1 - \exp\{-z/\lambda\}) / z$. In both cases, $\tau$ could be any positive number, $E=1$, and $F_{\tau}=(\tau/e)^{\tau}$.

    \item Kernel Principal Component Regression (KPCR): The choice of spectral cut-off function $\phi_{\lambda}(z) = z^{-1} \mathbbm{1}_{z\geq \lambda}$ corresponds to Kernel Principal Component Regression, which is motivated by using finite components to recover $T_{K,n}$. In such case, $\tau$ could be any positive number, $\tau= E = F_{\tau}=1$.
\end{enumerate}

\subsection{Transfer learning under concept shift}\label{subsec: learning problem}
Suppose there are two unknown probability measures on $\mcX \times \mcY$, the measure $P$ for the source domain and $Q$ for the target domain. The concept shift setting, which is alternatively known as model shifts \citep{wang2014flexible,lei2021near} or posterior drift \citep{scott2019generalized,cai2021transfer}, posits the marginal distributions of $X$ ($P_{X}$ and $Q_{X}$) are identical, while the conditional distributions of $Y$ given $X$ ($P_{Y|X}$ and $Q_{Y|X}$) are different across the domains.

For data, we observe $n_{P}$ i.i.d. labeled source samples $\mcD^{P} = \{ (x_{i}^{P},y_{i}^{P}) \}_{i=1}^{n_{P}}$ drawn from distribution $P$, and $n_{Q}$ i.i.d. labeled target samples $\mcD^{Q} = \{ (x_{i}^{Q},y_{i}^{Q}) \}_{i=1}^{n_{Q}}$ drawn from distribution $Q$, with the data generation process as
\[
    y_{i}^{P} = f^{P}( x_{i}^{P} ) + \epsilon_{i}^{P} \quad \text{and} \quad y_{i}^{Q} = f^{Q}( x_{i}^{Q} ) + \epsilon_{i}^{Q},
\]
where $f^{P}$ and $f^{Q}$ are the underlying regression functions, and noise terms $\epsilon_{i}^{P}$ and $\epsilon_{i}^{Q}$ are i.i.d. random noise with zero mean. Given the observed samples from both domains, the goal is to learn an estimator $\hat{f}^{Q}$ that minimizes the excess risk over the target distribution $Q$, i.e., 
\[
\E_{X\sim Q_{X}}[ ( \hat{f}^{Q}(X) - f^{Q}(X) )^{2} ].
\]

The learning framework we consider is an extension of the HTL procedure adopted from \citet{du2017hypothesis}, which is both practically and theoretically prevalent for learning and adaptation under concept shift. The motivation of the framework is to decompose the difficult task of learning $f^{Q}$ into separately learning the source function $f^{P}$ and an intermediate function $f^{\delta}$, which captures the discrepancy between the two domains (formally defined in Section~\ref{subsection: convergence rate of RAHTL}). Typically, $f^{P}$ is less regularized (``more complex'') but can be learned effectively due to the abundance of source samples. Conversely, the intermediate function $f^{\delta}$ is assumed to be highly regularized (``simple'', e.g., linear functions), making it learnable even with limited target samples. The complete procedure is outlined in Algorithm~\ref{algo: HTL}.

\begin{algorithm}[ht]
\caption{Learning under Concept Shift via HTL}\label{algo: HTL}
    
    \textbf{Input:} Source samples $\mcD^{P}$,
    target samples $\mcD^{Q}$, hypothesis classes $\mcH^{P}$ and $\mcH^{\delta}$,
    learning algorithms $\mcA^{P}:(\mcX \times \mcY)^{n_{P}} \rightarrow \mcH^{P}$, and $\mcA^{\delta}:(\mcX \times \mcY)^{n_{Q}} \rightarrow \mcH^{\delta}$, and
    data transformation function $g: \mcY \times \mcH^{P} \rightarrow \mbR$ and model transformation function $G: \mbR \times \mcH^{P} \rightarrow \mcY$.

    {\textbf{Step 1 (Source Learning):}} Obtain the source hypothesis $\hat{f}^{P} \in \mcH^{P}$ via $\mcA^{P}$ and $\mcD^{P}$:
    \[ \hat{f}^{P} = \mcA^{P}( \mcD^{P} ). \]

    {\textbf{Step 2 (Transformation):}} Construct the intermediate dataset using the data transformation function $g$:
    \[
        \mcD^{\delta} = \left\{ (x_{i}^{Q},y_{i}^{\delta} ) \right\}_{i=1}^{n_{Q}}\quad \text{where} \quad y_{i}^{\delta} = g\left( y_{i}^{Q} , \hat{f}^{P}(x_{i}^{Q}) \right).
    \]

    {\textbf{Step 3 (Shift Estimation):}} Obtain the intermediate hypothesis $\hat{f}^{\delta} \in \mcH^{\delta}$ via $\mcA^{\delta}$ and $\mcD^{\delta}$:
    \[ \hat{f}^{\delta} = \mcA^{\delta}(\mcD^{\delta}). \]

    {\textbf{Step 4 (Target Reconstruction):}} Construct the final target estimator $\hat{f}^{Q}$ using the model transformation function $G$:
    \[
        \hat{f}^{Q}(X) = G\left( \hat{f}^{\delta}(X), \hat{f}^{P}(X) \right).
    \]
\end{algorithm}

This procedure encapsulates a wide range of transfer learning algorithms under concept shifts through specific choices of transformation functions in Steps 2 and 4 of Algorithm~\ref{algo: HTL}, namely $g$ for transforming target labels and $G$ for reconstructing the final predictor. For instance, selecting $g(x,y) = x - y$ and $G(x,y) = x+y$ represents the offset transfer learning, where the target function is modeled as the sum of the source and intermediate functions. This formulation recovers numerous algorithms based on biased regularization \citep{scholkopf2001generalized}, which regularize the target hypothesis towards the pre-trained hypothesis via regularized ERM, such as those by \citet{kuzborskij2013stability,wang2015generalization,li2022transfer,tian2022transfer}, to name a few. Additionally, based on \citet{minami2021general}, setting $g(x,y) = (x-\tau y)/(1-\tau)$ and $G(x,y)=(1-\rho)y+\rho x$ allows Algorithm~\ref{algo: HTL} to represent different transfer learning or adaptation algorithms such as posterior ratio estimation \citep{liu2016estimating} and neural network fine-tuning \citep{yosinski2014transferable} when the hyperparameters $\tau$ and $\rho$ belong to some specific regimes.

\section{Spectral algorithms with Gaussian kernels}\label{sec: SA with Gaussians}
This section begins with a proposition for existing misspecified and saturation results, highlighting the limitations that motivate the use of Gaussian kernels. We then establish the minimax optimal convergence rate for spectral algorithms with Gaussian kernels, followed by detailed discussions.

\subsection{Warm-up: misspecification with optimal rate and saturation effect}
We first review the convergence behavior of spectral algorithms with misspecified kernels.
\begin{proposition} \label{proposition: target-only learning}
    For a symmetric and positive semi-definite kernel $K:\mcX \times \mcX \rightarrow \mathbb{R}$, let $\hat{f} \in \mcH_{K}$ be the spectral algorithms estimator in Definition~\ref{def: spectral algorithms} with qualification $\tau$ and regularization parameter $\lambda$, and imposed kernel as $K$. Under Assumption~\ref{assumption: smoothness} and~\ref{assumption: error tail}, the convergence rate of excess risk of $\hat{f}$ is given as follows.
\begin{enumerate}
    \item (Misspecification) Suppose the imposed kernel $K$ satisfies condition~\eqref{eqn: Fourier Transform of reproducing kernel} with order $m' > \frac{d}{2}$, that is, its associated RKHS, $\mcH_{K}$, is norm-equivalent to $H^{m'}$. Furthermore, given $\lambda \asymp n^{- \frac{2m'}{2m + d} }$ and $\gamma = \min\{2\tau, \frac{m}{m'} \}$, we have 
    \[
        \left\| \hat{f} - f^{*} \right\|_{L^{2}}^{2} = O_{\mathbb{P}} \left( n^{-\frac{2\gamma m'}{2\gamma m' + d}} \right).
    \]

    \item (Saturation Effect) If $m' < \frac{m}{2\tau}$, then for any choice of parameter $\lambda(n)$ satisfying $\lambda(n)\rightarrow 0$, we have 
    \[
        \left\| \hat{f} - f^{*} \right\|_{L^{2}}^{2} = \Omega_{\mathbb{P}} \left( n^{-\frac{4 \tau m'}{4\tau m' + d}} \right).
    \]
\end{enumerate}
\end{proposition}

When $m' = m$, the misspecified result recovers the classical well-specified rate \citep{geer2000empirical,caponnetto2007optimal}. The misspecified result is derived by combining (with modifications) Theorems 15 and 16 in \citet{wang2022gaussian} and Theorem 1 in \citet{zhang2024optimality}. The saturation effect is proved by Corollary 3.2 in \citet{li2024generalization}. 

Since the minimax optimal rate for $f^{*} \in H^{m}$ is $n^{-2m/(2m+d)}$, Proposition~\ref{proposition: target-only learning} indicates that the optimal rate is still attainable even when the imposed hypothesis is misspecified,  provided that the smoothness of the kernel is sufficient, i.e., $m' \geq m/2\tau$, and an appropriately chosen $\lambda$. However, if $m' < m/2\tau$, i.e., the regression function $f^{*}$ is much smoother than the hypothesis $\hat{f}$, the saturation effect occurs, meaning there is a persistent gap between the lower bound of the estimator and the information-theoretic lower bound of the learning problem by choosing a less smooth $\mcH_{K}$. While a natural remedy is to use spectral algorithms with higher qualification $\tau$, this comes at the cost of reduced flexibility in algorithm selection and limits the ability to tailor algorithms to different application scenarios. 

This limitation underscores the need for a kernel that can provide universal robustness against both model misspecification and the saturation effect. Formally, for any $f^{*} \in H^{m}$ with $m > d/2$ and arbitrary spectral algorithm, employing the RKHS of this kernel as hypothesis space would ensure that there always exists an optimal $\lambda$ such that the minimax optimal convergence rate $n^{-2m/(2m+d)}$ is attainable.

\subsection{Optimal convergence rates with Gaussian kernels}
We begin by motivating the use of Gaussian kernels. Proposition~\ref{proposition: target-only learning} implies that when the imposed RKHS's smoothness $m'$ exceeds $m/2\tau$, there always exists an optimal $\lambda$ for attaining the optimal rate. This suggests that saturation primarily arises from a large mismatch in smoothness between $f^*$ and $\hat{f}$, i.e., estimating a smooth function with a much less smooth estimator. A natural remedy is to always employ kernels satisfying~\eqref{eqn: Fourier Transform of reproducing kernel} with large $m$, such as high-order Matérn kernels, which ensure $\hat{f}$ lies in a sufficiently smooth Sobolev space. However, precisely selecting the appropriate smoothness for $\hat{f}$ is challenging, as the true smoothness $m$ is typically unknown a priori. To this end, we  consider fixed-bandwidth Gaussian kernels. The motivation is that the RKHS associated with the isotropic Mat\'ern kernel $K_{\nu}$ \citep{stein1999interpolation} is isomorphic to the Sobolev space $H^{\nu + \frac{d}{2}}$, and, more importantly, the Gaussian kernel is the limit of the Mat\'ern kernel, i.e.,
\[
    \lim_{\nu \rightarrow\infty} K_{\nu}(x; h) = \exp \left( - \frac{\|x\|_{2}^{2}}{2h^2} \right). 
\]
The infinite smoothness of the Gaussian kernel makes its RKHS embedded in the Sobolev space $H^{m}$ for any $m > d/2$ \citep{rieger2010sampling,fasshauer2011reproducing}. Thus, by using the Gaussian kernel, one might expect to achieve the optimal rate with a feasible $\lambda$ and avoid the bottleneck that causes saturation since its RKHS consists of functions that are consistently ``smoother'' than any $f \in H^{m}$.

In addition to the assumptions in Section~\ref{subsec: learning problem}, we require a technical condition to obtain the convergence rates for Gaussian spectral algorithms. We introduce the following assumption, presented in two variants to maximize generality and applicability.
\begin{assumption}\label{assumption: assumption on Gaussian kernel}
Suppose that at least one of the following conditions holds:
\begin{enumerate}[label=(\alph*)]
    \item \label{assumption: bounds on effective dimension}
    Let $\mcN(\lambda)$ denote the effective dimension. There exists a constant $E_K > 0$ such that
    \begin{equation}\label{eqn: bounds on effective dimension}
        \mcN_{\infty}(\lambda):=\sup_{x\in \mcX} \sum_{j\ge1}
        \frac{s_j}{s_j+\lambda} e_j^2(x)
        \le E_K^2 \mcN(\lambda).
    \end{equation}

    \item \label{assumption: bounds on adjusted eigenfunction}
    There exists a constant $E_{K} > 0$ and a non-decreasing function $h:[1,\infty) \rightarrow \mathbb{R}$ satisfying (1) $h'(t)\leq c t h(t)$ for $t \ge 1$ and some $c>0$, and (2) $h(x) = O(x^{d-1})$, such that
    \begin{equation}\label{eqn: bounds on adjusted eigenfunction}
        \sup_{x\in \mcX, j\geq 1} | h(j)^{-\frac{1}{2}} \cdot e_{j}(x) | \leq E_{K}.
    \end{equation}
\end{enumerate}
\end{assumption}

Assumption~\ref{assumption: assumption on Gaussian kernel}\ref{assumption: bounds on effective dimension} is a standard capacity condition in the kernel learning literature \citep{lu2024saturation,zhang2025optimal}. It is weaker than the uniform boundedness of eigenfunctions assumed in classical misspecified kernel learning \citep{mendelson2010regularization}. For fixed-bandwidth Gaussian kernels, such conditions are known to hold when $\mcX=\mathbb{S}^{1}$ and $P_X$ is the uniform measure \citep{minh2006mercer}. On general domains, uniform boundedness of the Gaussian eigenfunctions is not guaranteed. In particular, such boundedness does not follow from the explicit Gaussian RKHS orthonormal basis constructed by \citet{steinwart2006explicit}. The arguments that use this basis to infer uniform boundedness of the eigenfunctions are therefore not justified by that construction alone \citep{hagrass2024spectral,hagrass2024spectral1}. We therefore do not assume uniform boundedness. Moreover, when considering particular kernels, additional eigenfunction conditions are often imposed; for example, uniformly bounded eigenfunctions are explicitly assumed for Mat\'ern kernels in \citet{wang2022gaussian}.

While Assumption~\ref{assumption: assumption on Gaussian kernel}\ref{assumption: bounds on effective dimension} is rather standard, it is hard to verify directly on general domains. As an alternative, we introduce Assumption~\ref{assumption: assumption on Gaussian kernel}\ref{assumption: bounds on adjusted eigenfunction}, which is directly verifiable in certain cases. This condition is grounded in the observation that even if  eigenfunctions are not uniformly bounded, their $L^{\infty}$-norm divergence is often controlled by a polynomial rate in the index $j$. By introducing the growth function $h$, one can achieve uniform bounds like \eqref{eqn: bounds on adjusted eigenfunction} on eigenfunctions and further control $\mcN_{\infty}(\lambda)$ in a manner similar to \eqref{eqn: bounds on effective dimension}. We refer readers to Appendix~\ref{apd: proof of SA with Gaussian} for technical details. Specifically, a recent work by \citet{dommel2025approximation} demonstrates that such polynomial $L^{\infty}$-norm growth indeed holds for fixed-bandwidth Gaussians, making this condition satisfied and not merely a theoretical construct. Finally, if the uniform boundedness holds, Assumption~\ref{assumption: assumption on Gaussian kernel}\ref{assumption: bounds on adjusted eigenfunction} reduces to Assumption~\ref{assumption: assumption on Gaussian kernel}\ref{assumption: bounds on effective dimension}.

The following result shows that there exists an optimal order of $\lambda$ such that the spectral algorithms with Gaussian kernels attain the optimal convergence rate $n^{-\frac{2m}{2m+d}}$.
\begin{theorem}[Non-Adaptive Rate]\label{thm: non-adaptive rate of SA with Gaussian}
    Let the imposed kernel, $K$, be the fixed-bandwidth Gaussian kernel, and $\hat{f} = \phi_{\lambda}(T_{K,n}) g_{n}$ be the estimator derived from an arbitrary spectral algorithm. Suppose Assumptions~\ref{assumption: smoothness}, ~\ref{assumption: error tail} and ~\ref{assumption: assumption on Gaussian kernel} hold. By choosing $\log(1/\lambda) \asymp n^{\frac{2}{2m+d}}$, then for any $\delta \in (0,1)$ and sufficiently large $n\geq 1$, with probability at least $1-\delta$, we have
    \begin{equation*}
        \left\| \hat{f} - f^{*} \right\|_{L^{2}}^{2} \leq C \left(\log \frac{4}{\delta} \right)^2  n^{-\frac{2m}{2m + d}},
    \end{equation*}
    where $C$ is a constant independent of $n$ and $\delta$. 
\end{theorem}

The optimal rate in Theorem~\ref{thm: non-adaptive rate of SA with Gaussian} is independent of the qualification of the filter function $\phi_{\lambda}$. This implies that Gaussian spectral algorithms offer a robust solution to avoid the potential saturation effect in spectral algorithms with low qualifications. We emphasize that proving this minimax optimality for fixed-bandwidth Gaussian kernels is nontrivial. The key idea is to leverage the spectral characterization for RKHSs to tightly control the approximation error. We refer readers to Section~\ref{subsec: discussion for SA with Gaussian} and Appendix~\ref{apd: proof of SA with Gaussian} for discussion and technical details.

Despite the optimality in Theorem~\ref{thm: non-adaptive rate of SA with Gaussian}, a key methodological challenge is that the optimal choice of $\lambda$ depends on the unknown smoothness $m$. To address this, we adopt a standard training-validation technique by \citet{steinwart2008support}, and establish an adaptive rate. Let $\mcM = \{m_{\min}< \cdots < m_{\max}\}$ be a finite arithmetic sequence where $m_{i} - m_{i-1}\asymp 1/\logn$ with $m_{\min}>d/2$, $m_{\max}$ large enough such that $m\leq m_{\max}$. We then split dataset $\mcD$ into training $\mcD_{1} := \{(x_{1},y_{1}),\cdots,(x_{j},y_{j})\}$ and validation $\mcD_{2} = \mcD \backslash \mcD_{1}$. The adaptive estimator is obtained by the following procedure:
\begin{enumerate}
    \item (Training): For each $m \in \mcM$, train $\hat{f}_{\lambda_{m}} = \mcA_{K,\lambda_{m}}(\mcD_{1})$ with the $\lambda_{m}$ is chosen as $\lambda = \exp\{ -Cn^{\frac{2}{2m+d}} \}$ for some constant $C$, following Theorem~\ref{thm: non-adaptive rate of SA with Gaussian}.
    \item (Validation): Selecting the estimator $\hat{f}_{\lambda_{m}}$ that minimizes empirical $L^{2}$ error over $\mcD_{2}$ as the adaptive estimator $\hat{f}_{\hat{\lambda}}$.
\end{enumerate}
The following theorem demonstrates that this adaptive estimator attains the minimax optimal rate up to a logarithmic factor.
\begin{theorem}[Adaptive Rate]\label{thm: adaptive rate of SA with Gaussian}
    Suppose the same assumptions in Theorem~\ref{thm: non-adaptive rate of SA with Gaussian} hold. Then, for any $\delta \in (0,1)$, with probability $1-\delta$, we have 
    \begin{equation}\label{eqn: target-only adaptive rate}
         \left\| \hat{f}_{\hat{\lambda}} - f^{*} \right\|_{L^{2}}^{2} \leq C \left( \log\frac{4}{\delta} \right)^2 \left(\frac{n}{\logn} \right)^{-\frac{2m}{2m +d}},
    \end{equation}
    where $C$ is a constant independent of $n$ and $\delta$. 
\end{theorem}

\begin{remark}
    We note that the aforementioned training validation approach is not the exclusive method for achieving adaptivity. Alternative data-driven selection procedures, such as Lepski's method \citep{lepskii1991problem}, can also be employed to derive similar adaptive guarantees. 
\end{remark}

\subsection{Discussion}\label{subsec: discussion for SA with Gaussian}
It is interesting to note that even though the motivation for imposing Gaussian kernels is that the Gaussian kernel is the limit of Mat\'ern kernel $K_{m'}$ as $m'\rightarrow\infty$, the misspecified rates in Proposition~\ref{proposition: target-only learning} cannot simply be extrapolated to recover our findings. Specifically, setting the imposed smoothness parameter $m'$ as infinity in Proposition~\ref{proposition: target-only learning}, the polynomial decay form of $\lambda$ never makes the optimal order of $\lambda$ trackable since $\lim_{m' \rightarrow \infty} n^{-2m'/(2m+d) } = 0$, and only indicates that the optimal $\lambda$ should converge to $0$ much faster than the polynomial form. In contrast, our result explicitly reveals that the optimal order of $\lambda$ should decay exponentially.

This exponential decay form of $\lambda$ originates from the approximation error. The classical analyses of misspecified results, e.g., Proposition~\ref{proposition: target-only learning}, typically rely on the standard real interpolation technique to control the approximation error. Specifically, they use the so-called source condition, which assumes $f^{*}$ lies in the interpolation space of the $\mcH_{K}$. This allows one to expand the $f^{*}$ and the intermediate term $f_{\lambda}$ (see definition in Appendix~\ref{apd: proof of SA with Gaussian}) under the same basis, yielding an approximation error scale as $\lambda^{m/m'}\|f\|_{H^{m}}$ and an optimal $\lambda$ decays polynomially like Proposition~\ref{proposition: target-only learning}. This technique is standard and commonly used for approximation error in misspecified kernel literature; see, for example, \citet{zhang2023optimality,zhang2024optimality,meunier2024optimal}. However, in our case, the intermediate term $f_{\lambda}$ lies in the RKHS of Gaussian kernels while $f^*$ does not, making expanding $f_{\lambda}$ and $f^{*}$ with the same basis no longer feasible. Instead, we leverage the spectra characterization of RKHSs~\eqref{eqn: Spectral Characetrization of RKHS}, which allows us to derive the following bound, with a constant $C$
\begin{equation}\label{eqn: bound for approximation error version 1}
    \inf_{\|f\|_{\mcH_{K}} \leq R} \left( \| f - f^{*} \|_{L^{2}}^{2} + \lambda \|f\|_{\mcH_{K}}^{2} \right) \leq C \left( \log \frac{1}{\lambda} \right)^{-m} \| f^{*} \|_{H^{m}}^{2}.
\end{equation}
Thus, the approximation error decreases like $\log(1/\lambda)^{-m}$, revealing the exponential decay of $\lambda$. We refer readers to the proofs in Appendix~\ref{apd: proof of SA with Gaussian} for technical details.

We now highlight how our spectral characterization technique helps to improve approximation error compared to classical results. For a function that lies in $H^{m}$, the classical norm-constrained best approximation in the Gaussian RKHS in term of $m$ is given by Proposition 1 in \citet{smale2003estimating}, states that $\inf_{\|f\|_{\mcH_{K}} \leq R} \| f - f^{*} \|_{L^{2}}^{2} \leq C_{f^{*}} (\log R)^{-\frac{m}{2}}$, for a constant $\tilde{C}_{f^{*}}$ independent of $R$. Therefore, setting $R=\lambda^{-\frac{1}{3}}$, there exists $\tilde{\lambda}$ such that $\lambda^{\frac{1}{3}} \leq \log(1/\lambda)^{-\frac{m}{2}}$ for all $0<\lambda \leq \tilde{\lambda}$, and hence one can immediately bound the left-hand side of \eqref{eqn: bound for approximation error version 1} by 
\begin{equation}\label{eqn: bound for approximation error version 2}
    \inf_{\|f\|_{\mcH_{K}} \leq R} \left( \| f - f^{*} \|_{L^{2}}^{2} + \lambda \|f\|_{\mcH_{K}}^{2} \right) \leq \tilde{C}_{f^{*}} \left( \log \frac{1}{\lambda} \right)^{-\frac{m}{2}}.
\end{equation}
Consequently, as $\lambda\rightarrow 0$, \eqref{eqn: bound for approximation error version 1} yields a strictly tighter bound for the approximation error than \eqref{eqn: bound for approximation error version 2}. This sharper control is the key for attaining optimal convergence rates as the ``optimal'' $\lambda$ under \eqref{eqn: bound for approximation error version 2} fails to yield the optimal estimation error.

Lastly, to further highlight our findings, we compare them with state-of-the-art results (with slight modifications to align with our setting) that consider general, Mat\'ern, or Gaussian kernels for the spectral algorithms (or their special cases). Table~\ref{table: convergence rate comparison} summarizes these non-adaptive convergence rates. We refer readers to Section~\ref{sec: related work} for a detailed literature review. 

\begin{table*}[tbhp]
\centering
\caption{Comparison of non-adaptive convergence rates. The function $f^{*}$ is assumed to reside in Sobolev space $H^{m}$. ``Imposed RKHS'' means the hypothesis space. ``Type'' specifies the spectral algorithms that are being considered. The parameter $h$ means the bandwidth of the imposed kernel, and entries marked ``$-$'' means the bandwidth is fixed during learning. $\mcH_{K}$ denotes the RKHS of the Gaussian kernel.}
\label{table: convergence rate comparison}
\resizebox{0.8\columnwidth}{!}{%
\renewcommand*{\arraystretch}{1.5}
\begin{tabular}{|c|c|c|c|c|c|} 
\hline
  \multicolumn{1}{|c|}{Paper} & \multicolumn{1}{|c|}{Imposed RKHS}  &  \multicolumn{1}{|c|}{Rate} & \multicolumn{1}{|c|}{$\lambda$} & \multicolumn{1}{|c|}{$h$} & \multicolumn{1}{|c|}{Type} \\
\hline 
\citet{wang2022gaussian}& $H^{m'}, m'>\frac{m}{2}$ & $n^{-\frac{2m}{2m + d}}$ & $n^{-\frac{2m'}{2m + d}}$ & $-$ & KRR \\
\hline 
\citet{zhang2024optimality}& $H^{m'}, m'>\frac{m}{2\tau}$ & $n^{-\frac{2m}{2m + d}}$ & $n^{-\frac{2m'}{2m + d}}$ & $-$ & \text{SA} \\
\hline 
\citet{eberts2013optimal}& $\mcH_{K}$ & $n^{-\frac{2m}{2m + d}+ \xi}, \forall \xi> 0$ & $n^{-1}$ & $n^{-\frac{1}{2m + d} }$ & \text{KRR}\\
\hline
\citet{hamm2021adaptive}& $\mcH_{K}$ & $n^{-\frac{2m}{2m + d}}\log^{d+1}(n)$ & $n^{-1}$ & $n^{-\frac{1}{2m + d} }$ & \text{KRR}\\
\hline
Theorem~\ref{thm: non-adaptive rate of SA with Gaussian} & $\mcH_{K}$ & $n^{-\frac{2m}{2m + d}}$ & $ \exp\{-Cn^{\frac{2}{2m + d}}\} $ & $-$  & \text{SA}\\
\hline
\end{tabular}
}
\end{table*}

The optimal choice of $\lambda$ in this work differs fundamentally from previous attempts that justify the minimax optimality of the misspecified fixed bandwidth Mat\'ern kernel or variable bandwidth Gaussian kernel. Especially in the regime of variable bandwidth, given both $\gamma$ and $\lambda$ decay polynomially in $n$, prior work \citep{eberts2013optimal} shows that the convergence rate could be arbitrarily close to the optimal rate, while the follow-up work \citep{hamm2021adaptive} attained the optimal rate (up to a logarithmic factor) under the so-called DIM condition. However, these variable-bandwidth Gaussian kernel results are not directly comparable to ours due to the difference in model settings and assumptions. For example, both \citet{eberts2013optimal} and \citet{hamm2021adaptive} analyze the broader regularized ERM problems, accommodating loss functions beyond the square error, and they also assumed data distributions are bounded. Moreover, their analysis is built on empirical-process arguments for regularized ERM with a bandwidth shrinking with $n$; this approach does not carry over to fixed-bandwidth Gaussian kernels used within arbitrary spectral algorithms. In contrast, our analysis focuses on spectral algorithms under squared loss with the moment-based error assumption.

\section{Robust transfer learning under concept shift}\label{sec: transfer learning}
In this section, we first leverage the stronger robustness presented by Gaussian spectral algorithms to achieve robust transfer learning under concept shift. We then theoretically analyze the minimax optimality of this robust transfer procedure and discuss some of its insights into transfer learning and adaptation.


\subsection{Methodology}\label{subsec: methodology for HTL}
A key observation of Algorithm~\ref{algo: HTL} is that the learning process can be decoupled into learning the source function $f^{P}$ and the intermediate function $f^{\delta}$ separately. $f^{P}$ and $f^{\delta}$ are estimated sequentially using independent datasets, $\mcD^{P}$ and $\mcD^{\delta}$, and  spectral algorithms, $\mcA^{P}$ and $\mcA^{\delta}$, respectively. Consequently, achieving robust transfer boils down to achieving the same goal in learning $f^{P}$ and $f^{\delta}$ individually. Leveraging the adaptive results from Theorem~\ref{thm: adaptive rate of SA with Gaussian}, we obtain a robust transfer estimator, via Algorithm~\ref{algo: HTL}, as:
\[
    \hat{f}^{Q} = G\left ( \hat{f}_{\hat{\lambda}_{\delta}}^{\delta}, \hat{f}_{\hat{\lambda}_{P}}^{P} \right). 
\]
Here, components $\hat{f}_{\hat{\lambda}_{P}}^{P} = \mcA_{K,\hat{\lambda}_{P}}^{P}(\mcD^{P})$ and $\hat{f}_{\hat{\lambda}_{\delta}}^{\delta} = \mcA_{K,\hat{\lambda}_{\delta}}^{\delta}(\mcD^{\delta})$ are obtained using Gaussian spectral algorithms, with regularization parameters $\lambda_{P}$ and $\lambda_{\delta}$ selected via the training-validation. We note that the algorithms $\mcA^{P}$ and $\mcA^{\delta}$ are not required to be the same, unlike previous frameworks that often restrict both phases to be the same algorithm \citep{du2017hypothesis,lin2024smoothness}.

\subsection{Optimal Convergence Rate Analysis}\label{subsection: convergence rate of RAHTL}
We begin by formalizing the relationship between $f^{P}$ and $f^{Q}$ through the intermediate function $f^{\delta}$. Let $g$ and $G$ be the data and model transformation functions defined in Algorithm~\ref{algo: HTL} respectively. The intermediate shift function $f^{\delta}$ is defined as the expected transformation, 
\begin{equation}\label{eqn: expected transformation}
    f^{\delta} (x) = \E_{Q}[ g(Y, f^{P}(x)) \mid X = x ].
\end{equation}
Thus, the target function $f^{Q}$ can be recovered via model transformation $G(f^{\delta}(x), f^{P}(x)) := f^{Q}(x) = \E_{Q}[Y|X = x]$. This formulation guarantees consistency, ensuring that Step 3 of Algorithm~\ref{algo: HTL} produces an unbiased estimator for $f^{\delta}$.

To establish the convergence rates under concept shift, we introduce the following assumptions. The first assumption specifies the regularity of the functions and the tail behavior of the noise, extending Assumptions~\ref{assumption: smoothness} and~\ref{assumption: error tail} to the concept shift setting.
\begin{assumption}\label{assumption: parameter space for concept shift}
    For the functions $f^{P}$ and $f^{Q}$, and noise $\epsilon_{i}^{P}$ and $\epsilon_{i}^{Q}$, suppose
    \begin{itemize}
        \item there exist positive constants $m_{P}, m_{Q}, m_{\delta}$ and $R_{P}, R_{Q}, R_{\delta}$ such that:
        \[ f^{P} \in H^{m_{P}}(R_{P}), \quad f^{Q} \in H^{m_{Q}}(R_{Q}), \quad \text{and} \quad f^{\delta} \in H^{m_{\delta}}(R_{\delta}); \]

        \item for each domain $t \in \{P,Q\}$, there exist constants $\sigma^{t}, L^{t} > 0$ such that the noise $\epsilon^{t}$ satisfies the moment conditions in Assumption~\ref{assumption: error tail}.
    \end{itemize}
\end{assumption}
This assumption indicates the parameter space for the concept shift problem is
\[
    \Theta(R_{P},R_{\delta},m_{P},m_{\delta})  =  \left\{  (P, Q):  \left\|f^{P}\right\|_{H^{m_{P}}}\leq R_{P},  \left\|f^{\delta}\right\|_{H^{m_{\delta}}} \leq R_{\delta}  \right\}.
\]
The constant $R_{\delta}$ mathematically quantifies the similarity between the source and target functions as a smaller $R_{\delta}$ signifies a higher degree of proximity between $f^{P}$ and $f^{Q}$. Imposing such an upper bound on model discrepancy is essential for establishing rigorous optimality results in various concept shift settings. Analogous constraints appear in various settings, such as $\ell^{1}$ or $\ell^{0}$ distance in high-dimensional setting \citep{li2022transfer,tian2022transfer}, Fisher-Rao distance in low-dimensional setting \citep{zhang2022class}, RKHS distance in functional setting \citep{lin2024on}.

Our next assumption specifies the regularity of the transformation functions $g$ and $G$. These conditions are standard in the HTL literature \citep{du2017hypothesis} and facilitate translating the convergence of the intermediate estimators into convergence of the target estimator.

\begin{assumption}\label{assumption: assumption on g and G}
    The data transformation function, $g$, and model transformation function, $G$, satisfy the following conditions:
    \begin{enumerate}
        \item (Invertibility) the function $G(z, y)$ is invertible with respect to its first argument $z$, i.e., for any fixed $y$, the inverse is given by $g(\cdot, y)$, such that $G(g(u, y), y) = u$;
        \item (Lipschitz continuity) the model transformation function $G$ is $L_{G}$-Lipschitz, i.e., for all $(z,y), (z', y') \in \mbR^2$, $|G(z,y) - G(z', y')| \leq L_{G} \sqrt{(z-z')^2 + (y-y')^2}$.
        Additionally, the data transformation function $g(z, y)$ is $L_{g}$-Lipschitz with respect to its second argument, i.e., $|g(z,y) - g(z, y')| \leq L_{g} |y - y'|$.
    \end{enumerate}
\end{assumption}

The Lipschitz continuity condition guarantees $\hat{f}^{Q}$ converges to $f^{Q}$ if both $\hat{f}^{P}$ and $\hat{f}^{\delta}$ converge their counterparts. Meanwhile, the invertibility condition assures the injection between the values of the target function $f^{Q}$ and the intermediate function $f^{\delta}$.

Our first result specifies the minimax lower bound, which elucidates the information-theoretic difficulty of the transfer learning problem.
\begin{theorem}[Lower Bound]\label{thm: lower bound of OTL}
    Suppose Assumption~\ref{assumption: parameter space for concept shift} holds. For any $\delta \in (0,1)$, with probability at least $1-\delta$, the following lower bound holds
    \begin{equation*}
        \inf_{\tilde{f}}  \sup_{\Theta(R_{P},R_{\delta},m_{P},m_{\delta})}  \mathbb{P} \left\{ \left\| \tilde{f} - f^{Q} \right\|_{L^{2}}^2  \geq C_{1} \delta  \left( R_{P}^{2} \cdot n_{P}^{-\frac{2m_{P}}{2m_{P}+d}} + R_{\delta}^{2} \cdot n_{Q}^{-\frac{2m_{\delta}}{2m_{\delta}+d}}  \right) \right\} \geq 1 - \delta,
    \end{equation*}
    where $C_{1}$ is a constant independent of $n_{P}$, $n_{Q}$, $R_{P}$, $R_{\delta}$, and $\delta$. The infimum is taken over all possible estimators $\tilde{f}$ constructed from samples $\mcD^{P}$ and $\mcD^{Q}$. 
\end{theorem}

The next theorem provides an upper bound on the excess risk, demonstrating that our estimator achieves the minimax optimal rate (up to logarithmic factors).
\begin{theorem}[Upper Bound]\label{thm: upper bound of OTL}
    Suppose Assumptions~\ref{assumption: assumption on Gaussian kernel}, \ref{assumption: parameter space for concept shift} and \ref{assumption: assumption on g and G} hold, and $n_{P}$ and $n_{Q}$ are sufficiently large but still in the transfer learning regime ($n_{P} \gg n_{Q}$). Denote the $\hat{f}^{Q}$ as the estimator. Then, for any $\delta \in (0,1)$, with probability at least $1-\delta$, we have
    \begin{equation}\label{eqn: RAHTL upper bound}
        \left\| \hat{f}^{Q} - f^{Q} \right\|_{L^{2}}^{2}   \leq C  \left(\log \frac{12}{\delta}\right)^2 
         \ \left\{ C_{1}\cdot \left(\frac{n_{P}}{\log n_{P}}\right)^{-\frac{2m_{P}}{2m_{P}+d}} + C_{2} \cdot \left(\frac{n_{Q}}{\log n_{Q}}\right)^{-\frac{2m_{\delta}}{2m_{\delta}+d}}  \right\},
    \end{equation} 
    where constants $C_{1} \propto ( \|f^{P}\|_{H^{m_{P}}}^{2} + \sigma_{P}^{2} )$, $C_{2} \propto ( \|f^{\delta}\|_{H^{m_{\delta}}}^{2} + \sigma_{Q}^{2} )$, and $C$ is a constant independent of $n_{P}$, $n_{Q}$, $R_{P}$, $R_{\delta}$. Furthermore, replacing the specific norms $\|f^{P}\|_{H^{m_{P}}}$ and $\|f^{\delta}\|_{H^{m_{\delta}}}$ with their respective upper bound $R_{P}$ and $R_{\delta}$ yields a uniform excess risk bound over $\Theta$.
\end{theorem}

Theorem~\ref{thm: upper bound of OTL} indicates that the convergence rate of the estimator $\hat{f}^{Q}$ consists of two components: the first term is the pre-training error which represents the error of learning $f^{P}$ with the source samples, while the second term is the fine-tuning error, which arises from learning the intermediate function $f^{\delta}$ with the constructed samples. 
\begin{remark}
By exploiting the Lipschitz continuity of the model transformation function $G$, the excess risk of $\hat{f}^{Q}$ can be bounded by $\|\hat{f}^{P} - f^{P}\|_{L^{2}}^2 + \|\hat{f}^{\delta} - f^{\delta}\|_{L^{2}}^2$. At first glance, it might appear that the upper bound could be derived by simply applying Theorem~\ref{thm: adaptive rate of SA with Gaussian} to these two terms individually. However, there is actually an additional error, stemming from $\|\hat{f}^{\delta} - f^{\delta}\|_{L^{2}}^2$, induced by using the estimated source function $\hat{f}^{P}$ rather than using the true $f^{P}$ to construct the intermediate labels $\{y_{i}^{\delta}\}_{i=1}^{n_{Q}}$. 

When both $\mcA^{P}$ and $\mcA^{\delta}$ are specified as KRR, the prior work \citep{du2017hypothesis} showed that this error is an amplified version of the pre-training error, which is proportional to the product of the pre-training error and a monotonic increase function of $n_{Q}$, i.e., $n_{Q}^{2/{2m_{P}+d}}\log(n_{Q}) \cdot \|\hat{f}^{P} - f^{P}\|_{L^{2}}^2$. A similar conclusion is reached in \citet{wang2016nonparametric} by using the algorithmic stability technique for KRR and is claimed to be nearly tight. However, for sufficiently large but fixed $n_{P}$ and growing $n_{Q}$, such bounds suggest that the excess risk of $\hat{f}^{Q}$ will explode. This is counterintuitive since when there is a sufficiently good pre-trained estimator $\hat{f}^{P}$, the estimated intermediate label will also be close to the true intermediate label, leading to sufficiently good $\hat{f}^{\delta}$.

In contrast, our results show that the error induced by using $\hat{f}^{P}$ to construct $y_{i}^{\delta}$ is bounded directly by the pre-training error $\|\hat{f}^{P} - f^{P}\|_{L^{2}}^2$; see Theorem~\ref{thm: bounds on intermediate error} for details. This ensures the excess risk of $\hat{f}^{Q}$ receives no amplification and is minimax optimal. Notably, \citet{lin2024smoothness} obtained a similar result in the KRR case by using the property of Tikhonov filter function. However, directly extending those arguments to our current setting is not feasible, which necessitates a different treatment. To address this, we bound this error using the operator techniques adapted from \citet{smale2007learning}. For further details, we refer to Appendix~\ref{apd: proof of transfer learning}. 
\end{remark}

\subsection{Discussion}\label{subsection: RAHTL discussion}
We now discuss the optimal rate to provide some insights to better understand transfer learning and adaptation under concept shift. Specifically, we explore how the source samples help to improve the learning efficiency over the target domain and identify what factors contribute to better efficiency. These insights are derived by comparing the following two rates of convergence:
\begin{itemize}
    \item the target-only learning rate $\psi_{n_{Q}}(Q)$ from Theorem~\ref{thm: adaptive rate of SA with Gaussian} where
    \[
    \psi_{n_{Q}}(Q):= \left(\frac{n_{Q}}{\log n_{Q}}\right)^{-\frac{2m_{Q}}{2m_{Q}+d}};
    \]

    \item the HTL rate $\psi_{n_{P},n_{Q}}(P,Q)$ from Theorem~\ref{thm: upper bound of OTL}, which re-parametrizes \eqref{eqn: RAHTL upper bound} as
    \[
    \psi_{n_{P},n_{Q}}(P,Q):= \left(\frac{n_{P}}{\log n_{P}}\right)^{-\frac{2m_{P}}{2m_{P}+d}} +  \xi \cdot \left(\frac{n_{Q}}{\log n_{Q}}\right)^{-\frac{2m_{\delta}}{2m_{\delta}+d}}
    \]
    where $\xi := ( \|f^{\delta}\|_{H^{m_{\delta}}}^{2} + \sigma_{Q}^{2} ) / ( \|f^{P}\|_{H^{m_{P}}}^{2} + \sigma_{P}^{2})$.
\end{itemize}



However, it is hard to directly compare these rates without explicitly understanding the relationship between $m_{Q}$, $m_{P}$, and $m_{\delta}$. To facilitate the comparison, we derive the following proposition, which characterizes how the target function $f^Q$ inherits Sobolev regularity from the components $f^P$ and $f^\delta$ under the model transformation $G$.

\begin{proposition}\label{prop: smoothness of target}
    Suppose Assumption~\ref{assumption: parameter space for concept shift} holds. 
    Further, suppose the model transformation function $G$ belongs to the smooth function class 
    $C^{\max(\lceil m_{P} \rceil, \lceil m_{\delta} \rceil)}$. 
    Then, for all $x \in \mcX$,
    \[
    f^{Q}(x) := G(f^{\delta}(x), f^{P}(x)) \in H^{\min(m_{P}, m_{\delta})}(\mcX).
    \]
\end{proposition}

Proposition~\ref{prop: smoothness of target} implies that the target function $f^Q$ belongs to $H^{\min(m_P, m_\delta)}$. Since $m_Q$ in Assumption~\ref{assumption: parameter space for concept shift} denotes a Sobolev index such that $f^Q \in H^{m_Q}$, Proposition~\ref{prop: smoothness of target} shows that we may take $m_Q = \min(m_P, m_\delta)$ as a valid choice for the purpose of rate comparison. In special cases where $G$ takes an explicit form, such as the simplest offset transfer learning with $G(x,y)=x+y$, if $m_\delta \ge m_P$, we have $m_Q = m_P$ for the purpose of rate comparison, recovering the condition used in \citet{lin2024smoothness}. However, when $G$ lacks an explicit form, it is unavoidable to require additional assumptions relating $m_Q$ and $m_\delta$, such as $m_\delta \ge m_Q$, if one would like to compare these rates \citep{du2017hypothesis}. 


Theorem~\ref{thm: upper bound of OTL} indicates the benefit gained from Algorithm~\ref{algo: HTL} depends jointly on the source sample size $n_{P}$ and the ratio $\xi$, a quantity that plays a pivotal role in the comparison. We define a phase transition point $\xi^{*}$ as 
\[
    \xi^{*}:=(n_{Q}/\log n_{Q})^{\frac{2m_{\delta}}{2m_{\delta}+d}} \cdot (n_{P}/\log n_{P})^{-\frac{2m_{P}}{2m_{P}+d}}.
\]
When $\xi$ is smaller than $\xi^{*}$ (up to constant), the pre-training error dominates, and the convergence rate $\psi_{n_{P},n_{Q}}(P,Q)$ is strictly faster than the target-only one $\psi_{n_{Q}}(Q)$ given $n_{P} \gg n_{Q}$. Even in the regime where $\xi > \xi^{*}$ and thus the fine-tuning error dominates, the convergence rate $\psi_{n_{P},n_{Q}}(P,Q)$ is still at least as tight as $\psi_{n_{Q}}(Q)$ since $m_{Q}\leq m_{\delta}$ by Proposition~\ref{prop: smoothness of target}.

Next, we discuss $\xi$ as its magnitude controls the phase transition. To the best of our knowledge, the form of $\xi$ is the first form that links the signal strength from intermediate and source functions to explain the transition in the HTL literature. Previous works, such as \citet{wang2015generalization,wang2016nonparametric,du2017hypothesis}, overlooked the existence of $\xi$ in the upper bound and thus were unable to provide a concrete explanation of how model discrepancy, i.e., the signal strength of $f^{\delta}$, governs the phase transition. Some recent works advance this by introducing the signal strength of intermediate models $\|f^{\delta}\|_{H^{\delta}}$ (or its upper bound $R_{\delta}$) to measure the domain discrepancy in the same way as we do, see \citet{li2022transfer,tian2022transfer} and more references therein. These works only identified $\xi$ is proportional to the shift-strength, i.e., having $\xi \propto \|f^{\delta}\|_{H^{\delta}}^{2}$ (or $\xi \propto R_{\delta}^2$) in the convergence rate $\psi_{n_{P},n_{Q}}(P,Q)$ in our context. They assert that a smaller shift-strength leads to higher efficiency. Our analysis refines this by showing that efficiency is driven not by the absolute shift but by the ratio $\xi$. A smaller shift-strength does not necessarily imply smaller discrepancy if the source signal itself is weak even for some simple $G$. We illustrate this by a simple but concrete example.

\begin{example}[An illustration of the role of $\xi$ in offset transfer learning]  
    We consider the offset transfer learning, where $f^{Q} = f^{P} + f^{\delta}$, which implies $m_{P} = m_{Q}$, and non-zero source and offset functions. To eliminate the impact of noise, we also consider the noiseless setting, and thus $\xi$ simplifies $\xi \propto \|f^{\delta}\|_{H^{m_{\delta}}}^2 / \|f^{P}\|_{H^{m_{P}}}^2$. The quantity $\xi$ can be interpreted as shift-to-signal ratio. In the noisy scenario, one can extract the factor $(1 + \sigma_{P}^{2} \vee (\sigma_{Q}^{2}/\tilde{\xi}) )$ from \eqref{eqn: RAHTL upper bound} to achieve the exact same form. 

    \begin{figure}[tbhp]
        \centering
        \subfloat[]{\label{fig: angle figure}
            \centering
            \begin{tikzpicture}
                \filldraw[fill=gray!30, draw=black, thick] (0,0) circle (1.5cm);
        
                \draw[->, thick, red] (-4,-1.5) -- node [above] {$f_{1}^{P}$} (0,0) coordinate (p1);
                \draw[->, thick, red] (-4,-1.5) -- node [below] {$f_{1}^{Q}$} (0,-1.5) coordinate (p2);
                \coordinate (p3) at (-4,-1.5);
                \pic [draw, -, "$\theta_{1}$", angle eccentricity=1.5, angle radius=10mm, red] {angle = p2--p3--p1};
                
                \draw[->, thick, blue] (-1.5,-1.5) -- (0,0)  coordinate (p11);
                \node[blue] at (-1.15, -0.75) {$f_{2}^{P}$};
                \draw[->, thick, blue] (-1.5,-1.5) -- node[below] {$f_{2}^{Q}$} (0,-1.5) coordinate (p22);
                \coordinate (p33) at (-1.5,-1.5);
                \pic [draw, -, "$\theta_{2}$", angle eccentricity=1.8, angle radius=5mm, blue] {angle = p22--p33--p11};
        
                \draw[dashed, thick, black] (0,-1.5) -- node [right] {$f^{\delta}$} (0,0);
                \draw[thick, black] (0,-1.3) -- (-0.2,-1.3);
                \draw[thick, black] (-0.2,-1.3) -- (-0.2,-1.5);
            \end{tikzpicture}
        }
        \hspace{2cm}
        \subfloat[]{\label{fig: both angle and distance}
            \centering
            \begin{tikzpicture}
                \draw[->, thick, blue] (0,0) -- node [above] {$f^{P}$} (2,2) coordinate (p1);
                \draw[->, thick, red] (0,0) -- node [below] {$f_{1}^{Q}$} (2.828,0) coordinate (p2);
                \draw[dashed, thick, black] (2,2) -- node [right] {$f_{1}^{\delta}$} (2.828,0);
            
                \draw[->, thick, red] (0,0) -- node [below] {$f_{2}^{Q}$} (4,3) coordinate;
                \draw[dashed, thick, black] (2,2) -- node [above] {$f_{2}^{\delta}$} (4,3);
            \end{tikzpicture}
        }
        \caption{Geometric illustration for how $\xi$ will affect the transfer efficiency. The length of the lines represents the magnitude of $\|f^{P}\|_{H^{m_{P}}}$, $\|f^{Q}\|_{H^{m_{P}}}$ and $\|f^{\delta}\|_{H^{m_{\delta}}}$, respectively. (a) The circle represents a ball centered around the $f^{P}$ with radius $\|f^{\delta}\|_{H^{m_{\delta}}}$. A key observation is $\theta = \arcsin (\|f^{P} - f^{Q}\|_{H^{m_{\delta}}} /\|f^{P}\|_{H^{m_{P}}} )$. (b) $f^{P}$ and $f_{1}^{Q}$ possess the same magnitude but a rather large angle while $f^{P}$ and $f_{2}^{Q}$ possess a smaller angle but their magnitude differs.}
    \end{figure}

    Consider two source-target pairs $(P_{1},Q_{1})$ and $(P_{2},Q_{2})$ sharing the same $m_{P}$, $m_{\delta}$ and $f^{\delta}$ but different $f^{P}$ (and thus $f^{Q}$); see Figure~\ref{fig: angle figure}. Despite sharing an identical offset, these two pairs of ($f^{P}$, $f^{Q}$) exhibit different angles  $\theta_{i} = \arcsin( \|f^{\delta}\|_{H^{m_{\delta}}}/\|f_{i}^{P}\|_{H^{m_{P}}} )$, which thus represent varying degrees of model discrepancy. While conventional formulations $\xi \propto \|f^{\delta}\|_{H^{m_{\delta}}}^2$ implies the Algorithm~\ref{algo: HTL} achieves same learning efficiency on both two pairs, our convergence rate indicates the pair $(P_{1}, Q_{1})$ has a higher efficiency, and aligns with the geometric intuition that $f_{1}^{P}$ and $f_{1}^{Q}$ are more similar due to their smaller relative angle.

    This example highlights that the ratio $\xi$ is critical in characterizing the efficiency. To facilitate interpretation, we further refine how different geometric factors affect the magnitude of $\xi$. With $m:= m_{P} = m_{Q} = m_{\delta}$, we can express the ratio $\xi$ as
    \[
        \xi = \frac{\|f^{P} - f^{Q}\|_{H^{m}}^2}{\|f^{P}\|_{H^{m}}^2} = 1 + \frac{\|f^{Q }\|_{H^{m}}^2}{\|f^{P }\|_{H^{m}}^2} - 2  \frac{\|f^{Q }\|_{H^{m}}}{\|f^{P }\|_{H^{m}}} \cos(\angle_{H^{m}}(f^{P }, f^{Q })),
    \]
    which incorporates both relative signal strength $\|f^{Q }\|_{H^{m}}/\|f^{P}\|_{H^{m}}$ and the angle $\angle_{H^{m}}(f^{P}, f^{Q})$ between $f^{P}$ and $f^{Q}$ in $H^{m}$ into $\xi$. This expression is an unimodal function of signal strength $\|f^{Q}\|_{H^{m}}/\|f^{P}\|_{H^{m}}$, reaching its minimum at $1$. Simultaneously, it is a monotonically increasing function of angle $\angle(f^{P}, f^{Q})$, with its minimum value when the angle is $0$. This means that if the relative signal strength of $f^{P}$ and $f^{Q}$ are closer, and the angle between $f^{P}$ and $f^{Q}$ is smaller, $\xi$ decreases, resulting in higher transfer efficiency; see Figure~\ref{fig: both angle and distance}. This behavior aligns with the fact that $\xi$ achieves its minimum when $f^{P}$ and $f^{Q}$ coincide exactly. 

\end{example}
While this work focuses on nonparametric regression, the principle in the example discussed above generalizes to most offset transfer learning for various statistical models, e.g., \citet{scholkopf2001generalized, kuzborskij2013stability, li2022transfer}. The explanation becomes more complex for general model transformation functions $G$ due to the absence of an explicit form of $G$, making it challenging to analyze the precise components that influence the efficiency. In such cases, we can only conclude that the efficiency depends on the shift-to-signal ratio, which still refines the results in previous literature that ignored or misjudged $\xi$.

\section{Experiments}\label{sec: numerical experiments}

\subsection{Experiments for spectral algorithms}

\begin{figure}[!t]
    \centering
    \subfloat[Non-adaptive]{
        \centering
        \includegraphics[width=0.48\linewidth]{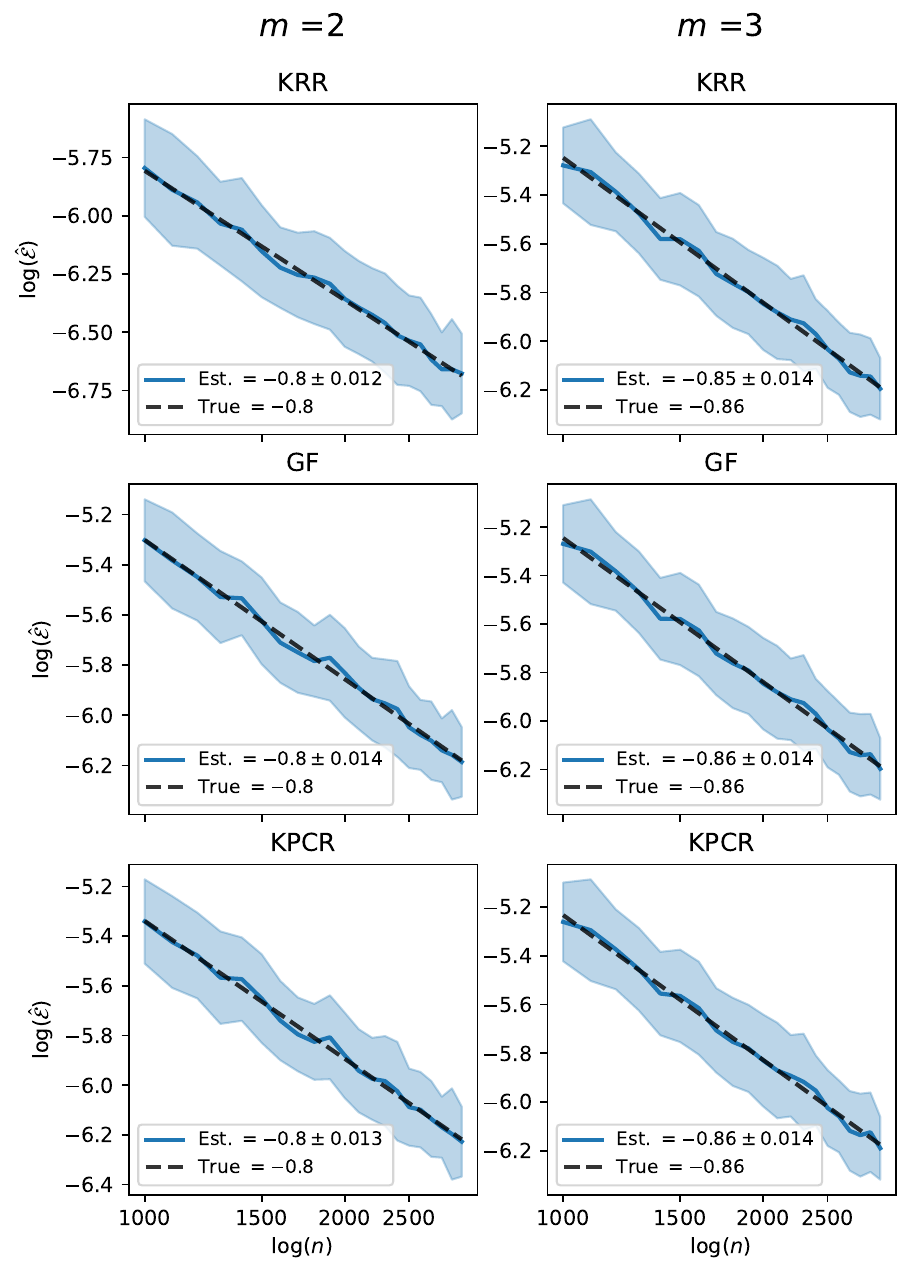} 
    }
    \subfloat[Adaptive]{
        \centering
        \includegraphics[width=0.48\linewidth]{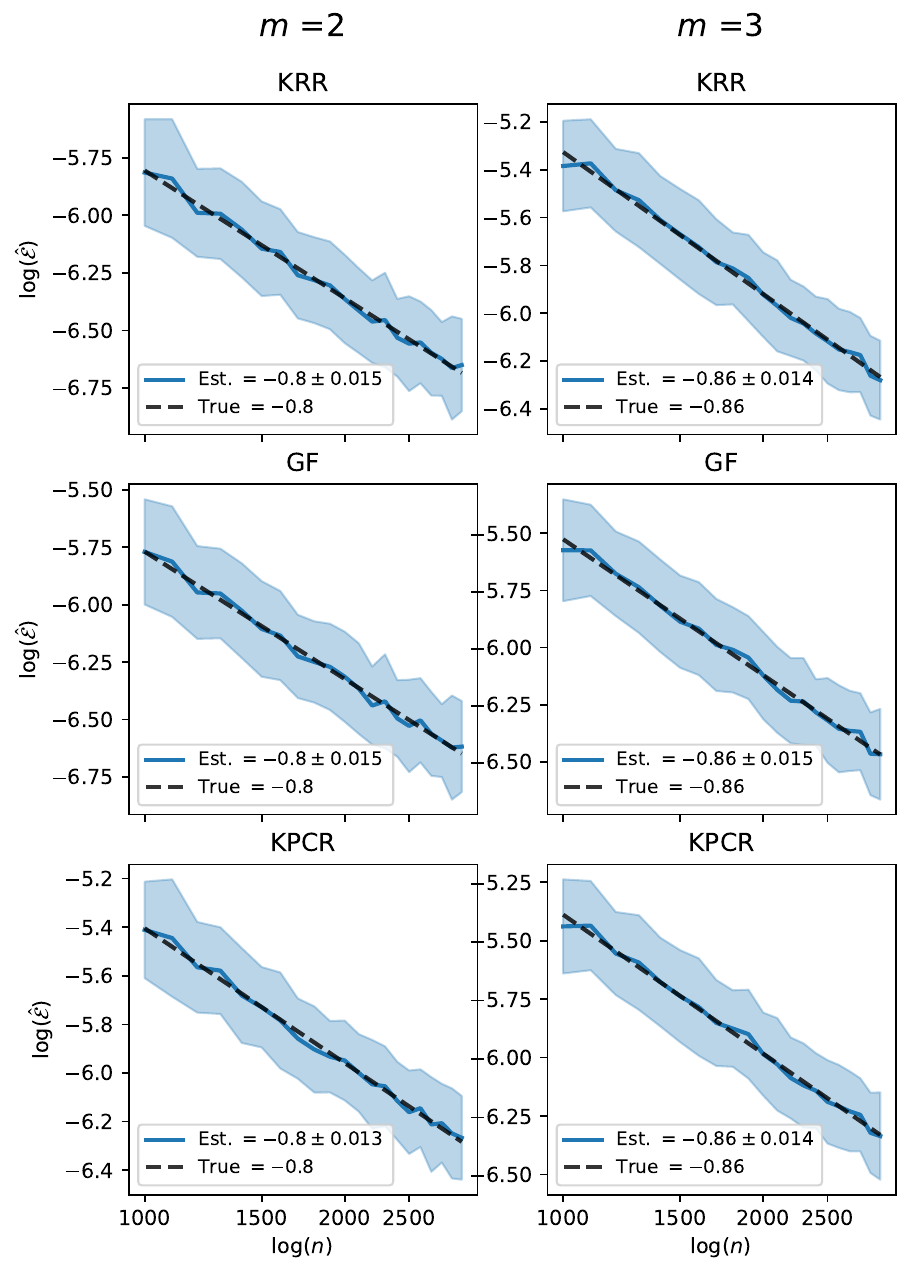} 
    }
    \caption{Error decay curves of spectral algorithms with Gaussian kernels with best $C$. Both axes are plotted on a log scale. The dashed black lines denote the theoretical regression line of $\log \mcE$ on $\log n$ with slope $-\frac{2m}{2m+1}$, denoted by ``True''. Blue curves denote the average empirical excess risk over repeated trials, with shaded regions indicating $\pm 1$ standard deviation. "Est." denotes the estimated slope of the regression line $\pm$ its standard error.} 
    \label{fig: Gaus_SA_best_slop}
\end{figure}

We first aim to empirically validate the minimax optimality of the non-adaptive and adaptive rates established in Theorems~\ref{thm: non-adaptive rate of SA with Gaussian} and \ref{thm: adaptive rate of SA with Gaussian}. Specifically, we consider a regression problem on the domain $\mcX = [0,1]$ with the following data generation process
\[
    y_{i} = f^{*}(x_{i}) + 0.5\epsilon_{i} \qquad i = 1, \cdots, n
\]
where $x_{i}$ are i.i.d. samples drawn from the uniform distribution on $\mcX$, $\epsilon_{i}$ are i.i.d. standard Gaussian random variables, and sample size $n=100k$ with $k=1,2,\cdots,30$. 

To generate ground truth $f^{*}$ with specific Sobolev regularity, we leverage the correspondence between Gaussian Processes (GP) and Sobolev spaces. Since the sample path of GP with specified covariate kernels indeed lies in certain Sobolev spaces \citep{kanagawa2018gaussian}, we generate true functions $f^{*}$ by setting them as sample paths that are generated from a Gaussian process with isotropic Mat\'ern kernels $K_{\nu}$. Based on Corollary 4.15 in \citet{kanagawa2018gaussian}, we set $\nu = 2.01$ and $\nu = 3.01$ to generate $f^{*}$ in $H^{2}(\mcX)$ and $H^{3}(\mcX)$.

We investigate three spectral algorithms, including KRR, GF, and KPCR, to recover $f^{*}$. For non-adaptive rates, we set the regularization parameter $\lambda = \exp\{-Cn^{\frac{2}{2m + 1}}\}$ with a fixed $C$. For adaptive rates, we set the candidate smoothness as $[1,2,\cdots,5]$ and construct the estimators through the training and validation as described in Section~\ref{sec: SA with Gaussians}. For each combination of $n$ and $m$, we repeat the experiments $100$ times and compute $\hat{\mcE}_{i} = \|\hat{f} - f^{*}\|_{L^{2}}^{2}$ by Simpson's rule with $5000$ testing points. The excess risk is then approximated by $\hat{\mcE} = \frac{1}{100}\sum_{i=1}^{100}\hat{\mcE}_{i}$. To verify the convergence rate of the error is sharp, we regress the $\log(\hat{\mcE})$ on $\log(n)$ and compare the regression coefficient to its theoretical counterpart $-2m/(2m+1)$.

We test different values of $C$ in sequence $[0.05,0.1,\cdots, 4]$, and report the optimal curve for the best choice of $C$ in Figure~\ref{fig: Gaus_SA_best_slop} for non-adaptive and adaptive rates, respectively. It can be seen that the corresponding empirical error decay curves align with the theoretical decay curves in both non-adaptive and adaptive cases for all three spectral algorithms. Also, the estimated regression coefficient closely agrees with the theoretical exponent $-\frac{2m}{2m+1}$. From these results, it can be concluded that the spectral algorithms with fixed-bandwidth Gaussian kernels attain the minimax optimal convergence rates when $f^{*}$ belongs to Sobolev spaces. Additionally, we also show that the optimal rates do not tie to a specific choice of $C$ but hold for a wide range of $C$. We refer readers to Appendix~\ref{apd: additional experiments} for more results.



\subsection{Experiments for transfer learning under concept shift}
In this section, we empirically validate the minimax optimality derived in Theorems~\ref{thm: lower bound of OTL} and \ref{thm: upper bound of OTL}, as well as the role of $\xi$ that we discussed in Section~\ref{subsection: RAHTL discussion}. We focus on the setting with the following data generation process:
\begin{equation*}
    \{x_{i}^{Q}\}_{i=1}^{n_{Q}},  \{x_{i}^{P}\}_{i=1}^{n_{P}} \stackrel{i.i.d.}{\sim} U([0,1]), \quad 
    y_{i}^{P} = f^{P}(x_{i}^{P}) + 0.5 \epsilon_{i}, \quad 
    y_{i}^{Q} = y_{i}^{P} + f^{\delta}(x_{i}^{Q}), 
\end{equation*}
where $\epsilon$ are standard Gaussian noise. We set the smoothness such that $f^{P} \in H^{1}$ and $f^{\delta} \in H^{m_{\delta}}$ with $m_{\delta} \in \{2, 3\}$, making the target $f^{Q}$ has smoothness $m_{Q} = 1$. Throughout this part, we compare the excess risk of two learning strategies
\begin{itemize}
    \item Target-Only: We use KRR with Gaussian kernels to recover $f^{Q}$ with only the target dataset $\mcD^{Q}$. This serves as the non-transfer baseline.
    
    \item Transfer: The robust learning procedure introduced in Section~\ref{subsec: methodology for HTL}. To mimic practical scenarios, the pre-training phase conducts gradient methods, i.e., GF, to recover $f^{P}$ while the fine-tuning phase uses KRR for $f^{\delta}$.
\end{itemize}

\paragraph{Optimal transfer learning rates}
First, we verify the minimax optimal rates of convergence. We set $n_{P} = n_{Q}^{1.5}$ while varying $n_{Q} \in \{40,45,\cdots,150\}$. We also adjust the Sobolev norm of $f^{P}$ and $f^{\delta}$ such that the factor $\xi$ belongs to $\{0.25, 0.5, 1\}$. Such settings ensure the convergence rate of excess risk for the algorithm is dominated by the fine-tuning error, i.e., $\|\hat{f}^{Q} - f^{Q}\|_{L^{2}}^{2} = \mathcal{O}_{\mathbb{P}}( n_{Q}^{-\frac{3m_{P}}{2m_{P}+1}} + \xi n_{Q}^{-\frac{2m_{\delta}}{2m_{\delta}+1}} ) = \mathcal{O}_{\mathbb{P}}( n_{Q}^{-1} + \xi n_{Q}^{-\frac{2m_{\delta}}{2m_{\delta}+1}} ) = \mathcal{O}_{\mathbb{P}}( n_{Q}^{-\frac{2m_{\delta}}{2m_{\delta}+1}} )$. Therefore, optimality can thus be confirmed by verifying that the empirical excess risk decay rate matches the theoretical rates $n_{Q}^{-\frac{2m_{\delta}}{2m_{\delta}+1}}$. The excess risks under different settings are presented in Figure~\ref{fig: transfer with vary n_Q}, where each colored curve denotes the average error over $100$ repeated experiments. We can see the empirical excess risks of the HTL algorithm, i.e., the blue curves, consistently aligned with the theoretical counterparts, i.e., the black dashed curves, and also significantly outperform the target-only baselines in all settings. Moreover, the risk for $m_{\delta} = 3$ is lower and decays faster than those for $m_{\delta} = 2$, which aligns with the theoretical indications. These results justify the optimality of the HTL algorithm we derived in Section~\ref{sec: transfer learning}.


\begin{figure}[ht]
    \centering
    \subfloat[]{\label{fig: transfer with vary n_Q}
        \centering
        \includegraphics[width=0.48\linewidth]{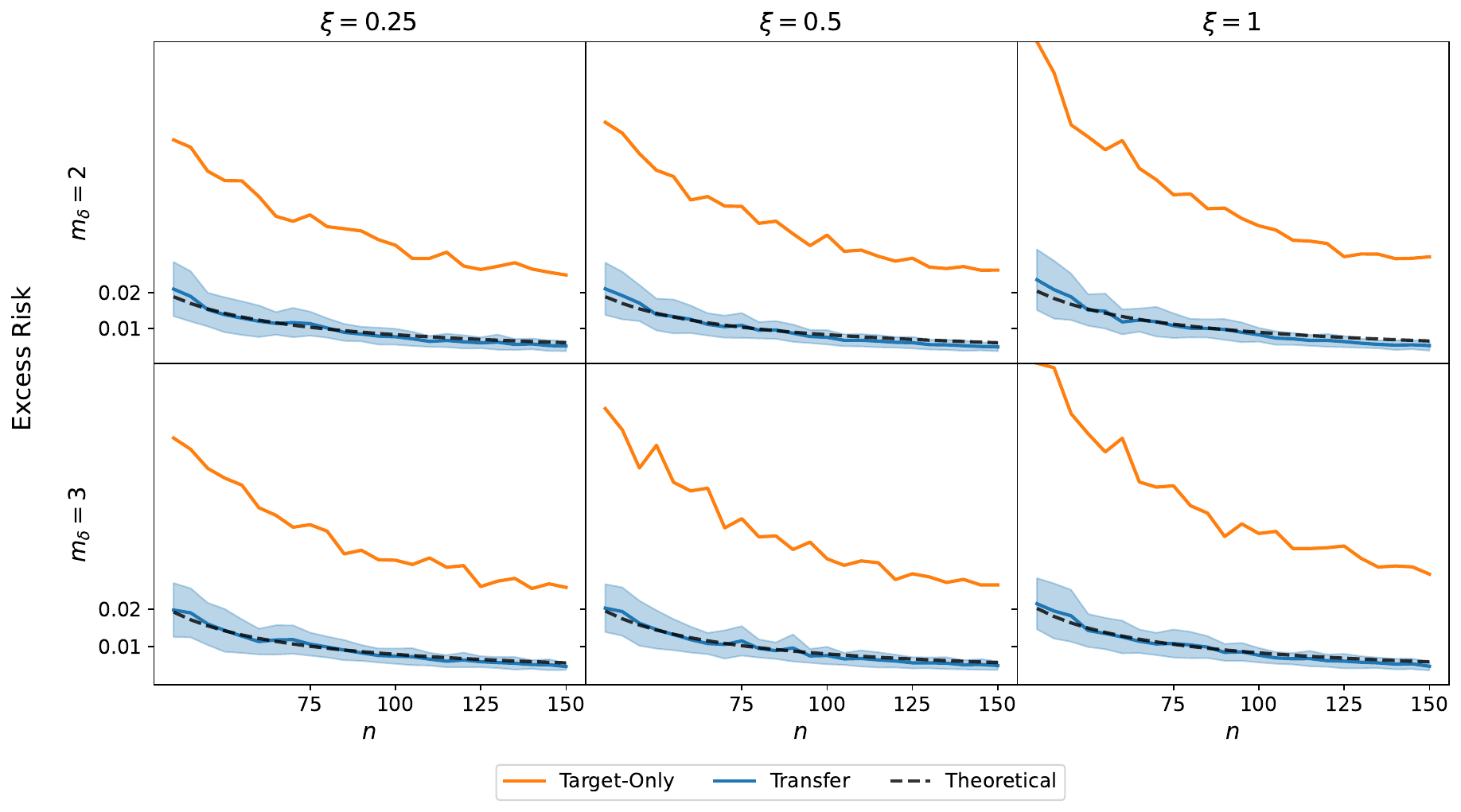} 
    }
    \subfloat[]{\label{fig: transfer with fixed n_Q}
        \centering
        \includegraphics[width=0.48\linewidth]{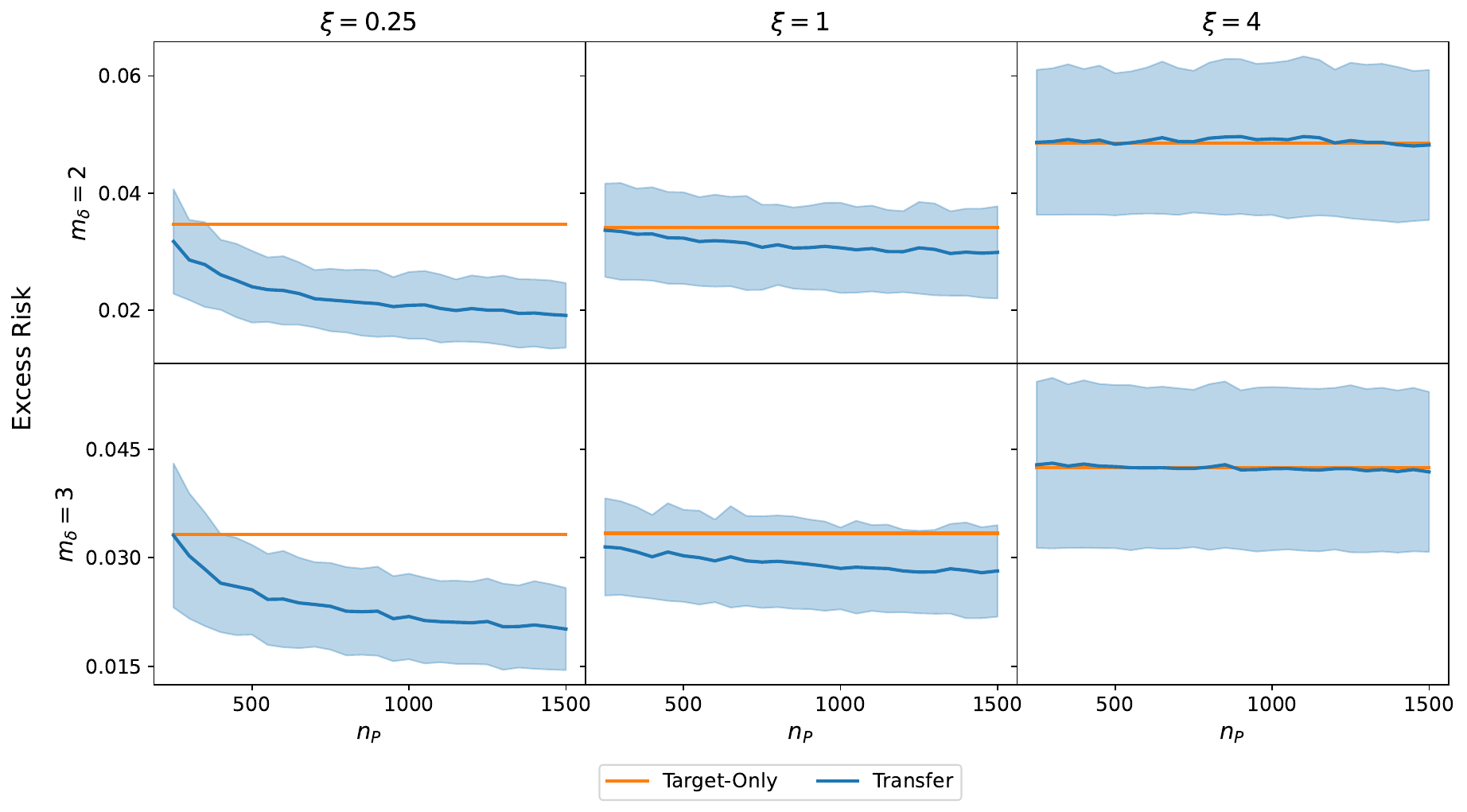} 
    }
    \caption{Left: Excess risk under different combinations of $\xi$ and $m_{\delta}$ with varied target sample size $n_{Q}$. The theoretical convergence rate is $C n_{Q}^{-\frac{2m_{\delta}}{2m_{\delta}+1}}$ for some constants $C$. Right: Excess risk under different $\xi$ and $m_{\delta}$ with a fixed target sample size $n_{Q}$. } 
\end{figure}

\paragraph{Influence of $\xi$ on the efficiency} As discussed in Section~\ref{subsection: RAHTL discussion}, the magnitude of $\xi$ affects the efficiency. To confirm such influence of $\xi$, we fix $n_{Q} = 200$ while varying $n_{P}$ from $200$ to $1500$, and set the values of $\xi$ in $\{0.25, 1, 4\}$. Under such settings, it is expected that as $\xi$ increases, the fine-tuning error will start to dominate and remain unchanged, even as $n_{P}$ keeps increasing. The results, presented in Figure~\ref{fig: transfer with fixed n_Q}, confirm this behavior by showing that the decrease in the excess risks of the HTL algorithm diminishes as $\xi$ increases. Specifically, for low $\xi$ ($\xi=0.25$), increasing $n_P$ continues to reduce the total risk, as the pre-training error remains a significant component, while for high $\xi$ the error curve becomes horizontal, indicating that increasing source samples provides no further benefit because the bottleneck is the estimation of the shift $f^{\delta}$. These observations align with the influence of $\xi$ discussed in Section~\ref{subsection: RAHTL discussion}.


\section{Conclusion}\label{sec: discussion}

In this work, we establish the convergence of spectral algorithms with fixed-bandwidth Gaussian kernels for nonparametric regression. We show that when the true functions belong to Sobolev spaces, specifying the hypothesis space as the Gaussian RKHS allows spectral algorithms to produce estimators that attain minimax optimal convergence rates. A critical discovery is that this optimality necessitates an exponential decay of the regularization parameter $\lambda$, a departure from the polynomial decay typically seen in misspecified kernel literature. Importantly, this result offers a general theoretical insight for misspecified spectral algorithms and a remedy to overcome the well-known saturation effect, independent of the specific choice of learning algorithms.

By leveraging these robust results, we address the underexplored problem of transfer learning under concept shift with the misspecified models. We proposed a two-stage robust and adaptive learning procedure that leverages Gaussian spectral algorithms to achieve minimax optimal rates. Our theoretical analysis reveals that transfer efficiency is governed not just by sample size, but by a ratio factor largely overlooked in prior work.

These findings in this work present a foundational step toward rigorous, robust learning under classical nonparametric regression and under distribution shifts with statistical guarantees. It opens up several directions for further research, including further the refinement of Gaussian kernel assumptions, the exploration of learning dynamics for non-additive model transformations, and the incorporation of covariate shift into the learning framework.

\appendix

\section{Proofs of Main Results for Gaussian Spectral Algorithms}\label{apd: proof of SA with Gaussian}
\subsection{Proof of Theorem~\ref{thm: non-adaptive rate of SA with Gaussian}}
In addition to the notations defined in Section~\ref{sec: preliminaries}, we denote the operator norm of a bounded linear operator as $\|\cdot\|_{op}$ and denote the effective dimension of $T_{K}$ as $\mcN(\lambda) =\operatorname{tr}( (T_{K} + \lambda \mathbf{I})^{-1} T_{K} )$. 

The proof of the non-adaptive rate is mainly consistent with approximation-estimation error decomposition and control of each of them, respectively. We first state the approximation-estimation error decomposition. The estimator of spectral algorithms can be written as 
\begin{equation*}
    \hat{f} = \phi_{\lambda}(T_{K,n}) g_{n}.
\end{equation*}
Then, we decompose the excess risk into approximation error and estimation error, i.e. 
\begin{equation*}
\underbrace{\left\| \hat{f} - f^{*} \right\|_{L^{2}} }_{\text{excess risk of $\hat{f}$}} \leq \underbrace{\left\| \hat{f} - f_{\lambda} \right\|_{L^{2}}}_{\text{estimation error}} + \underbrace{\left\| f_{\lambda} - f^{*} \right\|_{L^{2}}}_{\text{approximation error}},
\end{equation*}
where the intermediate term $f_{\lambda}$ is defined as follows,
\begin{equation*}
    f_{\lambda} = \phi_{\lambda}(T_{K}) T_{K}(f^{*})= \phi_{\lambda}(T_{K}) g.
\end{equation*}
We now start to control each error separately. The following two theorems are used, which aim to bound the approximation error and estimation error, respectively. 

\begin{theorem}[Approximation error]\label{thm: approximation error for SA with Gaussian}
    Suppose $f_{\lambda}$ is defined as $f_{\lambda} = \phi_{\lambda}(T_{K}) T_{K}(f^{*})$.
    Then, the following inequality holds,
    \[
    \| f_{\lambda} - f^{*} \|_{L^{2}}^{2} \leq C  \log\left(\frac{1}{\lambda}\right)^{-m} \| f^{*} \|_{H^{m}}^{2},
    \]
    where $C$ is a constant.
\end{theorem}
\begin{remark}
As we mentioned in Section~\ref{subsec: discussion for SA with Gaussian}, the intermediate term $f_{\lambda}$ is placed in $[H^{m}]^{s}$ while the true function in $H^{m}$ (here, $s$ denotes the interpolation index) for classical misspecified kernel methods. This allows one to expand $f_{\lambda}$ and $f^{*}$ under the same basis. In our case, the intermediate term, $f_{\lambda}$, lies in the RKHS associated with Gaussian kernels ($s$ needs to be $\infty$). Therefore, the previous technique is no longer feasible. Therefore, the techniques we used are the Fourier transform of the Gaussian kernel and the Plancherel Theorem.
\end{remark}

\begin{proof}[Proof of Theorem~\ref{thm: approximation error for SA with Gaussian}]
    Since $\mcX$ has Lipschitz boundary, there exists an extension mapping from $L^{2}(\mcX)$ to $L^{2}(\mathbb{R}^{d})$, such that the smoothness of functions in $L^{2}(\mcX)$ get preserved. Therefore, there exist constants $C_{1}$ and $C_{2}$  such that for any function $g \in H^{m}(\mcX)$, there exists an extension of $g$, $g_{e}\in H^{m}(\mathbb{R}^{d})$ satisfying 
    \begin{equation}
        C_{1} \|g_{e}\|_{H^{m}(\mathbb{R}^{d})} \leq \|g\|_{H^{m}(\mcX)} \leq C_{2} \|g_{e}\|_{H^{m}(\mathbb{R}^{d})}.
    \end{equation}
    Denote $f_{\lambda,e}$ as the extension of $f_{\lambda}$, then we have
    \begin{equation}\label{equ: bias proof eqn 1}
        \left\| f_{\lambda} - f^{*} \right\|_{L^{2}(\mcX)}^{2}   \leq 
        \left\| f_{\lambda, e}|_{\mcX} - f^{*} \right\|_{L^{2}(\mcX)}^{2}   \leq C_{2} \left( \left\| f_{\lambda,e} - f_{0,e} \right\|_{L^{2}(\mathbb{R}^{d})}^{2}  \right).
    \end{equation}
    where $f_{\lambda,e}|_{\mcX}$ is the restriction of $f_{\lambda,e}$ on $\mcX$. By Fourier transform of the Gaussian kernel and Plancherel Theorem, we have 
    \begin{equation}
    \begin{aligned}\label{equ: bias proof eqn 2}
        & \left\|f_{\lambda, e}-f_{0,e}\right\|_{L^{2}(\mathbb{R}^{d})}^2 \\
        = & \int_{\mathbb{R}^{d}}\left|\mathcal{F}\left(f_{0,e}\right)(\omega)-\mathcal{F}\left(f_{\lambda, e}\right)(\omega)\right|^2 d \omega \\
         = & \int_{\mathbb{R}^{d}} |\mathcal{F}\left(f_{0,e}\right)(\omega)|^2 \big[ 1 - \phi_{\lambda}\left( \mcF(K)(\omega) \right) \mcF(K)(\omega) \big]^2  d \omega \\ 
         \leq & C_{3}\int_{\mathbb{R}^{d}} \lambda \exp\left\{C\|\omega\|_{2}^{2}\right\} |\mathcal{F}\left(f_{0,e}\right)(\omega)|^2 d \omega \\
        \leq & C_{3} \int_{\mathbb{R}^{d}} \lambda \exp\left\{C(1+\|\omega\|_{2}^{2})\right\} |\mathcal{F}\left(f_{0,e}\right)(\omega)|^2 d \omega
    \end{aligned}
    \end{equation}
    where $\mcF(K) = \exp\{-C\|\omega\|^2\}$ is the Fourier transform of the Gaussian kernel, and the first equality is based on the property of the filter function with $\beta = 0.5$. To apply the property of the filter function, we also need to verify the Fourier transform of the Gaussian kernel ranges from $0$ to $\kappa^2$. For $x,x^{'} \in \mcX$, let $\xi = \|x - x^{'}\|$ denote the Euclidean distance between $x$ and $x'$, then for $\omega \in \mbR$
    \begin{equation}
    \begin{aligned}
        \mcF(K)(\omega) & = \frac{1}{(\sqrt{2\pi})^{d}} \int_{\mbR^{d}} \exp\left\{- \frac{\xi^2}{\gamma}\right\} \exp\left\{-i \omega \xi\right\} d\xi \\
        & = \frac{1}{(\sqrt{2\pi})^{d}} \prod_{j=1}^{d} \int_{\mbR} \exp\left\{- \frac{\xi_{j}^2}{\gamma}\right\} \exp\left\{-i \omega_{j} \xi_{j}\right\} d\xi_{j} \\
        & = \frac{1}{(\sqrt{2\pi})^{d}} \prod_{j=1}^{d} \int_{\mbR} \exp\left\{- \frac{\xi_{j}^2}{\gamma}\right\} \cos\left( \omega_{j} \xi_{j} \right) d\xi_{j} \\
        & \leq \left(\sqrt{\frac{2}{\pi}}\right)^{d}\prod_{j=1}^{d} \int_{\mbR^{+}}  \cos\left( \omega_{j} \xi_{j} \right) d\xi_{j} \\
        & \leq \left(\sqrt{\frac{2}{\pi}}\right)^{d} \leq 1 \leq \kappa^{2}
    \end{aligned}
    \end{equation}
    where the last inequality is based on the fact that $\sup_{x}K(x,x)\leq 1$ for Gaussian kernels.

    Define $\Omega= \{ \omega : \lambda \exp\{C(1+\|\omega\|_{2}^{2})\} <1  \}$ and $\Omega^{C} =  \mathbb{R}^{d} \backslash \Omega$. Notice over $\Omega^{C}$, we have
    \begin{equation}
        (1+\|\omega\|_{2}^{2}) \geq \frac{1}{C}\log\left( \frac{1}{\lambda} \right) \implies C^{m} \log\left( \frac{1}{\lambda} \right)^{-m} (1+\|\omega\|_{2}^{2})^{m} \geq 1.
    \end{equation}
    Besides, over $\Omega$, we first note that the function $h(\omega) = \exp\{ C(1+\|\omega\|_{2}^{2})\} / (1+\|\omega\|_{2}^{2})^{m}$ reaches its maximum $C^{m} \lambda^{-1} \log(\frac{1}{\lambda})^{-m}$ if $\lambda$ satisfies $\lambda < \exp\{-m\}$ and $\lambda \log(\frac{1}{\lambda})^{m} \leq C^{m} \exp\{-C\}$. One can verify when $\lambda \rightarrow 0$ as $n\rightarrow \infty$, the two previous inequality holds. Then
    \begin{equation}
        \lambda \exp\{C(1+\|\omega\|_{2}^{2})\} \leq C^{m} \log\left( \frac{1}{\lambda} \right)^{-m} (1+\|\omega\|_{2}^{2})^{m} \quad \forall \omega \in \Omega.
    \end{equation}
    Combining the inequality over $\Omega$ and $\Omega^{C}$,
    \begin{equation}
    \begin{aligned}\label{equ: bias proof eqn 3}
        & \int_{\mathbb{R}^{d}} \lambda \exp\{C(1+\|\omega\|_{2}^{2})\} |\mathcal{F}\left(f_{0,e}\right)(\omega)|^2 d \omega \\
        \leq & \int_{\Omega} \lambda \exp\{C(1+\|\omega\|_{2}^{2})\} |\mathcal{F}\left(f_{0,e}\right)(\omega)|^2 d \omega + \int_{\Omega^{C}} |\mathcal{F}\left(f_{0,e}\right)(\omega)|^2 d \omega  \\
        \leq & C^{m} \log\left( \frac{1}{\lambda} \right)^{-m} \int_{\mathbb{R}^{d}} (1+\|\omega\|_{2}^{2})^{m} |\mcF(f_{0,e})(\omega)|^2 d \omega \\
         = & C^{m} \log\left( \frac{1}{\lambda} \right)^{-m} \|f_{0,e}\|_{H^{m}(\mathbb{R}^{d})}^{2} \\
        \leq & C^{'} \log\left( \frac{1}{\lambda} \right)^{-m} \|f^{*}\|_{H^{m}(\mcX)}^{2}.
    \end{aligned} 
    \end{equation}
    Combining the inequality (\ref{equ: bias proof eqn 1}), (\ref{equ: bias proof eqn 2}), and (\ref{equ: bias proof eqn 3}), we have 
    \begin{equation}
        \|f_{\lambda} - f^{*}\|_{L^{2}(\mcX)}^2 \leq C' \log\left( \frac{1}{\lambda} \right)^{-m} \|f^{*}\|_{H^{m}(\mcX)}^{2}.
    \end{equation}
    for some constant $C'$, which completes the proof.
\end{proof}

As discussed in Section~\ref{sec: SA with Gaussians}, we require either Assumption~\ref{assumption: assumption on Gaussian kernel}~\ref{assumption: bounds on effective dimension} or Assumption~\ref{assumption: assumption on Gaussian kernel}~\ref{assumption: bounds on adjusted eigenfunction} hold to bound the estimation error. In the following, we consider the case where Assumption~\ref{assumption: assumption on Gaussian kernel}~\ref{assumption: bounds on effective dimension} hold, and refer to Section~\ref{apd: Bound of estimation error under adjusted eigenfunction}.

\begin{theorem}[Estimation error]\label{thm: estimation error for SA with Gaussian}
    Suppose Assumptions~\ref{assumption: smoothness},~\ref{assumption: error tail} and~\ref{assumption: assumption on Gaussian kernel}~\ref{assumption: bounds on effective dimension} hold. Then by choosing $\log(1/\lambda)\asymp n^{\frac{2}{2m+d}}$, for any fixed $\delta \in (0,1)$, when $n$ is sufficient large, with probability $1-\delta$, we have 
    \begin{equation*}
        \left \| \hat{f} - f_{\lambda}  \right\|_{L^{2}} \leq C \log\left(\frac{4}{\delta}\right)  n^{-\frac{m}{2m+d}}
    \end{equation*}
    where $C$ is a constant proportional to $\sigma$. 
\end{theorem}

\begin{remark}
    For estimation error, we apply standard integral operator techniques from \citet{smale2007learning}, following a strategy similar to \citet{fischer2020sobolev,zhang2024optimality}. Unlike most existing work focusing on polynomial eigenvalue decay, we refine the proof to address the Gaussian kernel case, where the eigenvalues decay exponentially.
\end{remark}

\begin{proof}[Proof of Theorem~\ref{thm: estimation error for SA with Gaussian}]
First, decompose the estimation error as
\begin{equation*}
\begin{aligned}
        \left\| \hat{f} - f_{\lambda} \right\|_{L^{2}} & = \left\| T_{K}^{\frac{1}{2}} \left(  \hat{f} - f_{\lambda} \right) \right\|_{\mcH_{K}} \\
        & \leq \underbrace{\left\|T_{K}^{\frac{1}{2}} \left( T_{K} + \lambda I  \right)^{-\frac{1}{2}} \right\|_{op}}_{A_{1}}  \cdot  \underbrace{\left\| \left( T_{K} + \lambda I  \right)^{\frac{1}{2}} \left( T_{K,n} + \lambda I  \right)^{-\frac{1}{2}}  \right\|_{op}}_{A_{2}}  \\
        & \quad \cdot \underbrace{\left\| \left( T_{K,n} + \lambda I  \right)^{\frac{1}{2}} \left( \hat{f} - f_{\lambda} \right)  \right\|_{\mcH_{K}}}_{A_{3}}
\end{aligned}
\end{equation*}
For the first term $A_{1}$, we have 
\begin{equation*}
    A_{1} = \left\|T_{K}^{\frac{1}{2}} \left( T_{K} + \lambda I  \right)^{-\frac{1}{2}} \right\|_{op} = \sup_{j\geq 1} \left( \frac{s_{j}}{s_{j} + \lambda}\right)^{\frac{1}{2}} \leq 1.
\end{equation*}
For the second term $A_{2}$, using Lemma~\ref{lemma: bounds for A2} with sufficient large $n$, we have 
\begin{equation*}
    u:=\frac{\mcN(\lambda)}{n} \log ( \frac{8\mcN(\lambda)}{\delta} \frac{(\|T_{K}\|_{op}+\lambda)}{\|T_{K}\|_{op}}) \leq \frac{1}{8}
\end{equation*}
such that
\begin{equation*}
     \left\| \left( T_{K} + \lambda I \right)^{-\frac{1}{2}} \left( T_{K} - T_{K,n} \right) \left( T_{K} + \lambda I \right)^{-\frac{1}{2}} \right\|_{op} \leq \frac{4}{3}u + \sqrt{2u}\leq  \frac{2}{3}
\end{equation*}
holds with probability $1-\frac{\delta}{2}$. Thus,
\begin{equation*}
\begin{aligned}
A_{2}^{2} &= \left\| (T_{K} + \lambda I)^{\frac{1}{2}} (T_{K,n} + \lambda I)^{-\frac{1}{2}}\right\|_{op}^2 \\
& = \left\| (T_{K} + \lambda I)^{\frac{1}{2}} (T_{K,n} + \lambda I)^{-1} (T_{K} + \lambda I )^{\frac{1}{2}}\right\|_{op}\\
& =\left\|\left( (T_{K} + \lambda I)^{-\frac{1}{2}}\left(T_{K,n}+\lambda\right) (T_{K} + \lambda I) ^{-\frac{1}{2}}\right)^{-1}\right\|_{op} \\
& =\left\|\left(I-  (T_{K} + \lambda I )^{-\frac{1}{2}}\left(T_{K,n}-T_{K}\right) (T_{K} + \lambda I )^{-\frac{1}{2}}\right)^{-1}\right\|_{op} \\
& \leq \sum_{k=0}^{\infty}\left\| (T_{K} + \lambda I )^{-\frac{1}{2}}\left(T_{K} - T_{K,n}\right) (T_{K} + \lambda I )^{-\frac{1}{2}}\right\|_{op}^k \\
& \leq \sum_{k=0}^{\infty}\left(\frac{2}{3}\right)^k \leq 3,
\end{aligned}
\end{equation*}
For the third term $A_{3}$, notice 
\begin{equation*}
    \hat{f} - f_{\lambda} = \phi_{\lambda}(T_{K,n})\left[ g_{n} - T_{K,n}(f_{\lambda}) \right] - \psi_{\lambda}(T_{K,n})(f_{\lambda})
\end{equation*}
where $\psi_{\lambda}(z) = 1 - z\phi_{\lambda}(z)$. Therefore,
\begin{equation*}
\begin{aligned}
    A_{3} & = \left\| \left( T_{K,n} + \lambda I  \right)^{\frac{1}{2}} \left[ \hat{f} - f_{\lambda} \right] \right\|_{\mcH_{K}} \\
    & \leq \underbrace{  \left\| \left( T_{K,n} + \lambda I  \right)^{\frac{1}{2}} \phi_{\lambda}(T_{K,n}) \left[ g_{n} - T_{K,n}(f_{\lambda}) \right] \right\|_{\mcH_{K}}}_{A_{31}} + \underbrace{\left\| \left( T_{K,n} + \lambda I  \right)^{\frac{1}{2}} \psi_{\lambda}(T_{K,n}) f_{\lambda} \right\|_{\mcH_{K}}}_{A_{32}} 
\end{aligned}
\end{equation*}
For $A_{31}$, we have
\begin{equation}\label{eqn: equation of A_31}
\begin{aligned}
    A_{31} \leq &  \left\| \left( T_{K,n} + \lambda I \right)^{\frac{1}{2}} \phi_{\lambda}(T_{K,n}) \left( T_{K,n} + \lambda I \right)^{\frac{1}{2}}  \right\|_{op} \cdot \left\| \left( T_{K,n} + \lambda I \right)^{-\frac{1}{2}} \left( T_{K} + \lambda I \right)^{\frac{1}{2}}  \right\|_{op} \\
    & \quad  \cdot \left\| \left( T_{K} + \lambda I \right)^{-\frac{1}{2}} \left[ g_{n} - T_{K,n}(f_{\lambda})\right] \right\|_{\mcH_{K}}
\end{aligned}
\end{equation}
\begin{itemize}
    \item For the first term in (\ref{eqn: equation of A_31}), the properties of filter function indicate $z\phi_{\lambda}(z)\leq E$ and $\lambda \phi_{\lambda}(z)\leq E$, thus we have
    \begin{equation*}
         \left\| \left( T_{K,n} + \lambda I \right)^{\frac{1}{2}} \phi_{\lambda}(T_{K,n})  \left( T_{K,n} + \lambda I \right)^{\frac{1}{2}}  \right\|_{op} =  \left\| \left( T_{K,n} + \lambda I \right)^{1} \phi_{\lambda}(T_{K,n})  \right\|_{op} \leq 2E.
    \end{equation*}

    \item For the second term in (\ref{eqn: equation of A_31}), the bounds for $A_{2}$ implies 
    \begin{equation*}
        \left\| \left( T_{K,n} + \lambda I \right)^{-\frac{1}{2}} \left( T_{K} + \lambda I \right)^{\frac{1}{2}}  \right\|_{op} \leq \sqrt{3}.
    \end{equation*}

    \item For the third term in (\ref{eqn: equation of A_31}), by choosing $\log(1/\lambda) \asymp n^{\frac{2}{2m+d}}$,
    \begin{equation*}
    \begin{aligned}
        & \left\| \left( T_{K} + \lambda I \right)^{-\frac{1}{2}} \left[ g_{n} - T_{K,n}(f_{\lambda})\right] \right\|_{\mcH_{K}} \\
        & \leq \left\| \left( T_{K} + \lambda I \right)^{-\frac{1}{2}} \left[ (g_{n} - T_{K,n}(f_{\lambda})) - (g - T_{K}(f_{\lambda}))\right] \right\|_{\mcH_{K}} \\
        & \qquad + \left\| \left( T_{K} + \lambda I \right)^{-\frac{1}{2}} \left[  g - T_{K}(f_{\lambda})\right] \right\|_{\mcH_{K}} \\
        & \leq C \log\left(\frac{4}{\delta}\right)n^{-\frac{m}{2m+d}} + \left\| \left( T_{K} + \lambda I \right)^{-\frac{1}{2}} T_{K} \right\|_{op} \cdot \left\|f^{*} - f_{\lambda} \right\|_{L^{2}} \\
        & \leq C \log\left(\frac{4}{\delta}\right)n^{-\frac{m}{2m+d}} + C^{'} n^{-\frac{m}{2m+d}} \left\| f^{*} \right\|_{H^{m}} \\
        & \leq C \log\left(\frac{4}{\delta}\right)n^{-\frac{m}{2m+d}}
    \end{aligned}
    \end{equation*}
    where $\left\| \left( T_{K} + \lambda I \right)^{-\frac{1}{2}} T_{K} \right\|_{op}\leq 1$ is the same as bounding $A_{1}$, the second inequality is based on Theorem~\ref{thm: bounds for A3}, and the third inequality is based on estimation error.
\end{itemize}

Combining all results, we have 
\begin{equation*}
    A_{31} \leq C \log\left(\frac{4}{\delta}\right)n^{-\frac{m}{2m+d}}
\end{equation*}
with probability $1-\frac{\delta}{2}$. 

Turning to $A_{32}$, with the properties of filter function, we have 
\begin{equation*}
    \left\| \left( T_{K,n} + \lambda I  \right)^{\frac{1}{2}} \psi_{\lambda}(T_{K,n})\right\|_{op}  \leq \sup_{z \in\left[0, \kappa^2\right]}(z + \lambda)^{\frac{1}{2}}\psi_{\lambda}(z)  \leq \sup_{z \in\left[0, \kappa^2\right]}(z^{\frac{1}{2}} + \lambda^{\frac{1}{2}})\psi_{\lambda}(z)  \leq 2 F_{\tau} \lambda^{\frac{1}{2}}.
\end{equation*}
Thus,
\begin{equation*}
\begin{aligned}
    A_{32} & \leq \left\| \left( T_{K,n} + \lambda I  \right)^{\frac{1}{2}} \psi_{\lambda}(T_{K,n})\right\|_{op} \cdot \left\| \phi_{\lambda}(T_{K})T_{K}(f^{*}) \right\|_{\mcH_{K}} \\
    & \leq 2 F_{\tau} \lambda^{\frac{1}{2}} \left\| \phi_{\lambda}(T_{K})T_{K} \right\|_{op} \left\| f^{*} \right\|_{L^{2}}\\
    & \leq 2 F_{\tau} \lambda^{\frac{1}{2}} E \left\| f^{*} \right\|_{H^{m}}.
\end{aligned}
\end{equation*}
Notice having $\lambda \asymp \exp\{-Cn^{\frac{2}{2m+d}}\}$, we have $A_{32} = o(A_{31})$. Finally, combining the bounds on $A_{1},A_{2},A_{3}$, we finish the proof.
\end{proof}

\begin{proof}[Proof of Theorem~\ref{thm: non-adaptive rate of SA with Gaussian}]
For approximation error, applying Theorem~\ref{thm: approximation error for SA with Gaussian} leads to
\begin{equation*}
    \|f_{\lambda} - f^{*}\|_{L^{2}}^{2} \leq C_{1} \log(\frac{1}{\lambda})^{-m} \|f^{*}\|_{H^{m}}^{2}. 
\end{equation*}
Then selecting $\log(1/\lambda) \asymp n^{\frac{2}{2m+d}}$ leads to 
\begin{equation*}
        \|f_{\lambda} - f^{*}\|_{L^{2}}^{2} \leq C_{1} n^{-\frac{2m}{2m+d}} \|f^{*}\|_{H^{m}}^{2}. 
\end{equation*}
For estimation error, applying Theorem~\ref{thm: estimation error for SA with Gaussian} provides, with probability $1-\delta$, 
\begin{equation*}
    \left \| \hat{f} - f_{\lambda}  \right\|_{L^{2}}^{2} \leq C_{2} \left(\log \frac{4}{\delta} \right)^{2}  n^{-\frac{2m}{2m+d}}.
\end{equation*}
Combining both results finishes the proof, i.e., we have
\begin{equation*}
    \left \| \hat{f} - f^{*}  \right\|_{L^{2}}^{2} \leq C_{3} \left(\log \frac{4}{\delta} \right)^{2}  n^{-\frac{2m}{2m+d}},
\end{equation*}
with probability $1-\delta$, and $C_{3} \propto \sigma^{2} + \|f^{*}\|_{H^{m}}^{2}$
\end{proof}

\subsection{Proof of Theorem~\ref{thm: adaptive rate of SA with Gaussian}}
\begin{proof}
    To simplify the notation, for a given smoothness $m$ and sample size $n$, we define 
    \begin{equation*}
         \psi_{n}(m) = \left( \frac{n}{\logn} \right)^{-\frac{2m}{2m + d}}.
    \end{equation*}
    First, we show that it is sufficient to consider the true Sobolev smoothness $m$ in the finite candidate set $\mcA = \{m_{1},\cdots, m_{N}\}$ with $ m_{j} - m_{j-1}\asymp 1/\logn$. If $m \in (m_{j-1}, m_{j})$, then by the embedding of Sobolev spaces, we have $H^{m_{j}} \subset H^{m} \subset H^{m_{j-1}}$. Therefore, since $\psi_{n}(m)$ is squeezed between $\psi_{n}(m_{j-1})$ and $\psi_{n}(m_{j})$, it remains to show $\psi_{n}(m_{j-1}) \asymp \psi_{n}(m_{j})$. By the definition of $\psi_{n}(m)$, the claim follows since
    \begin{equation*}
        \log \frac{\psi_{n}(m_{j-1})}{\psi_{n}(m_{j})} = \left( -\frac{2m_{j-1}}{2m_{j-1}+d} + \frac{2m_{j}}{2m_{j} + d} \right) \log \frac{n}{\logn} \asymp (m_{j} - m_{j-1})\logn \asymp 1.
    \end{equation*}
    Therefore, we can safely assume $f^{*} \in H^{m_{i}}$ where $i\in \{1,2,\cdots,N\}$. 

    Let $n_{0} = \lfloor \frac{n}{2} + 1 \rfloor$, i.e. $n_{0}\geq \frac{n}{2}$, by Theorem~\ref{thm: non-adaptive rate of SA with Gaussian}, for some constants $C$ that doesn't depend on $n$, we have
    \begin{equation}\label{eqn: eqn1 in adaptive estimator via TV}
    \begin{aligned}
        \left\| \hat{f}_{\lambda_{m},\mcD_{1}} - f^{*} \right\|_{L^{2}}^{2} & \leq \left( \log\frac{4}{\delta} \right)^2 \left(  \text{E}(\lambda_{m},n_{0}) + \text{A}(\lambda_{m},n_{0}) \right)\\
    \end{aligned}
    \end{equation}
    for all $m \in \mcA$ simultaneously with probability at least $1 - N \delta$. Here, $\text{E}(\lambda,n)$ and $\text{A}(\lambda,n)$ denote the estimation and approximation error that depends on the regularization parameter $\lambda$ and sample size $n$ in non-adaptive rate proof. 
    
    Furthermore, by Theorem 7.2 in \citet{steinwart2008support} and Assumption~\ref{assumption: error tail}, we have 
    \begin{equation}\label{eqn: eqn2 in adaptive estimator via TV}
    \begin{aligned}
         \left\|\hat{f}_{\hat{\lambda}} - f^{*} \right\|_{L^{2}}^{2} & < 6 \left( \inf_{m \in \mcA}  \left\|\hat{f}_{\lambda_{m}} - f^{*} \right\|_{L^{2}}^{2} \right) + \frac{128\sigma^{2}L^2 \left(\log\frac{1}{\delta} + \log(1 + N)\right)}{n-n_{0}}\\
        & < 6 \left( \inf_{m \in \mcA} \left\|\hat{f}_{\lambda_{m}} - f^{*} \right\|_{L^{2}}^{2} \right) + \frac{512 \sigma^{2}L^2 \left(\log\frac{1}{\delta} + \log(1 + N)\right)}{n}
    \end{aligned}
    \end{equation}
    with probability $1 - \delta$, where the last inequality is based on the fact that $n-n_{0} \geq \frac{n}{2} - 1 \geq \frac{n}{4}$.
    
    Combining (\ref{eqn: eqn1 in adaptive estimator via TV}) and (\ref{eqn: eqn2 in adaptive estimator via TV}), we have 
    \begin{equation*}
    \begin{aligned}
        \left\|\hat{f}_{\hat{\lambda}} - f^{*} \right\|_{L^{2}}^{2}
        & < 6 \left( \log\frac{4}{\delta} \right)^2 \left(  \inf_{m \in \mcA}  \text{E}(\lambda_{m},n_{0}) + \text{A}(\lambda_{m},n_{0}) \right) + \frac{512 \sigma^{2}L^2 \left(\log\frac{1}{\delta} + \log(1 + N)\right)}{n} \\
        & \leq 6 C \left( \log\frac{4}{\delta} \right)^2 n_{0}^{-\frac{2m}{2m + d}} + \frac{512 \sigma^{2}L^2 \left(\log\frac{1}{\delta} + \log(1 + N)\right)}{n} \\
        & \leq 12 C \left( \log\frac{4}{\delta} \right)^2 n^{-\frac{2m}{2m + d}} + \frac{512 \sigma^{2}L^2 \left(\log\frac{1}{\delta} + \log(1 + N)\right)}{n} \\
    \end{aligned}
    \end{equation*}
    with probability at least $1 - (1+N)\delta$. With a variable transformation, we have 
    \begin{equation}\label{eqn: eqn3 in adaptive estimator via TV}
        \left\|\hat{f}_{\hat{\lambda}}  - f^{*} \right\|_{L^{2}}^{2} 
         \leq 12 C \left( \log\frac{4(1+N)}{\delta} \right)^2 n^{-\frac{2m}{2m + d}} + \frac{512 \sigma^{2}L^2 \left(\log\frac{1+N}{\delta} + \log(1 + N)\right)}{n} 
    \end{equation}
    with probability $1-\delta$. Therefore, for the first term 
    \begin{equation}\label{eqn: eqn4 in adaptive estimator via TV}
    \begin{aligned}
        12 C \left( \log\frac{4(1+N)}{\delta} \right)^2 n^{-\frac{2m}{2m + d}} & \leq 24C \left\{\left(\log\frac{4}{\delta} \right)^2 \log^{2}(1+N) + 1 \right\}n^{-\frac{2m}{2m + d}} \\
        & \leq 24C' \left(\log\frac{4}{\delta} \right)^2  \left( \frac{n}{\logn}\right)^{-\frac{2m}{2m + d}} + 24 C n^{-\frac{2m}{2m + d}}
    \end{aligned}
    \end{equation}
    where the first inequality is based on the fact that $a+b<ab+1$ for $a,b>1$, while the second inequality is based on the fact that $\log(x) \leq x^{\frac{m}{2m + d}}$ for some $n$ such that $\log(\logn) / \logn < 1/4$. For the second term,
    \begin{equation}\label{eqn: eqn5 in adaptive estimator via TV}
    \begin{aligned}
        \frac{512 \sigma^{2}L^2 \left(\log\frac{1+N}{\delta} + \log(1 + N)\right)}{n} & \leq \frac{512 \sigma^{2}L^2 \left(\log\frac{1}{\delta} + 1 + 2\log(1 + N)\right)}{n} \\
        &\leq \frac{512 \sigma^{2}L^2 \left(\log\frac{1}{\delta} + 1 + 2\logn\right)}{n}
    \end{aligned}
    \end{equation}
    The proof is finished by combining (\ref{eqn: eqn3 in adaptive estimator via TV}), (\ref{eqn: eqn4 in adaptive estimator via TV}) and (\ref{eqn: eqn5 in adaptive estimator via TV}).
\end{proof}

\subsection{Supporting Theorem}\label{apd subsec: proof of theorems for SA with Gaussian}

\begin{theorem}\label{thm: bounds for A3}
    Suppose that Assumptions in the estimation error theorem hold. We have
    \begin{equation*}
         \left\| \left( T_{K} + \lambda I  \right)^{-\frac{1}{2}} \left( g_{n}  - \left( T_{K,n} + \lambda I  \right)^{-1}f_{\lambda} \right)  \right\|_{\mcH_{K}} \leq C \log\left(\frac{4}{\delta}\right) n^{-\frac{m}{2m+d}} 
    \end{equation*}
    where $C$ is a universal constant.
\end{theorem}
\begin{proof}
    Denote
    \begin{equation*}
    \begin{aligned}
        & \xi_{i} = \xi(x_i,y_i) = \left( T_{K} + \lambda I \right)^{-\frac{1}{2}}(K_{x_i}y_i - T_{K,x_i}f_{\lambda}) \\
        & \xi_{x} = \xi(x,y) = \left( T_{K} + \lambda I \right)^{-\frac{1}{2}}(K_{x}y - T_{K,x}f_{\lambda}),
    \end{aligned}
    \end{equation*}
    then it is equivalent to show 
    \begin{equation*}
        \left\| \frac{1}{n}\sum_{i=1}^{n} \xi_{i}  - \E \xi_{x} \right\|_{\mcH_{K}} \leq C \log\left(\frac{4}{\delta}\right) n^{-\frac{m}{2m+d}} 
    \end{equation*}
    Since $f^{*}\in H^{m}(R)$ with $m>d/2$ and $\mcX$ has a Lipschitz boundary, the Sobolev embedding $H^{m}(\mcX)\hookrightarrow L^{\infty}(\mcX)$ gives $\|f^{*}\|_{L^{\infty}}\leq C_{\mathrm{emb}}\|f^{*}\|_{H^{m}}\leq C_{\mathrm{emb}}R =: B_{\infty} < \infty$. Denoting
    \begin{equation*}
        \left\| \left( T_{K} + \lambda I  \right)^{-\frac{1}{2}} \left( g_{n}  - \left( T_{K,n} + \lambda I  \right)^{-1}f_{\lambda} \right)  \right\|_{\mcH_{K}} = \underbrace{\left\|\frac{1}{n} \sum_{i=1}^n \xi_i -\E \xi_x \right\|_{\mcH_{K}}}_{B_{1}},
    \end{equation*}
    it remains to bound $B_{1}$.
    Applying Theorem~\ref{theorem: bounds on B1}, for any $\delta \in (0,1)$, with probability $1-\delta$, we have
    \begin{equation}\label{eqn: upper bound for B1}
        B_{1} \leq \log\left(\frac{2}{\delta}\right) 
        \left( \frac{C_{1}\sqrt{\mcN(\lambda)}}{n} \tilde{M} + \frac{C_{2}\sqrt{\mcN(\lambda)}}{\sqrt{n}} + \frac{C_{1}\log(\frac{1}{\lambda})^{-\frac{m}{2}}\sqrt{\mcN(\lambda)}}{\sqrt{n}}  \right)
    \end{equation}
    where $C_{1} = 8\sqrt{2}$, $C_{2} = 8\sigma$ and $\tilde{M} = L + M = L + \sqrt{2}\,E\,E_{K}\sqrt{\mcN(\lambda)}\,\|f^{*}\|_{L^{2}} + B_{\infty}$. By choosing $\lambda \asymp \exp\{ -C n^{\frac{2}{2m+d}}\}$ and applying Lemma \ref{lemma: bound for effective dimension}, we have
    \begin{itemize}
        \item for the second term in \eqref{eqn: upper bound for B1},
        \begin{equation*}
            \frac{C_{2}\sqrt{\mcN(\lambda)}}{\sqrt{n}} \asymp n^{-\frac{m}{2m+d}}.
        \end{equation*}

        \item for the third term in \eqref{eqn: upper bound for B1},
        \begin{equation*}
             \frac{C_{1}\log(\frac{1}{\lambda})^{-\frac{m}{2}}\sqrt{\mcN(\lambda)}}{\sqrt{n}} \lesssim  \frac{C_{2}\sqrt{\mcN(\lambda)}}{\sqrt{n}} \asymp n^{-\frac{m}{2m+d}}.
        \end{equation*}

        \item for the first term in \eqref{eqn: upper bound for B1}, using $\mcN(\lambda)\asymp n^{\frac{d}{2m+d}}$ and that $B,L$ are constants,
        \begin{equation*}
            \frac{C_{1}\sqrt{\mcN(\lambda)}}{n} \tilde{M}\leq \frac{C_1 (L+B_{\infty}) \sqrt{\mcN(\lambda)}}{n} + \frac{C_1 \sqrt{2} E E_{K} \|f^{*}\|_{L^{2}}\, \mcN(\lambda)}{n}  \lesssim \frac{\mcN(\lambda)}{n} = n^{-\frac{2m}{2m+d}},
        \end{equation*}
        which is faster than the target rate $n^{-\frac{m}{2m+d}}$.
    \end{itemize}
    Combining all facts, with probability $1-\delta$ we have
    \begin{equation*}
        \left\| \frac{1}{n}\sum_{i=1}^{n} \xi_{i}  - \E \xi_{x} \right\|_{\mcH_{K}} = B_{1} \leq C \log\left(\frac{2}{\delta}\right) n^{-\frac{m}{2m+d}},
    \end{equation*}
    which is the desired estimate.
\end{proof}

\begin{theorem}\label{theorem: bounds on B1}
    Under the same conditions as Theorem~\ref{thm: bounds for A3}, we have 
    \begin{equation*}
    \begin{aligned}
            & \left\|\frac{1}{n} \sum_{i=1}^n \xi_i -\E \xi_x \right\|_{\mcH_{K}}  \\
            & \qquad \leq \log\left(\frac{2}{\delta}\right) 
        \left( \frac{C_{1}\sqrt{\mcN(\lambda)}}{n} \tilde{M} + \frac{C_{2}\sqrt{\mcN(\lambda)}}{\sqrt{n}} + \frac{C_{1}\log(\frac{1}{\lambda})^{-\frac{m}{2}}\sqrt{\mcN(\lambda)}}{\sqrt{n}}  \right)
    \end{aligned}
    \end{equation*}
    where $C_{1} = 8\sqrt{2}$, $C_{2} = 8\sigma$ and $\tilde{M} = L + M = L + \sqrt{2}\,E\,E_{K}\sqrt{\mcN(\lambda)}\,\|f^{*}\|_{L^{2}} + B_{\infty}$.
\end{theorem}

\begin{proof}
    The objective can be typically controlled by certain concentration inequality. Therefore, to leverage Bernstein inequality, i.e., Lemma \ref{lemma: Bernstein inequality}, we first bound the $m$-th moment of $\xi_{x}$.
    \begin{equation*}
    \begin{aligned}
    \E\left\|\xi_{x} \right\|_{\mcH_{K}}^m & =\E\left\|\left( T_{K} + \lambda I \right)^{-\frac{1}{2}} K_x\left(y-f_\lambda(x)\right) \right\|_{\mcH_{K}}^m \\
    & \leq \E\left(\left\|\left( T_{K} + \lambda I \right)^{-\frac{1}{2}} K(x, \cdot)\right\|_{\mcH_{K}}^m \E\left(\left|\left(y-f_\lambda(x)\right) \right|^m \mid x\right)\right) .
    \end{aligned}
    \end{equation*}
    Using the inequality $(a+b)^m \leq 2^{m-1}\left(a^m+b^m\right)$, we have
    \begin{equation*}
    \begin{aligned}
    \left|y-f_\lambda(x)\right|^m & \leq 2^{m-1}\left(\left|f_\lambda(x)-f^{*}(x)\right|^m+\left|f^{*}(x)-y\right|^m\right) \\
    & =2^{m-1}\left(\left|f_\lambda(x)-f^{*}(x)\right|^m+|\epsilon|^m\right) .
    \end{aligned}
    \end{equation*}
    Combining the inequalities, we have
    \begin{equation*}
    \begin{aligned}
     \E\left\|\xi_{x} \right\|_{\mcH_{K}}^m & \leq \underbrace{2^{m-1} \E\left(\left\|\left( T_{K} + \lambda I \right)^{-\frac{1}{2}} K(x, \cdot)\right\|_{\mcH_{K}}^m\left|f_\lambda(x)-f^{*}(x)\right|^m\right)}_{C_{1}} \\
    &\quad + \underbrace{2^{m-1} \E\left(\left\|\left( T_{K} + \lambda I \right)^{-\frac{1}{2}} K(x, \cdot)\right\|_{\mcH_{K}}^m \E\left(\left|\epsilon \right|^m \mid x\right)\right)}_{C_{2}}.
    \end{aligned}
    \end{equation*}
    We first focus on $C_{2}$, by Lemma \ref{lemma: bound for kernel function}, we have
    \begin{equation*}
        \E\left\|\left( T_{K} + \lambda I \right)^{-\frac{1}{2}} K(x, \cdot)\right\|_{\mcH_{K}}^m \leq \left( E_{K}^{2} \mcN(\lambda) \right)^{\frac{m}{2}}.
    \end{equation*}
    By the error moment assumption, we have 
    \begin{equation*}
        \E\left(\left|\epsilon \right|^m \mid x\right) \leq \frac{1}{2}m! \sigma^2 L^{m-2},
    \end{equation*}
    together, we have 
    \begin{equation}\label{eqn: bound on A2}
        C_{2} \leq \frac{1}{2} m! \left(\sqrt{2}\sigma \sqrt{\mcN(\lambda)} \right)^{2} \left(2L \mcN(\lambda)\right)^{m-2}.
    \end{equation}
    Turning to bounding $C_{1}$, we first control the $L^{\infty}$ scale $\| f_{\lambda} - f^{*} \|_{L_{\infty}} \leq \|f_{\lambda}\|_{L^{\infty}}+\|f^{*}\|_{L^{\infty}}$. For the intermediate term, expanding $f^{*}=\sum_{j}a_{j}e_{j}$ gives $f_{\lambda}=\sum_{j}s_{j}\phi_{\lambda}(s_{j})a_{j}e_{j}$, and by Cauchy--Schwarz,
    \begin{equation*}
    \begin{aligned}
        |f_{\lambda}(x)|
        & \;\leq\; \left( \sum_{j} \frac{s_{j}}{s_{j}+\lambda} e_{j}^{2}(x) \right)^{\frac{1}{2}} \left( \sum_{j} (s_{j}+\lambda) s_{j} \phi_{\lambda}(s_{j})^{2} a_{j}^{2} \right)^{\frac{1}{2}} \\
        & \;\leq\; \left\| \left( T_{K} + \lambda I \right)^{-\frac{1}{2}} K(x,\cdot) \right\|_{\mcH_{K}} \cdot \sqrt{2}\,E\, \| f^{*} \|_{L^{2}}
        \;\leq\; \sqrt{2}\, E\, E_{K} \sqrt{\mcN(\lambda)}\, \| f^{*} \|_{L^{2}},
    \end{aligned}
    \end{equation*}
    where the second factor uses the filter property $(s_{j}+\lambda)s_{j}\phi_{\lambda}(s_{j})^{2}\leq 2E^{2}$ and the last step applies Lemma~\ref{lemma: bound for kernel function} (under Assumption~\ref{assumption: assumption on Gaussian kernel}\ref{assumption: bounds on effective dimension}; under \ref{assumption: bounds on adjusted eigenfunction}, $\mcN(\lambda)$ is replaced by $\mcN_{h}(\lambda)$ via Lemma~\ref{lemma: bound for kernel function under adjusted eigenfunction}). Hence
    \begin{equation*}
        \left\| f_{\lambda} - f^{*} \right\|_{L_{\infty}} \;\leq\; \sqrt{2}\, E\, E_{K} \sqrt{\mcN(\lambda)}\, \| f^{*} \|_{L^{2}} + B_{\infty} \;=: M,
    \end{equation*}
    where $\|f^{*}\|_{L^{2}}\leq R$ and $B_{\infty} = \|f^{*}\|_{L^{\infty}}\leq C_{\mathrm{emb}}R$ are constants independent of $n$.
    With bounds on approximation error, we get the upper bound for $C_{1}$ as
    \begin{equation}\label{eqn: bound on A1}
    \begin{aligned}
        C_{1} & \leq 2^{m-1} \mcN(\lambda)^{\frac{m}{2}}  \left\| f_{\lambda} - f^{*} \right\|_{L_{\infty}}^{m-2}  \left\| f_{\lambda} - f^{*} \right\|_{L^{2}}^2\\
        & \leq 2^{m-1} \mcN(\lambda)^{\frac{m}{2}} M^{m-2} \log(\frac{1}{\lambda})^{-m} \\
        & \leq \frac{1}{2}m! \left( 2 \log(\frac{1}{\lambda})^{-\frac{m}{2}} \sqrt{\mcN(\lambda)} \right)^2 \left( 2M \sqrt{\mcN(\lambda)}\right)^{m-2}. 
    \end{aligned}
    \end{equation}
    Denote 
    \begin{equation*}
        \begin{aligned}
            & \tilde{L} = 2(L+M)\sqrt{\mcN(\lambda)}, \\
            & \tilde{\sigma} = \sqrt{2}\sigma \sqrt{\mcN(\lambda)} + 2 \log(\frac{1}{\lambda})^{-\frac{m}{2}} \sqrt{\mcN(\lambda)},
        \end{aligned}
    \end{equation*}
    and combine the upper bounds for $C_{1}$ and $C_{2}$, i.e. (\ref{eqn: bound on A1}) and (\ref{eqn: bound on A2}), then we have 
    \begin{equation*}
         \E\left\|\xi_{x} \right\|_{\mcH_{K}}^m \leq \frac{1}{2} m! \tilde{\sigma}^2 \tilde{L}^{m-2}.
    \end{equation*}
    The proof is finished by applying Lemma \ref{lemma: Bernstein inequality}.
\end{proof}

\subsection{Supporting Lemmas}

\begin{lemma}\label{lemma: bounds for A2}
    For all $\delta \in (0,1)$, with probability at least $1-\delta$, we have 
    \begin{equation*}
        \left\| \left( T_{K} + \lambda I \right)^{-\frac{1}{2}} \left( T_{K} - T_{K,n} \right) \left( T_{K} + \lambda I \right)^{-\frac{1}{2}} \right\|_{op} \leq \frac{4\mcN(\lambda)B}{3n} + \sqrt{\frac{2\mcN(\lambda)}{n}B}
    \end{equation*}
    where 
    \begin{equation*}
        B = \log \left( \frac{4\mcN(\lambda)}{\delta} \frac{(\|T_{K}\|_{op}+\lambda)}{\|T_{K}\|_{op}} \right).
    \end{equation*}
\end{lemma}
\begin{proof}
    Denote $A_{i} = (T_{K} + \lambda I )^{-\frac{1}{2}} (T_{K} - T_{K,x_{i}}) (T_{K} + \lambda I )^{-\frac{1}{2}}$, applying Lemma \ref{lemma: bound for kernel function}, we get 
    \begin{equation*}
    \begin{aligned}
        \|A_{i}\|_{op} & \leq \left\| \left( T_{K} + \lambda I \right)^{-\frac{1}{2}} T_{K,x} \left( T_{K} + \lambda I \right)^{-\frac{1}{2}}\right\|_{op} +  \left\| \left( T_{K} + \lambda I \right)^{-\frac{1}{2}} T_{K,x_{i}} \left( T_{K} + \lambda I \right)^{-\frac{1}{2}}\right\|_{op} \\
        & \leq  2 E_{K}^{2} \mcN(\lambda)
    \end{aligned}
    \end{equation*}
    Notice  
    \begin{equation*}
    \begin{aligned}
        \E A_{i}^2  & \preceq \E \left[ \left( T_{K} + \lambda I \right)^{-\frac{1}{2}} T_{K,x_{i}} \left( T_{K} + \lambda I \right)^{-\frac{1}{2}} \right]^2 \\
        & \preceq E_{K}^{2} \mcN(\lambda ) \E \left[ \left( T_{K} + \lambda I \right)^{-\frac{1}{2}} T_{K,x_{i}} \left( T_{K} + \lambda I \right)^{-\frac{1}{2}} \right] \\
        & = E_{K}^{2} \mcN(\lambda ) (T_{K} + \lambda I )^{-1}T_{K} := V
    \end{aligned}
    \end{equation*}
    where $A\preceq B$ denotes $B-A$ is a positive semi-definite operator. Notice 
    \begin{equation*}
        \|V\|_{op} = \mcN(\lambda) \frac{\|T_{K}\|_{op}}{\|T_{K}\|_{op} + \lambda}  \leq \mcN(\lambda), \quad \text{and} \quad \operatorname{tr}(V) = \mcN(\lambda)^2.
    \end{equation*}
    The proof is finished by applying Lemma~\ref{lemma: concentration inequality for self-adjoint operator} to $A_{i}$ and $V$.
\end{proof}

\begin{lemma}\label{lemma: bound for effective dimension}
By choosing $\log(1/\lambda)\asymp n^{\frac{2}{2m+d}}$, we have 
\begin{equation*}
    \mcN(\lambda) = O\left( n^{\frac{d}{2m+d}} \right).
\end{equation*}
\end{lemma}
\begin{proof}
For a positive integer $J\geq 1$
\begin{equation*}
\begin{aligned}
    \mcN(\lambda) & = \sum_{j=1}^{J} \frac{s_{j}}{s_{j} + \lambda} + \sum_{j = J+1}^{\infty} \frac{s_{j}}{s_{j} + \lambda} \\
    & \leq J + \sum_{j = J+1}^{\infty} \frac{s_{j}}{s_{j} + \lambda} \\
    & \leq J + \frac{C_{1}}{\lambda} \int_{J}^{\infty} \exp\{ -Cx^{2} \} dx \\
    & \leq J + \frac{1}{ \lambda} \frac{C_{1} \exp\{-C_{2} J^{2}\}}{\sqrt{2 C_{2}} J},
\end{aligned}
\end{equation*}
where we use the fact that the eigenvalue of the Gaussian kernel has a fast decay condition, i.e., $s_{j} \leq C_{1} \exp\{ -C_{2} j^{2} \}$ as stated in \citep{zhang2015divide}, and the inequality 
\begin{equation*}
    \int_{x}^{\infty} \exp\{ \frac{-t^2}{2} \}dt \leq \int_{x}^{\infty} \frac{t}{x} \exp\{ \frac{-t^2}{2} \}dt  \leq \frac{\exp\{-\frac{x^2}{2}\}}{x}.
\end{equation*}
Then select $J = \lfloor n^{\frac{d}{2 m + d}} \rfloor$ and $\lambda = \exp\{ - C^{'} n^{\frac{2}{2 m + d}}\}$ with $C'\leq C_{2}$ leads to 
\begin{equation*}
    \mcN(\lambda) = O\left( n^{\frac{d}{2m+d}} \right).
\end{equation*}
\end{proof}

\begin{lemma}\label{lemma: bound for kernel function}
    Suppose Assumption~\ref{assumption: assumption on Gaussian kernel}\ref{assumption: bounds on effective dimension} holds. Then, for $Q_{X}$-almost $x\in \mcX$, we have 
    \begin{equation*}
        \left\| \left( T_{K} + \lambda I \right)^{-\frac{1}{2}} K(x,\cdot) \right\|_{\mcH_{K}}^{2} \leq E_{K}^{2} \mcN(\lambda), \quad \text{and}\quad \E \left\| \left( T_{K} + \lambda I \right)^{-\frac{1}{2}} K(x,\cdot) \right\|_{\mcH_{K}}^{2} \leq \mcN(\lambda).
    \end{equation*}
    For some constant $E_K$. Consequently, we also have 
    \begin{equation*}
        \left\| \left( T_{K} + \lambda I \right)^{-\frac{1}{2}} T_{K,x} \left( T_{K} + \lambda I \right)^{-\frac{1}{2}}\right\|_{op} \leq E_{K}^{2} \mcN(\lambda).
    \end{equation*}
\end{lemma}
\begin{proof}
    For the first inequality, we have 
    \begin{equation*}
    \begin{aligned}
        \left\| \left( T_{K} + \lambda I \right)^{-\frac{1}{2}} K(x,\cdot) \right\|_{\mcH_{K}}^{2} & =  \left\| \sum_{j=1} \frac{1}{\sqrt{s_{j}+\lambda}} s_{j} e_{j}(x)e_{j}(\cdot)\right\|_{\mcH_{K}}^{2}\\
        & = \sum_{j=1}^{\infty} \frac{s_{j}}{s_{j} + \lambda} e_{j}^{2}(x)\\
        & \leq E_{K}^{2} \mcN(\lambda),
    \end{aligned}
    \end{equation*}
    where the first equality is by Mercer Theorem and the inequality is by Assumption~\ref{assumption: assumption on Gaussian kernel}\ref{assumption: bounds on effective dimension}.
    The second inequality follows given the fact that $\E e_{j}^{2}(x) = 1$. The third inequality comes from the observation that for any $f\in \mcH_{K}$
    \begin{equation*}
         \left( T_{K} + \lambda I \right)^{-\frac{1}{2}} T_{K,x} \left( T_{K} + \lambda I \right)^{-\frac{1}{2}} (f) = \left \langle  \left( T_{K} + \lambda I \right)^{-\frac{1}{2}} K(x,\cdot), f \right\rangle_{\mcH_{K}}  \left( T_{K} + \lambda I \right)^{-\frac{1}{2}} K(x,\cdot),
    \end{equation*}
    and
    \begin{equation*}
    \begin{aligned}
        \left\| \left( T_{K} + \lambda I \right)^{-\frac{1}{2}} T_{K,x} \left( T_{K} + \lambda I \right)^{-\frac{1}{2}}\right\|_{op} & = \sup_{\|f\|_{\mcH_{k}}=1}\|\left( T_{K} + \lambda I \right)^{-\frac{1}{2}} T_{K,x} \left( T_{K} + \lambda I \right)^{-\frac{1}{2}} (f)\|_{\mcH_{K}} \\
         & = \left\| \left( T_{K} + \lambda I \right)^{-\frac{1}{2}} K(x,\cdot) \right\|_{\mcH_{K}}^{2}.
    \end{aligned}
    \end{equation*}
\end{proof}

The following lemma provides the concentration inequality about self-adjoint Hilbert-Schmidt operator-valued random variables, which is widely used in related kernel method literature, e.g., Theorem 27 in \citet{fischer2020sobolev}, Lemma 26 in \citet{lin2020optimal} and Lemma 32 in \citet{zhang2024optimality}.

\begin{lemma} \label{lemma: concentration inequality for self-adjoint operator}
    Let $(\mcX,\mcB, \mu)$ be a probability space, and $\mcH$ be a separable Hilbert space. Suppose $A_{1},\cdots, A_{n}$ are i.i.d. random variables whose values are in the set of self-adjoint Hilbert-Schmidt operators. If $\E A_{i} = 0$ and the operator norm $\|A_{i}\|\leq L$ $\mu$-a.e. $x\in \mcX$, and there exists a self-adjoint positive semi-definite trace class operator $V$ with $\E A_{i}^2 \preceq V$. Then for $\delta\in(0,1)$, with probability at least $1-\delta$, we have 
    \begin{equation*}
        \left\| \frac{1}{n}\sum_{i=1}^{n} A_{i} \right\| \leq \frac{2L\beta}{3n} + \sqrt{\frac{2\|V\|\beta}{n}}, 
    \end{equation*}
    where $\beta = \log( 4\text{tr}(V) / \delta \|V\| )$.
\end{lemma}

\begin{lemma}(Bernstein inequality)\label{lemma: Bernstein inequality}
Let $(\Omega, \mathcal{B}, P)$ be a probability space, $H$ be a separable Hilbert space, and $\xi: \Omega \rightarrow H$ be a random variable with
\begin{equation*}
    \E\|\xi\|_H^m \leq \frac{1}{2} m ! \sigma^2 L^{m-2},
\end{equation*}
for all $m>2$. Then for $\delta \in(0,1)$, $\xi_i$ are i.i.d. random variables, with probability at least $1-\delta$, we have
\begin{equation*}
    \left\|\frac{1}{n} \sum_{i=1}^n \xi_i-\E \xi\right\|_H \leq 4 \sqrt{2} \log \left(\frac{2}{\delta}\right)\left(\frac{L}{n}+\frac{\sigma}{\sqrt{n}}\right).
\end{equation*}
\end{lemma}

\subsection{Bound of estimation error under Assumption~\ref{assumption: assumption on Gaussian kernel}~\ref{assumption: bounds on adjusted eigenfunction}}\label{apd: Bound of estimation error under adjusted eigenfunction}

The key idea to introduce the non-decreasing function $h$ in Assumption~\ref{assumption: assumption on Gaussian kernel}~\ref{assumption: bounds on adjusted eigenfunction} is to replace the effective dimension $\mcN(\lambda)$ in Theorem~\ref{theorem: bounds on B1} by a so-called amplified effective dimension $\mcN_{h}(\lambda)$, which is defined by 
\[
    \mcN_{h}(\lambda):= \sum_{j\geq 1}\frac{s_{j}}{s_{j} + \lambda}h(j).
\]

Then, the conclusion of Theorem~\ref{theorem: bounds on B1} still holds by directly using the following two Lemmas, i.e., Lemma~\ref{lemma: bound for kernel function under adjusted eigenfunction} and Lemma~\ref{lemma: bound for effective dimension with adjusted eigenfunction}.

\begin{lemma}\label{lemma: bound for kernel function under adjusted eigenfunction}
    Suppose Assumption~\ref{assumption: assumption on Gaussian kernel}~\ref{assumption: bounds on adjusted eigenfunction} holds. Then, for $Q_{X}$-almost $x\in \mcX$, we have 
    \begin{equation*}
        \left\| \left( T_{K} + \lambda I \right)^{-\frac{1}{2}} K(x,\cdot) \right\|_{\mcH_{K}}^{2} \leq E_{K}^{2} \mcN_{h}(\lambda), \quad \text{and}\quad \E \left\| \left( T_{K} + \lambda I \right)^{-\frac{1}{2}} K(x,\cdot) \right\|_{\mcH_{K}}^{2} \leq \mcN_{h}(\lambda).
    \end{equation*}
    For some constant $E_K$. Consequently, we also have 
    \begin{equation*}
        \left\| \left( T_{K} + \lambda I \right)^{-\frac{1}{2}} T_{K,x} \left( T_{K} + \lambda I \right)^{-\frac{1}{2}}\right\|_{op} \leq E_{K}^{2} \mcN_{h}(\lambda).
    \end{equation*}
\end{lemma}
\begin{proof}
    For the first inequality, we have 
    \begin{equation*}
    \begin{aligned}
        \left\| \left( T_{K} + \lambda I \right)^{-\frac{1}{2}} K(x,\cdot) \right\|_{\mcH_{K}}^{2} & =  \left\| \sum_{j=1} \frac{1}{\sqrt{s_{j}+\lambda}} s_{j} e_{j}(x)e_{j}(\cdot)\right\|_{\mcH_{K}}^{2}\\
        & = \sum_{j=1}^{\infty} \frac{s_{j}}{s_{j} + \lambda} h(j) \left( h(j)^{-\frac{1}{2}}  e_{j}(x) \right)^{2}\\
        & \leq \mcN_{h}(\lambda) \sup_{x,j} h(j)^{-1}  e_{j}^{2}(x) \\
        & \leq E_{K}^{2} \mcN_{h}(\lambda),
    \end{aligned}
    \end{equation*}
    where the inequality is by Assumption~\ref{assumption: assumption on Gaussian kernel}~\ref{assumption: bounds on adjusted eigenfunction}. The rest of the proof follows the same as Lemma~\ref{lemma: bound for kernel function}
\end{proof}

\begin{lemma}\label{lemma: bound for effective dimension with adjusted eigenfunction}
By choosing $\log(1/\lambda)\asymp n^{\frac{2}{2m+d}}$, we have 
\begin{equation*}
    \mcN_{h}(\lambda) = O\left( n^{\frac{d}{2m+d}} \right).
\end{equation*}
\end{lemma}
\begin{proof}
For a positive integer $J\geq 1$
\begin{equation*}
\begin{aligned}
    \mcN_{h}(\lambda) & = \sum_{j=1}^{J} \frac{s_{j}}{s_{j} + \lambda} h(j) + \sum_{j = J+1}^{\infty} \frac{s_{j}}{s_{j} + \lambda} h(j) \\
    & \leq Jh(J) + \sum_{j = J+1}^{\infty} \frac{s_{j}}{s_{j} + \lambda}h(j) \\
    & \leq Jh(J) + \frac{C_{1}}{\lambda} \int_{J}^{\infty} \exp\{ -C_{2}x^{2} \}h(x) dx \\
    & \leq Jh(J) + \frac{C_{1}}{ \lambda} \frac{1}{2C_{2} - c} \frac{\exp\{ -C_{2}J^{2}\}}{J} h(J).
\end{aligned}
\end{equation*}
where $c$ is a constant in $(0,2C_{2})$. The last inequality is based on Lemma~\ref{lemma: bound for tails of amplified effective dimensions}. Given $\log(\lambda^{-1}) \asymp n^{\frac{2}{2 m + d}}$, we select $J = \lfloor n^{\frac{1}{2 m + d}} \rfloor$ and $\lambda = \exp\{ - C^{'} n^{\frac{2}{2 m + d}}\}$ with $C'\leq C_{2}$, which leads to 
\begin{equation*}
    \mcN_{h}(\lambda) = O\left( n^{\frac{1}{2m+d}} h(n^{\frac{1}{2m+d}}) \right).
\end{equation*}
Hence, as long as $h(x) = O(x^{d-1})$, we have 
\begin{equation*}
    \mcN_{h}(\lambda) = O\left( n^{\frac{d}{2m+d}} \right).
\end{equation*}
\end{proof}

\begin{lemma}\label{lemma: bound for tails of amplified effective dimensions}
    Let $x\geq 1$ and assume
    \begin{enumerate}
        \item $h:[x,\infty) \rightarrow [1,\infty)$ is non-decreasing and absolutely continuous;
        \item for some $c \in [0,2C)$,
        \begin{equation*}
            \frac{d}{dt} \log(h(t)) = \frac{h'(t)}{h(t)} \leq ct \quad \forall t\geq x,
        \end{equation*}
    \end{enumerate}
    Then
    \begin{equation*}
        \int_{x}^{\infty} \exp\{ -Ct^{2} \} h(t)\, dt \leq \frac{1}{2C - c} \frac{\exp\{ -Cx^{2}\}}{x} h(x).
    \end{equation*}
\end{lemma}
\begin{proof}
    For $t\geq x$, $\frac{t}{x} \geq 1$, hence
    \begin{equation*}
        \int_{x}^{\infty} \exp\{ -Ct^{2} \}h(t)dt \leq \int_{x}^{\infty} \frac{t}{x} \exp\{ -Ct^{2} \}h(t)dt 
    \end{equation*}
    Let $u = h(t)$, $dv = \frac{t}{x} \exp\{ -Ct^{2} \}dt$, we have
    \begin{equation*}
        \int_{x}^{\infty} \frac{t}{x} \exp\{ -Ct^{2} \}h(t)dt = \frac{\exp\{-Cx^{2}\}}{2Cx}h(x) + \frac{1}{2Cx} \int_{x}^{\infty} \exp\{ -Ct^{2}\} h'(t)dt.
    \end{equation*}
    By the assumption, we have $h'(t) \leq ct h(t)$, hence
    \begin{equation*}
        \int_{x}^{\infty} \frac{t}{x} \exp\{ -Ct^{2} \}h(t)dt \leq \frac{\exp\{-Cx^{2}\}}{2Cx}h(x) + \frac{c}{2C}\int_{x}^{\infty} \frac{t}{x} \exp\{ -Ct^{2} \}h(t)dt
    \end{equation*}
    which leads to
    \begin{equation*}
        \int_{x}^{\infty} \frac{t}{x} \exp\{ -Ct^{2} \}h(t)dt \leq \frac{1}{2C-c}\frac{\exp\{-Cx^{2}\}}{x}h(x).
    \end{equation*}
\end{proof}


\section{Proofs of Main Results for Robust Learning under Concept Shift}\label{apd: proof of transfer learning}

\subsection{Proof of Theorem~\ref{thm: lower bound of OTL}}
For the lower bound, we prove the alternative but asymptotically equivalent version, i.e. 
\begin{equation}\label{eqn: alternative version of lower bound}
    \inf_{\tilde{f}} \sup_{\Theta(R_{P},R_{\delta},m_{P},m_{\delta})} \mathbb{P} \left(  \| \tilde{f} - f^{Q} \|_{L^{2}}^2  \geq C \delta  \left( R_{P}^2\cdot (n_{P} + n_{Q})^{-\frac{2m_{P}}{2m_{P}+d}} + R_{\delta}^{2} \cdot n_{Q}^{-\frac{2m_{\delta}}{2m_{\delta}+d}}  \right) \right) \geq 1 - \delta
\end{equation}
for some $C$ that does not depend on $n_{P}$, $n_{Q}$, $R_{P}$, $R_{\delta}$ and $\delta$.

\begin{proof}
Any lower bound for a specific case directly implies a lower bound for the general case. Thus, we analyze the following two cases.

\begin{itemize}
    \item Consider $f^{\delta}(x)=0$ for all $x\in \mcX$, i.e., $R_{\delta} = 0$. This means both source and target data are drawn from the same distribution; thus, there is no concept shift. In such case, the lower bound of the transfer learning problem becomes finding the lower bound of the classical nonparametric regression problem with sample size $n_{P} + n_{Q}$ where the true function $f^{Q}$ lies in $H^{m_{P}}$ with Sobolev norm $R_{P}$. Therefore, applying the Lemma~\ref{lemma: lower bound for single task SA} leads to
    \begin{equation*}
        \inf_{\tilde{f}} \sup_{ \Theta(R_{P},R_{\delta},m_{P},m_{\delta})} \mathbb{P} \left( \| \tilde{f} - f^{Q} \|_{L^{2}}^2 \geq C_{1} \delta R_{P}^{2} (n_{P} +n_{Q})^{-\frac{2m_{P}}{2m_{P} + d}}  \right) \geq 1 - \delta,
    \end{equation*}
    where $C_{1}$ is independent of $\delta$, $R_{P}$, $n_{P}$ and $n_{Q}$.

    \item Consider $f^{P}(x) = 0$ for all $x\in \mcX$. This means source domain provide no further information about $f^{Q}$, making the parameter space $\Theta$ reduce to $\{Q: \|f^{\delta}\|_{H^{m_{\delta}}} \leq R_{\delta}\}$. Therefore, the lower bound for this case is to use the target dataset to recover $f^{\delta}$ (also $f^{Q}$), which leads to 
    \begin{equation*}
        \inf_{\tilde{f}} \sup_{ \Theta(R_{P},R_{\delta},m_{P},m_{\delta})} \mathbb{P} \left( \| \tilde{f} - f^{Q} \|_{L^{2}}^2 \geq C_{2} \delta R_{\delta}^{2} (n_{Q})^{-\frac{2m_{\delta}}{2m_{\delta} + d}}  \right) \geq 1 - \delta,
    \end{equation*}
    where $C_{2}$ is independent of $\delta$, $R_{\delta}$, and $n_{Q}$.
\end{itemize}
Combining the lower bound from both cases, we obtain the desired lower bound.
\end{proof}

\begin{remark}
This alternative version is also used in other transfer learning contexts like high-dimensional linear regression or GLM; see \citep{li2022transfer,tian2022transfer}. While takes a slightly different form, the upper bound of the excess risk for the HTL algorithm is still sharp since we consider sufficient large $n_{P}$ and $n_{Q}$ in the transfer learning regime, i.e., it is always assumed $n_{P} \gg n_{Q}$, and leads to $(n_{P} + n_{Q})^{-\frac{2m_{P}}{2m_{P}+d}} \asymp n_{P}^{-\frac{2m_{P}}{2m_{P}+d}}$. 
\end{remark}

\subsection{Proof of Theorem~\ref{thm: upper bound of OTL}}

\begin{proof}[Proof of Theorem~\ref{thm: upper bound of OTL}]
    We first decompose the excess risk of $\hat{f}^{Q}$ as follows,
    \begin{equation*}
    \begin{aligned}
        \left\| \hat{f}^{Q} - f^{Q}\right\|_{L^{2}}^{2} & = \left\| G(\hat{f}^{\delta}, \hat{f}^{P}) - G(f^{\delta},f^{P})\right\|_{L^{2}}^{2} \\
        & \leq L_{G}^{2} \left( \underbrace{\left\| \hat{f}^{P} - f^{P}\right\|_{L^{2}}^{2} }_{\text{pre-training error}} + \underbrace{\left\| \hat{f}^{\delta} - f^{\delta}\right\|_{L^{2}}^{2}}_{\text{fine-tuning error}} \right) \\
        & \leq L_{G}^{2} \left( \underbrace{\left\| \hat{f}^{P} - f^{P}\right\|_{L^{2}}^{2} }_{\text{pre-training error}} + \underbrace{2\left\| \hat{f}^{\delta} - \tilde{f}^{\delta}\right\|_{L^{2}}^{2}}_{\text{fine-tuning error I}} + \underbrace{2\left\| \tilde{f}^{\delta} - f^{\delta}\right\|_{L^{2}}^{2}}_{\text{fine-tuning error II}}\right).
    \end{aligned}
    \end{equation*}
    The first inequality is by Lipschitz continuity of $G$. Denote $\tilde{y}_{i}^{\delta} = g(y_{i}^{Q}, f^{P}(x_{i}^{Q}))$ as the true intermediate labels, i.e., the label that uses the true source function $f^{P}$ to construct. Then, the term $\tilde{f}^{\delta}$ is defined as
    \begin{equation*}
        \tilde{f}^{\delta}:= \mcA_{K,\lambda_{2}}^{\delta}(\{x_{i}^{Q}, \tilde{y}^{\delta}\}_{i=1}^{n_{Q}})= \phi_{\lambda_{2}}(T_{K,n}^{Q}) \frac{1}{n_{Q}} \sum_{i=1}^{n_{Q}} K_{x_{i}^{Q}}^{*} \tilde{y}_{i}^{\delta}.
    \end{equation*}
    Therefore, the fine-tuning error II can be viewed as using the true observed intermediate dataset $\tilde{\mcD}^{\delta} = \{x_{i}^{Q}, \tilde{y}^{\delta}\}_{i=1}^{n_{Q}}$ to recover $f^{\delta}$, which can be reduced to classical single dataset nonparametric regression case. While for fine-tuning error I, by the Lipschitz continuity of data transformation function $g$, we have
    \begin{equation*}
    \begin{aligned}
        \left\| \hat{f}^{\delta} - \tilde{f}^{\delta}\right\|_{L^{2}} & = \left\| \phi_{\lambda_{2}}(T_{K,n}^{Q})\frac{1}{n_{Q}}\sum_{i=1}^{n_{Q}}K_{x_{i}^{Q}}^{*}\left( y_{i}^{\delta} - \tilde{y}_{i}^{\delta}\right)\right\|_{L^{2}} \\
        & = \left\| \phi_{\lambda_{2}}(T_{K,n}^{Q})\frac{1}{n_{Q}}\sum_{i=1}^{n_{Q}}K_{x_{i}^{Q}}^{*}\left\{ g\left(y_{i}^{Q},\hat{f}^{P}(x_{i}^{Q})\right) - g\left(y_{i}^{Q},f^{P}(x_{i}^{Q})\right)\right\}\right\|_{L^{2}} \\
        & \leq L_{g} \left\| \phi_{\lambda_{2}}(T_{K,n}^{Q})\frac{1}{n_{Q}}\sum_{i=1}^{n_{Q}}K_{x_{i}^{Q}}^{*}\left( \hat{f}^{P}(x_{i}^{Q}) - f^{P}(x_{i}^{Q}) \right)\right\|_{L^{2}}.
    \end{aligned}
    \end{equation*}
    Therefore, the fine-tuning error I can be viewed as the error induced by using $\hat{f}^{P}$ to construct the intermediate label for recovering $f^{\delta}$. The proof remains to bound three errors respectively. 

    For the pre-training error, applying Theorem~\ref{thm: adaptive rate of SA with Gaussian} leads to, with probability $1-\delta$,
    \begin{equation}\label{eqn: upper bound proof pre-training error}
    \begin{aligned}
        \left\| \hat{f}^{P} - f^{P}\right\|_{L^{2}}^{2} \leq C_{1} \left( \log \frac{4}{\delta} \right)^{2}  \left(\frac{n_{P}}{\log n_{P}} \right)^{-\frac{2m_{P}}{2m_{P} +d}}.
    \end{aligned}
    \end{equation}
    Similarly, for the fine-tuning error II, we have, with probability $1-\delta$,
    \begin{equation}\label{eqn: upper bound proof fine-tuning error II}
        \left\| \tilde{f}^{\delta} - f^{\delta}\right\|_{L^{2}}^{2} \leq C_{2} \left( \log \frac{4}{\delta} \right)^{2}  \left(\frac{n_{Q}}{\log n_{Q}} \right)^{-\frac{2m_{\delta}}{2m_{\delta} +d}}.
    \end{equation}
    For the fine-tuning error I, applying Theorem~\ref{thm: bounds on intermediate error} leads to, with probability $1-\delta$ and sufficient large $n_{Q}$,
    \begin{equation*}
    \begin{aligned}
         \left\| \hat{f}^{\delta} - \tilde{f}^{\delta}\right\|_{L^{2}} & \leq 6E_{\tau_{2}} \left\{1 + 4\sqrt{2} \log \left( \frac{6}{\delta} \right)  E_{K} n_{Q}^{-\frac{1}{2}\frac{4m_{\delta} + d}{2m_{\delta}+d}}  \right\} \left\| \hat{f}^{P} - f^{P}  \right\|_{L^{2}} \\
        &  \quad \quad   + 24 E_{\tau_{2}}\sqrt{2} \log \left( \frac{6}{\delta} \right)  n_{Q}^{-\frac{m_{\delta}}{2m_{\delta}+d}}. 
    \end{aligned}
    \end{equation*}
    Therefore, for $\delta\in (0,1)$, with sufficient large $n_{Q}$, we have
    \begin{equation}\label{eqn: upper bound proof fine-tuning error I}
        \left\| \hat{f}^{\delta} - \tilde{f}^{\delta}\right\|_{L^{2}}^{2} \leq   288 E_{\tau_{2}}^{2}C_{1} \left( \log \frac{4}{\delta} \right)^{2} \left(\frac{n_{P}}{\log n_{P}} \right)^{-\frac{2m_{P}}{2m_{P} +d}} .
    \end{equation}
    Combing \eqref{eqn: upper bound proof pre-training error}, \eqref{eqn: upper bound proof fine-tuning error II} and \eqref{eqn: upper bound proof fine-tuning error I}, we have with probability $1-3\delta$
    \begin{equation*}
        \left\| \hat{f}^{Q} - f^{Q}\right\|_{L^{2}}^{2} \leq \left( \log \frac{4}{\delta} \right)^{2} \left\{ C_{1}^{'} \left(\frac{n_{P}}{\log n_{P}} \right)^{-\frac{2m_{P}}{2m_{P} +d}} + C_{2}^{'} \left(\frac{n_{Q}}{\log n_{Q}} \right)^{-\frac{2m_{\delta}}{2m_{\delta} +d}}  \right\}, 
    \end{equation*}
    where $C_{1}' \propto \sigma_{P}^{2} + \|f^{P}\|_{H^{m_{P}}}^2 \leq \sigma_{P}^{2} + R_{P}^{2}$  and $C_{2}' \propto \sigma_{Q}^{2} + \|f^{\delta}\|_{H^{m_{\delta}}}^2 \leq \sigma_{Q}^{2} + R_{\delta}^{2}$. 
\end{proof}

\subsection{Proof of Proposition~\ref{prop: smoothness of target}}
The property of composition preserves the Sobolev order and is termed the superposition property for (fractional) Sobolev spaces \citep{leoni2023first} and has been studied from the 1970s to 1990s under various conditions. However, most of the existing results are built on univariate composition function; see a detailed review in \citet{brezis2001gagliardo}. We thus prove the superposition property for the bivariate composition functions, which take two functions within different Sobolev spaces as input.

\begin{proof}
    We first show that for $u \in H^{m_{\delta}}(\mcX)$ and $v \in H^{m_{P}}(\mcX)$, the mapping 
    \begin{equation*}
        (u,v) \mapsto DG(u,v) = \frac{\partial G}{\partial u} Du + \frac{\partial G}{\partial v} Dv \in H^{\min\{m_{P}, m_{\delta}\} - 1}(\mcX)
    \end{equation*}
    is well-defined and continuous. It suffices to show that the following mappings
    \begin{equation}\label{equation: mapping 1 in the smoothness of target}
        u \mapsto \frac{\partial G}{\partial u} Du \in H^{m_{\delta}-1}(\mcX) \qquad \text{for all fixed } v
    \end{equation}
    and 
    \begin{equation}\label{equation: mapping 2 in the smoothness of target}
        v \mapsto \frac{\partial G}{\partial v} Dv \in H^{m_{P}-1}(\mcX) \qquad \text{for all fixed } u
    \end{equation}
    is well-defined and continuous. Denote $k:= \max\{\lceil m_{P} \rceil,\lceil m_{\delta} \rceil\}$, since $G(x,y) \in C^{k}(\mbR^{2})$, then for each fixed $y$, we have $G(x,y)\in C^{k}(\mbR)$, and similarly for each fixed $x$, we have $G(x,y)\in C^{k}(\mbR)$. Then, the proof for showing the mappings~\eqref{equation: mapping 1 in the smoothness of target} and \eqref{equation: mapping 2 in the smoothness of target} is well-defined and continuous follows the same as Theorem 1 of \citet{brezis2001gagliardo}. This leads to 
    \begin{equation*}
        \frac{\partial G}{\partial u} Du + \frac{\partial G}{\partial v} Dv \in H^{\min\{m_{P}-1, m_{\delta}-1\}} (\mcX) = H^{\min\{m_{P}, m_{\delta}\}-1} (\mcX)
    \end{equation*}
    is well-defined and continuous. Finally, since $G$ is Lipschitz continuous, and $\mcX$ is compact, the mapping $(u,v) \mapsto G(u,v)$ is well-defined and continuous, which completes the proof.
\end{proof}

\subsection{Supporting Theorems}

In the proof for the upper bound of the excess risk, the following Theorem is used, which provides an upper bound for the error induced by using estimator $\hat{f}^{P}$ to construct intermediate labels. 

\begin{theorem}\label{thm: bounds on intermediate error}
Suppose the spectral algorithm used in the second phase possesses a filter function with regularization parameter $\lambda_{2}$ and qualification $\tau_{2}$. Given $\lambda_{2} = C \exp\{-n_{Q}^{\frac{2}{2m_{\delta}+d}}\}$, then for all $\delta \in (0, 1)$, with probability at least $1-\delta$, we have
\begin{equation*}
\begin{aligned}
    & \left\| \phi_{\lambda_{2}}(T_{K,n}^{Q}) \frac{1}{n_{Q}} \sum_{i=1}^{n_{Q}} K_{x_{i}^{Q}}^{*} \left( \hat{f}^{P}(x_{i}^{Q}) - f^{P}(x_{i}^{Q}) \right) \right\|_{L^{2}} \leq 6E_{\tau_{2}} \cdot \\
    &  \quad \quad  \left\{ \left[ 1 + 4\sqrt{2} \log \left( \frac{6}{\delta} \right)  E_{K} n_{Q}^{-\frac{1}{2}\frac{4m_{\delta} + d}{2m_{\delta}+d}}  \right]  \left\| \hat{f}^{P} - f^{P}  \right\|_{L^{2}} + 4\sqrt{2} \log \left( \frac{6}{\delta} \right)  n_{Q}^{-\frac{m_{\delta}}{2m_{\delta}+d}} \right\}. 
\end{aligned}
\end{equation*}
with respect to $(\mcD^{P}, \mcD^{Q})$.
\end{theorem}

\begin{proof}[Proof of Theorem~\ref{thm: bounds on intermediate error}]
The error induced by using pre-trained models to generate offset labels is
\begin{equation*}
\begin{aligned}
    & \left\| \phi_{\lambda_{2}}(T_{K,n}^{Q}) \frac{1}{n_{Q}} \sum_{i=1}^{n_{Q}} K_{x_{i}^{Q}}^{*} \left( \hat{f}^{P}(x_{i}^{Q}) - f^{P}(x_{i}^{Q}) \right) \right\|_{L^{2}} \\
    = & \left\| T_{K}^{\frac{1}{2}} \phi_{\lambda_{2}}(T_{K,n}^{Q}) \frac{1}{n_{Q}} \sum_{i=1}^{n_{Q}} K_{x_{i}^{Q}}^{*} \left( \hat{f}^{P}(x_{i}^{Q}) - f^{P}(x_{i}^{Q}) \right) \right\|_{\mcH_{K}} \\
    \leq & \underbrace{\left\|T_{K}^{\frac{1}{2}} \left( T_{K} + \lambda I  \right)^{-\frac{1}{2}} \right\|_{op}}_{D_{1}}  \cdot  \underbrace{\left\| \left( T_{K} + \lambda I  \right)^{\frac{1}{2}} \left( T_{K,n}^{Q} + \lambda I  \right)^{-\frac{1}{2}}  \right\|_{op}}_{D_{2}}  \\
        & \quad \cdot \underbrace{\left\| \left( T_{K,n}^{Q} + \lambda I  \right)^{\frac{1}{2}} \phi_{\lambda_{2}}(T_{K,n}^{Q}) \frac{1}{n_{Q}} \sum_{i=1}^{n_{Q}} K_{x_{i}^{Q}}^{*} \left( \hat{f}^{P}(x_{i}^{Q}) - f^{P}(x_{i}^{Q}) \right)\right\|_{\mcH_{K}}}_{D_{3}} 
\end{aligned}
\end{equation*}
For the first term $D_{1}$, we have 
\begin{equation}
    D_{1} = \left\|T_{K}^{\frac{1}{2}} \left( T_{K} + \lambda I  \right)^{-\frac{1}{2}} \right\|_{op} = \sup_{j\geq 1} \left( \frac{s_{j}}{s_{j} + \lambda}\right)^{\frac{1}{2}} \leq 1.
\end{equation}
For the second term $D_{2}$, using Lemma~\ref{lemma: bounds for A2} with sufficient large $n$, we have 
\begin{equation}
    u:=\frac{\mcN(\lambda)}{n} \log ( \frac{12\mcN(\lambda)}{\delta} \frac{(\|T_{K}\|_{op}+\lambda)}{\|T_{K}\|_{op}}) \leq \frac{1}{8}
\end{equation}
such that
\begin{equation}
     \left\| \left( T_{K} + \lambda I \right)^{-\frac{1}{2}} \left( T_{K} - T_{K,n}^{Q} \right) \left( T_{K} + \lambda I \right)^{-\frac{1}{2}} \right\|_{op} \leq \frac{4}{3}u + \sqrt{2u}\leq  \frac{2}{3}
\end{equation}
holds with probability $1-\frac{\delta}{3}$. Thus,
\begin{equation}
\begin{aligned}
D_{2}^{2} &= \left\| (T_{K} + \lambda I)^{\frac{1}{2}} (T_{K,n}^{Q} + \lambda I)^{-\frac{1}{2}}\right\|_{op}^2 \\
& = \left\| (T_{K} + \lambda I)^{\frac{1}{2}} (T_{K,n}^{Q} + \lambda I)^{-1} (T_{K} + \lambda I )^{\frac{1}{2}}\right\|_{op}\\
& =\left\|\left( (T_{K} + \lambda I)^{-\frac{1}{2}}\left(T_{K,n}^{Q}+\lambda\right) (T_{K} + \lambda I) ^{-\frac{1}{2}}\right)^{-1}\right\|_{op} \\
& =\left\|\left(I-  (T_{K} + \lambda I )^{-\frac{1}{2}}\left(T_{K,n}^{Q}-T_{K}\right) (T_{K} + \lambda I )^{-\frac{1}{2}}\right)^{-1}\right\|_{op} \\
& \leq \sum_{k=0}^{\infty}\left\| (T_{K} + \lambda I )^{-\frac{1}{2}}\left(T_{K} - T_{K,n}^{Q}\right) (T_{K} + \lambda I )^{-\frac{1}{2}}\right\|_{op}^k \\
& \leq \sum_{k=0}^{\infty}\left(\frac{2}{3}\right)^k \leq 3,
\end{aligned}
\end{equation}
For the third term $D_{3}$, notice 
\begin{equation*}
\begin{aligned}
    D_{3} & = \left\| \left( T_{K,n}^{Q} + \lambda I  \right)^{\frac{1}{2}} \phi_{\lambda_{2}}(T_{K,n}^{Q}) \frac{1}{n_{Q}} \sum_{i=1}^{n_{Q}} K_{x_{i}^{Q}}^{*} \left( \hat{f}^{P}(x_{i}^{Q}) - f^{P}(x_{i}^{Q}) \right)\right\|_{\mcH_{K}} \\
    & \leq \underbrace{\left\| \left( T_{K,n}^{Q} + \lambda I \right)^{\frac{1}{2}} \phi_{\lambda_{2}}(T_{K,n}^{Q}) \left( T_{K,n}^{Q} + \lambda I \right)^{\frac{1}{2}}  \right\|_{op}}_{D_{31}} \cdot \underbrace{\left\| \left( T_{K,n}^{Q} + \lambda I \right)^{-\frac{1}{2}} \left( T_{K} + \lambda I \right)^{\frac{1}{2}}  \right\|_{op}}_{D_{32}} \\
    & \quad \cdot \underbrace{\left\| \left( T_{K} + \lambda I  \right)^{-\frac{1}{2}} \frac{1}{n_{Q}} \sum_{i=1}^{n_{Q}} K_{x_{i}^{Q}}^{*} \left( \hat{f}^{P}(x_{i}^{Q}) - f^{P}(x_{i}^{Q}) \right)\right\|_{\mcH_{K}}}_{D_{33}} 
\end{aligned}
\end{equation*}
For $D_{31}$, the properties of the filter function indicate $z \phi_{\lambda_{2}} \leq E$, thus we have
\begin{equation*}
     \left\| \left( T_{K,n}^{Q} + \lambda I \right)^{\frac{1}{2}} \phi_{\lambda}(T_{K,n}^{Q})  \left( T_{K,n}^{Q} + \lambda I \right)^{\frac{1}{2}}  \right\|_{op} =  \left\| \left( T_{K,n}^{Q} + \lambda I \right) \phi_{\lambda}(T_{K,n}^{Q})  \right\|_{op} \leq 2E_{\tau_{2}}.
\end{equation*}
For $D_{32}$, the bound for $D_{2}$ implies, with probability $1-\frac{\delta}{3}$, we have
\begin{equation*}
    \left\| \left( T_{K,n}^{Q} + \lambda I \right)^{-\frac{1}{2}} \left( T_{K} + \lambda I \right)^{\frac{1}{2}}  \right\|_{op} \leq \sqrt{3}.
\end{equation*}
For $D_{33}$, applying Lemma~\ref{lemma: bounds for D3} and Lemma~\ref{lemma: bound for effective dimension}, we have
\begin{equation*}
\begin{aligned}
    & \left\| \left( T_{K} + \lambda_{2} I  \right)^{-\frac{1}{2}} \frac{1}{n_{Q}} \sum_{i=1}^{n_{Q}} K_{x_{i}^{Q}}^{*} \left( \hat{f}^{P}(x_{i}^{Q}) - f^{P}(x_{i}^{Q}) \right)\right\|_{\mcH_{K}} \\
    & \quad \leq 
    \left\{1 + 4\sqrt{6} \log \left( \frac{2}{\delta} \right)  E_{K} n_{Q}^{-\frac{1}{2}\frac{4m_{\delta} + d}{2m_{\delta}+d}}  \right\} \left\| \hat{f}^{P} - f^{P}  \right\|_{L^{2}} + 4\sqrt{2} \log \left( \frac{6}{\delta} \right)  n_{Q}^{-\frac{m_{\delta}}{2m_{\delta}+d}} 
\end{aligned}
\end{equation*}
with probability $1-\frac{\delta}{3}$ with respect to $\mcD^{Q}$ given $\mcD^{P}$. Combining all terms, we have 
\begin{equation*}
\begin{aligned}
    & \left\| \phi_{\lambda_{2}}(T_{K,n}^{Q}) \frac{1}{n_{Q}} \sum_{i=1}^{n_{Q}} K_{x_{i}^{Q}}^{*} \left( \hat{f}^{P}(x_{i}^{Q}) - f^{P}(x_{i}^{Q}) \right) \right\|_{L^{2}}\\
    & \quad  \leq 6E_{\tau_{2}} \left\{1 + 4\sqrt{2} \log \left( \frac{6}{\delta} \right)  E_{K} n_{Q}^{-\frac{1}{2}\frac{4m_{\delta} + d}{2m_{\delta}+d}}  \right\} \left\| \hat{f}^{P} - f^{P}  \right\|_{L^{2}} + 4\sqrt{2} \log \left( \frac{6}{\delta} \right)  n_{Q}^{-\frac{m_{\delta}}{2m_{\delta}+d}}
\end{aligned}
\end{equation*}
with probability at least $1-\delta$ with respect to $(\mcD^{P}, \mcD^{Q})$.
\end{proof}

\subsection{Supporting Lemmas}
The following lemma controls the error induced by using $\hat{f}^{P}$ under the Tikhonov regularization function.
\begin{lemma}\label{lemma: bounds for D3}
    Suppose $n_{P}$ is sufficient large and fixed (thus $\hat{f}^{P}$ is fixed), then for $\delta \in (0,1)$, with probability $1-\delta$, we have
    \begin{equation*}
    \begin{aligned}
        & \left\| \left( T_{K} + \lambda_{2} I  \right)^{-\frac{1}{2}} \frac{1}{n_{Q}} \sum_{i=1}^{n_{Q}} K_{x_{i}^{Q}}^{*} \left( \hat{f}^{P}(x_{i}^{Q}) - f^{P}(x_{i}^{Q}) \right)\right\|_{\mcH_{K}} \\
        & \quad \leq 
        \left\{1 + 4\sqrt{2} \log \left( \frac{2}{\delta} \right)  \frac{E_{K} \sqrt{\mcN(\lambda_{2})} }{n_{Q}}  \right\} \left\| \hat{f}^{P} - f^{P}  \right\|_{L^{2}} + 4\sqrt{2} \log \left( \frac{2}{\delta} \right)  \frac{\sqrt{ \mcN(\lambda_{2})} }{\sqrt{n_{Q}}}
    \end{aligned}
    \end{equation*}
\end{lemma}
\begin{proof}
\begin{equation*}
\begin{aligned}
    & \left\| \left( T_{K} + \lambda_{2} I  \right)^{-\frac{1}{2}} \frac{1}{n_{Q}} \sum_{i=1}^{n_{Q}} K_{x_{i}^{Q}}^{*} \left( \hat{f}^{P}(x_{i}^{Q}) - f^{P}(x_{i}^{Q}) \right)\right\|_{\mcH_{K}} \\
    \leq & \underbrace{\left\| \left( T_{K} + \lambda_{2} I \right)^{-\frac{1}{2}} T_{K} \left( \hat{f}^{P} - f^{P} \right) \right\|_{\mcH_{K}}}_{E_{1}} \\
    \quad & + \underbrace{\left\| \left( T_{K} + \lambda{2} I  \right)^{-\frac{1}{2}}  \left\{ \frac{1}{n_{Q}} \sum_{i=1}^{n_{Q}} K_{x_{i}^{Q}}^{*} \left( \hat{f}^{P}(x_{i}^{Q}) - f^{P}(x_{i}^{Q}) \right) - T_{K}\left(\hat{f}^{P} - f^{P}\right)\right\} \right\|_{\mcH_{K}}}_{E_{2}} .
\end{aligned}
\end{equation*}
For term $E_{1}$, 
\begin{equation*}
\begin{aligned}
    \left\| \left( T_{K} + \lambda_{2} I \right)^{-\frac{1}{2}} T_{K} \left( \hat{f}^{P} - f^{P} \right) \right\|_{\mcH_{K}} & = \left\| \left( T_{K} + \lambda_{2} I \right)^{-\frac{1}{2}} T_{K}^{\frac{1}{2}} \left( \hat{f}^{P} - f^{P} \right) \right\|_{L^{2}} \\
    & \leq \left\| \left( T_{K} + \lambda_{2} I \right)^{-\frac{1}{2}} T_{K}^{\frac{1}{2}} \right\|_{op} \cdot \left\| \hat{f}^{P} - f^{P}  \right\|_{L^{2}} \\
    & \leq \left\| \hat{f}^{P} - f^{P}  \right\|_{L^{2}},
\end{aligned}
\end{equation*}
where the last inequality is based on the fact that the operator norm of  $( T_{K} + \lambda_{2} I )^{-\frac{1}{2}} T_{K}^{\frac{1}{2}}$ is bounded by $1$. 

For term $E_{2}$, denote
\begin{equation*}
\begin{aligned}
    & \xi_{i} = \left( T_{K} + \lambda_{2} I \right)^{-\frac{1}{2}} K_{x_{i}^{Q}}^{*} \left( \hat{f}^{P}(x_{i}^{Q}) - f^{P}(x_{i}^{Q})\right), \\
    & \xi_{x} = \left( T_{K} + \lambda_{2} I \right)^{-\frac{1}{2}} K_{x^{Q}} \left( \hat{f}^{P}(x^{Q}) - f^{P}(x^{Q})\right).
\end{aligned}
\end{equation*}
Therefore, bounding $E_{2}$ is equivalent to bounding 
\begin{equation*}
    \left \| \frac{1}{n_{Q}} \sum_{i=1}^{n_{Q}} \xi_{i} - \E \xi_{x} \right\|_{\mcH_{K}}.
\end{equation*}
Notice for $r> 2$
\begin{equation*}
\begin{aligned}
    \E_{x^{Q}} \left\| \xi_{x} \right\|_{\mcH_{K}}^{r} & = \E_{x^{Q}} \left\| \left( T_{K} + \lambda_{2} I \right)^{-\frac{1}{2}} K_{x^{Q}} \left( \hat{f}^{P}(x^{Q}) - f^{P}(x^{Q}) \right)\right\|_{\mcH_{K}}^{r}\\
    & \leq \E_{x^{Q}} \left( \left\| \left( T_{K} + \lambda_{2} I \right)^{-\frac{1}{2}} K_{x^{Q}} \right\|_{\mcH_{K}}^{r} \cdot \left| \hat{f}^{P}(x^{Q}) - f^{P}(x^{Q})\right |^{r} \right).
\end{aligned}
\end{equation*}
By Lemma~\ref{lemma: bound for kernel function},
\begin{equation*}
    \left\| \left( T_{K} + \lambda_{2} I \right)^{-\frac{1}{2}} K_{x^{Q}} \right\|_{\mcH_{K}}^{r} \leq E_{K}^{r} \mcN(\lambda_{2})^{\frac{r}{2}}.
\end{equation*}
Further, for sufficient large but fixed $n_{P}$, by Cauchy–Schwarz inequality, we have,
\begin{equation*}
\begin{aligned}
    \E_{X\sim Q_{X}} \left| \hat{f}^{P}(X) - f^{P}(X) \right|^{r} & \leq \sqrt{ \int_{\mcX} (\hat{f}^{P} - f^{P})^{2} d Q_{X}(x) \cdot \int_{\mcX} (\hat{f}^{P} - f^{P})^{2r-2}d Q_{X}(x) } \\
    & = \left\| (\hat{f}^{P} - f^{P}) \right\|_{L^{2}}\cdot \left\| (\hat{f}^{P} - f^{P})^{r} \right\|_{L^{2}}  \\
    & \leq \left\| \hat{f}^{P} - f^{P} \right\|_{L^{2}}^{2}.
\end{aligned}
\end{equation*}
Combining, we have 
\begin{equation*}
\begin{aligned}
    \E_{x^{Q}} \left\| \xi_{x} \right\|_{\mcH_{K}}^{r} 
    & \leq E_{K}^{r } \mcN(\lambda_{2})^{\frac{r}{2}}   \left\| \hat{f}^{P} - f^{P} \right\|_{L^{2}}^{2} \\
    & \leq \frac{1}{2} r! \left( \sqrt{ \mcN(\lambda_{2})} \left\| \hat{f}^{P} - f^{P}\right\|_{L^{2}} \right)^{2} \cdot \left( E_{K} \sqrt{ \mcN(\lambda_{2})} \right)^{r-2} \\
    & := \frac{1}{2} r! \tilde{\sigma}^{2} \tilde{L}^{r-2}.
\end{aligned}
\end{equation*}
Then by applying Lemma~\ref{lemma: Bernstein inequality}, we have
\begin{equation*}
    E_{2} \leq 4\sqrt{2} \log\left(\frac{2}{\delta}\right) \left( \frac{E_{K} \sqrt{ \mcN(\lambda_{2})} \left\| \hat{f}^{P} - f^{P}\right\|_{L^{2}}}{n_{Q}} + \frac{\sqrt{\mcN(\lambda_{2})} }{\sqrt{n_{Q}}} \right)
\end{equation*}
with probability $1-\delta$.
\end{proof}

The following lemma states the minimax lower bound of the excess risk, often referred to as the information-theoretic lower bound, under the single dataset scenario. The proof is standard and can be found in different kernel methods literature; see \citet{rastogi2017optimal,zhang2024optimality}, among others. However, most of the works omit the property of constant $C$. In order to derive the form of $\xi$ in Theorem~\ref{thm: lower bound of OTL}, we apply a version from \citet{lin2024smoothness}, where the authors have shown that $C$ is proportional to the radius of the Sobolev ball.
\begin{lemma}
\label{lemma: lower bound for single task SA}
    Following the same classical nonparametric regression setting in Section~\ref{sec: SA with Gaussians}, the underlying true function $f^{*} \in \{ f\in H^{m}: \|f\|_{H^{m}} \leq R \}:=\mcB_{m}(R)$. Then, when $n$ is sufficiently large, one has 
    \begin{equation*}
        \inf_{\tilde{f}} \sup_{ f\in \mcB_{m}(R) } \mathbb{P} \left( \| \tilde{f} - f \|_{L^{2}}^2 \geq C \delta n^{-\frac{2m}{2m + d}}  \right) \geq 1 - \delta,
    \end{equation*}
    where the constant $C$ is proportional to $R^{2}$ and independent of $\delta$ and $n$.
\end{lemma}

\section{Additional Experimental Results}\label{apd: additional experiments}

In Figure~\ref{fig: Gaus_SA_all_slop}, we present the results that are based on different choices of $C$. It shows that the empirical excess risk decay rates still closely align with theoretical ones in both non-adaptive and adaptive cases.

\begin{figure}[ht]
    \centering
    \subfloat[Non-adaptive]{
        \centering
        \includegraphics[width=0.48\linewidth]{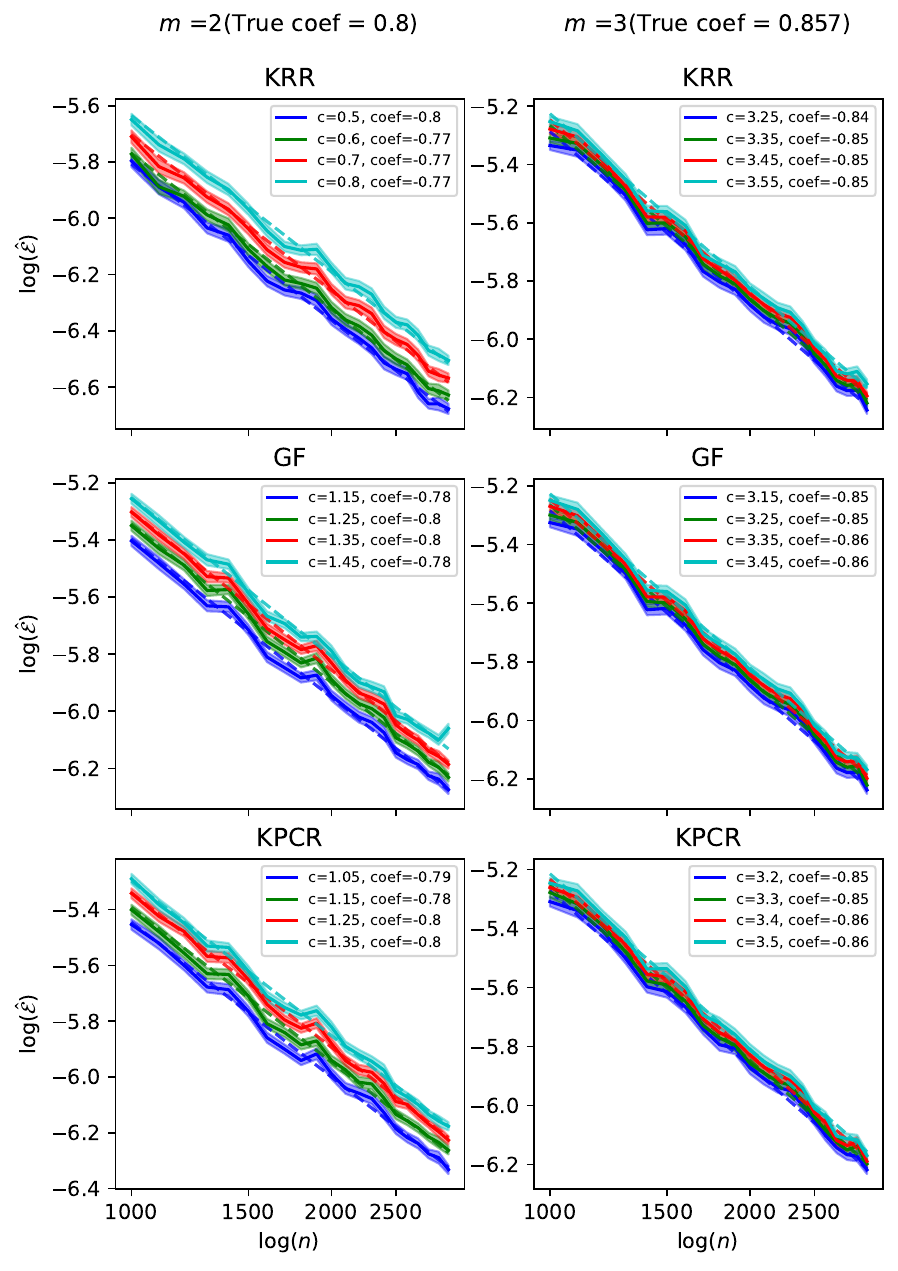} 
    }
    \subfloat[Adaptive]{
        \centering
        \includegraphics[width=0.48\linewidth]{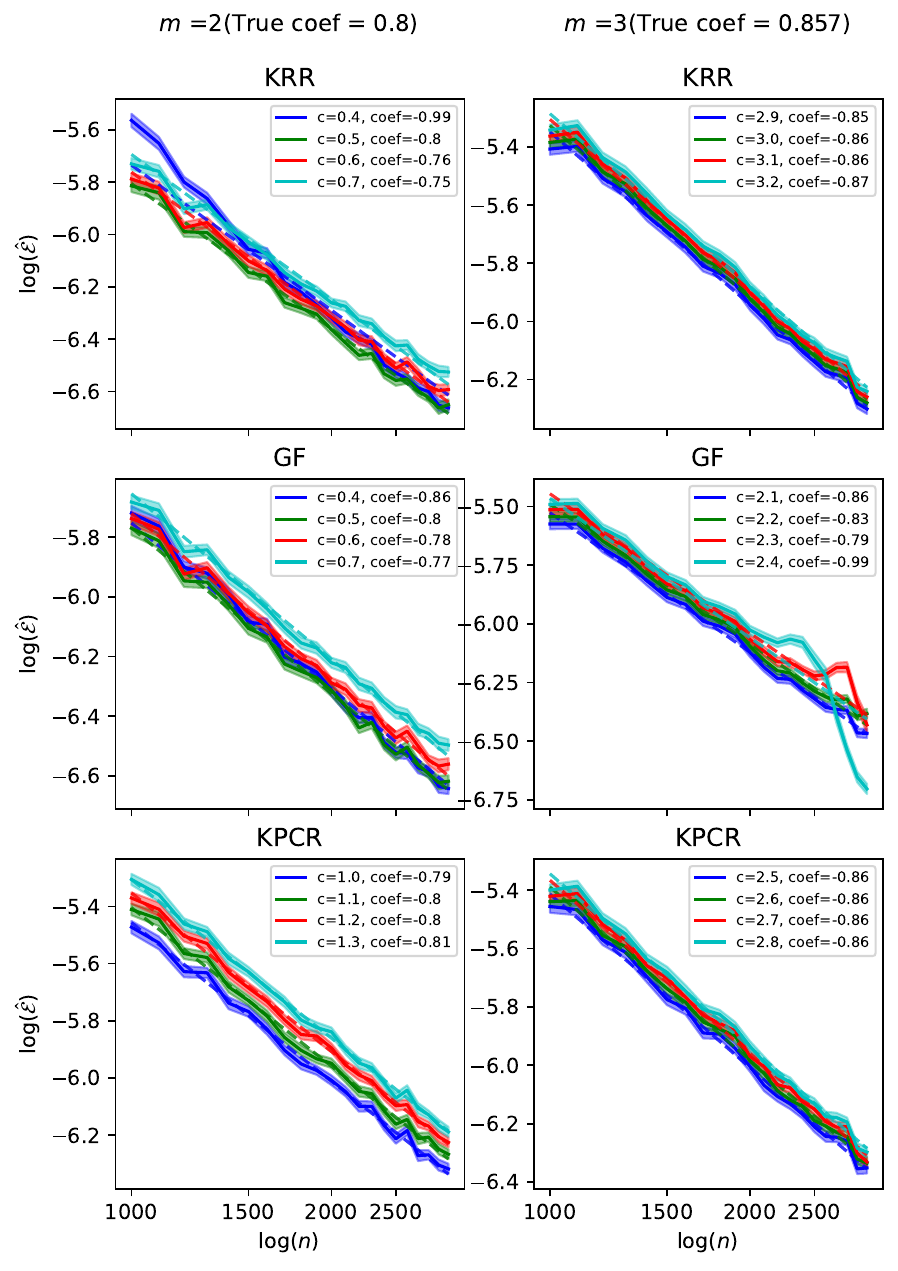} 
    }
    \caption{Error decay curves of spectral algorithms with Gaussian kernels with different selections of $C$. Both axes are plotted on a log scale. The dashed black lines denote the theoretical regression line of $\log \mcE$ on $\log n$ with slope $-\frac{2m}{2m+1}$, denoted by ``True''. Blue curves denote the average empirical excess risk over repeated trials, with shaded regions indicating $\pm 1$ standard error of the mean. "Est." denotes the estimated slope of the regression line.} 
    \label{fig: Gaus_SA_all_slop}
\end{figure}

\clearpage

{
\bibliographystyle{plainnat}
\bibliography{references}
}

\end{document}